\documentclass[11pt]{article}
\usepackage{mathpazo}
\usepackage{environ}

\usepackage[utf8x]{inputenc}

\usepackage[round]{natbib}
\usepackage{booktabs}
\usepackage{bm}
\usepackage{amsthm,amsfonts,amsmath,amssymb,epsfig,color,float,graphicx,verbatim, enumitem}

\usepackage{svg}

\definecolor{Gred}{RGB}{219, 50, 54}
\definecolor{Ggreen}{RGB}{60, 186, 84}
\definecolor{Gblue}{RGB}{72, 133, 237}
\definecolor{Gyellow}{RGB}{247, 178, 16}
\definecolor{ToCgreen}{RGB}{0, 128, 0}
\definecolor{myGold}{RGB}{231,141,20}
\definecolor{myBlue}{rgb}{0.19,0.41,.65}
\definecolor{myPurple}{rgb}{175,0,124}
\definecolor{niceRed}{RGB}{153,0,0}
\definecolor{niceRed}{RGB}{190,38,38}
\definecolor{blueGrotto}{HTML}{059DC0}
\definecolor{royalBlue}{HTML}{057DCD}
\definecolor{navyBlueP}{HTML}{0B579C}
\definecolor{limeGreen}{HTML}{81B622}

\usepackage{hyperref}
\hypersetup{
	colorlinks = true,
	urlcolor = {blue},
	linkcolor = {niceRed},
	citecolor = {blue}
}
\usepackage[nameinlink]{cleveref}
\usepackage{caption,subcaption}
\usepackage{algorithm, algorithmic}
\usepackage{listings}

\def\l{\ell}
\def\<{\langle}
\def\>{\rangle}
\def\eps{\epsilon}
\def\wt{\widetilde}
\def\wh{\widehat}

\DeclareMathOperator*{\argmin}{argmin}

\def\eps{\varepsilon}
\def\vec{\bm}

\newtheorem{theorem}{Theorem}

\newtheorem{claim}{Claim}

\newtheorem{definition}{Definition}
\newtheorem{condition}{Condition}

\newtheorem{problem}{Problem}
\newtheorem{remark}{Remark}
\newtheorem{fact}{Fact}

\renewcommand{\Pr}{\mathop{\bf Pr\/}}
\newcommand{\E}{\mathop{\bf E\/}}

\newcommand{\Var}{\mathop{\bf Var\/}}

\newcommand{\poly}{\textnormal{poly}}

\newcommand{\sgn}{\textnormal{sgn}}

\newcommand{\reals}{\mathbb R}

\newcommand{\nats}{\mathbb N}

\newcommand{\calC}{\mathcal{C}}
\newcommand{\calD}{\mathcal{D}}
\newcommand{\calE}{\mathcal{E}}
\newcommand{\calF}{\mathcal{F}}
\newcommand{\calG}{\mathcal{G}}
\newcommand{\calH}{\mathcal{H}}

\newcommand{\calM}{\mathcal{M}}

\newcommand{\calO}{\mathcal{O}}
\newcommand{\calP}{\mathcal{P}}

\newcommand{\calV}{\mathcal{V}}

\newcommand{\xSpace}{\mathbb X}
\newcommand{\labSpace}{\cal Y}
\newcommand{\permSpace}{\mathbb S}

\makeatletter
\renewenvironment{abstract}{%
	\if@twocolumn
	\section*{\abstractname}%
	\else 
	\begin{center}%
		{\bfseries \large\abstractname\vspace{\z@}}
	\end{center}%
	\quotation
	\fi}
{\if@twocolumn\else\endquotation\fi}
\makeatother

\begin{document}
	\title{Label Ranking through Nonparametric Regression%
	\thanks{This work is supported by NTUA Basic Research Grant (PEBE 2020) "Algorithm Design through Learning Theory: Learning-Augmented and Data-Driven Online Algorithms (LEADAlgo)".}}
	\author{
		\textbf{Dimitris Fotakis}\footnote{National Technical University of Athens, \url{fotakis@cs.ntua.gr} } \\
		\small NTUA \\
		\and
		\textbf{Alkis Kalavasis}\footnote{National Technical University of Athens, \url{kalavasisalkis@mail.ntua.gr} } \\
		\small NTUA \\
		\and
		\textbf{Eleni Psaroudaki} \footnote{National Technical University of Athens, \url{epsaroudaki@mail.ntua.gr} }  \\
		\small NTUA \\
	}
	\maketitle
	\thispagestyle{empty}

	\begin{abstract}
	\small
	Label Ranking (LR) corresponds to the problem of learning a hypothesis that maps features to rankings over a finite set of labels. We adopt a nonparametric regression approach to LR and obtain theoretical performance guarantees for this fundamental practical problem. We introduce a generative model for Label Ranking, in noiseless and noisy nonparametric regression settings, and provide sample complexity bounds for learning algorithms in both cases. In the noiseless setting, we study the  LR problem with full rankings and provide computationally efficient algorithms using decision trees and random forests in the high-dimensional regime. In the noisy setting, we consider the more general cases of LR with incomplete and partial rankings from a statistical viewpoint and obtain sample complexity bounds using the One-Versus-One approach of multiclass classification. Finally, we complement our theoretical contributions with experiments, aiming to understand how the input regression noise affects the observed output.
	\end{abstract}

\parindent=18pt
\section{Introduction}
\label{sec:introduction}
Label Ranking (LR) studies the problem of learning a mapping from features to rankings over a finite set of labels. This task emerges in many domains. Common practical illustrations include pattern recognition \citep{geng2014multilabel}, web advertisement \citep{djuric2014non}, sentiment analysis \citep{wang2011ranking}, document categorization \citep{jindal2015ranking} and bio-informatics \citep{balasubramaniyan2005clustering}. 
The importance of LR has spurred the development of several approaches for tackling this task from the perspective of the applied CS community \citep{vembu2010label, zhou2014taxonomy}.

The overwhelming majority of these solutions comes with experimental evaluation and no theoretical guarantees; e.g., algorithms based on decision trees are a workhorse for practical LR and lack theoretical guarantees. Given state-of-the-art experimental results, based on Random Forests \citep[see][]{zhou2018random}, we are highly motivated not only to work towards a theoretical understanding of this central learning problem but also to theoretically analyze how efficient tree-based methods can be under specific assumptions. 

LR comprises a supervised learning problem that extends multiclass classification \citep{dekel2003log}. In the latter, with instance domain $\xSpace \subseteq \reals^d$ and set of labels $[k] := \{1,\ldots, k\}$, the learner draws i.i.d. labeled examples $(\vec x, y) \in \xSpace \times [k]$ and aims to learn a hypothesis from instances to labels, following the standard PAC model.
In LR, the learner observes labeled examples $(\vec x, \sigma) \in \xSpace \times \permSpace_k$ and the goal is to learn a hypothesis $h : \xSpace \to \permSpace_k$ from instances to \emph{rankings of labels}, where $\permSpace_k$ is the symmetric group of $k$ elements. 
The ranking $h(\vec x)$ corresponds to the preference list of the feature $\vec x$ and, as mentioned in previous works \citep{hullermeier2008label},
a natural way to represent preferences is to evaluate individual alternatives through a real-valued utility (or score) function. 
Note that if the training data offer the utility scores directly, the problem is reduced to a standard regression problem. In our work, we assume that there exists such an \emph{underlying nonparametric score} function $\vec m : \xSpace \to [0,1]^k$, mapping features to score values. The value $m_i(\vec x)$ corresponds to the score assigned to the label $i \in [k]$ for input $\vec x$ and can be considered proportional to the 
posterior probability $ \Pr_{(\vec x, y)}[y = i | \vec x]$. For each LR example $(\vec x, \sigma)$, the label $\sigma$ is generated by sorting the underlying regression-score vector $\vec m(\vec x)$, i.e., $\sigma = \mathrm{argsort}(\vec m(\vec x) + \vec \xi)$ (with some regression noise $\vec \xi$). We are also interested in cases where some of the alternatives of $\sigma$ are missing, i.e., we observe incomplete rankings $\sigma \in \permSpace_{\leq k}$; the way that such rankings occur will be clarified later. 
Formally, we have:
\begin{definition}
[Distribution-free Nonparametric LR]
\label{def:distr-free-lr}
Let $\mathbb X \subseteq \reals^d$, $[k]$ be a set of labels, $\calC$ be a class of functions from $\mathbb X$ to $[0,1]^k$ and $\calD_x$ be an arbitrary distribution over $\mathbb X$. Consider a noise distribution $\calE$ over $\reals^k$. Let  $\vec m$ be an unknown target function in $\calC$. 
\begin{itemize}
    \item An example oracle $\mathrm{Ex}(\vec m, \calE)$ with complete rankings, works as follows: Each time $\mathrm{Ex}(\vec m, \calE)$ is invoked, it returns a labeled example $(\vec x, \sigma) \in \mathbb X \times \permSpace_k$, where (i) $\vec x \sim \calD_x$ and $\vec \xi \sim \calE$ independently and (ii) $\sigma = \mathrm{argsort}(\vec m(\vec x) + \vec \xi)$. Let $D_R$ be the joint distribution over $(\vec x, \sigma)$ generated by the oracle. In the noiseless case ($\vec \xi = \vec 0$ almost surely), we simply write $\mathrm{Ex}(\vec m)$.
    \item Let $\calM$ be a randomized mechanism that given a tuple $(\vec x, \vec y) \in \xSpace \times \reals^k$ generates an incomplete ranking $\calM(\vec x, \vec y) \in \permSpace_{\leq k}$.
    An example oracle $\mathrm{Ex}(\vec m, \calE, \calM)$ with incomplete rankings, works as follows: Each time $\mathrm{Ex}(\vec m, \calE, \calM)$ is invoked, it returns a labeled example $(\vec x, \sigma) \in \mathbb X \times \permSpace_{\leq k}$, where (i) $\vec x \sim \calD_x$, $\vec \xi \sim \calE$ , (ii) $\vec y = \vec m(\vec x) + \vec \xi$ and (iii) $\sigma = \calM(\vec x, \vec y)$. Let $(\vec x, \sigma) \sim \calD_R^{\calM}$.
\end{itemize} 
\end{definition}
We denote $h : \xSpace \to \permSpace_k$ the composition $h = \mathrm{argsort} \circ \vec m$. Note that the oracle $\mathrm{Ex}(\vec m, \calE, \calM)$ generalizes $\mathrm{Ex}(\vec m, \calE)$ (which generalizes $\mathrm{Ex}(\vec m)$ accordingly) since we can set $\calM$ to be $\calM(\vec x, \vec y) = \mathrm{argsort}(\vec y)$.
\subsection{Problem Formulation and Contribution}
Most of our attention focuses on the two upcoming  learning goals, which are stated for the abstract Label Ranking example oracle $\calO \in \{\mathrm{Ex}(\vec m), \mathrm{Ex}(\vec m, \calE), \mathrm{Ex}(\vec m, \calE, \calM)\}$. Let $d$ be an appropriate ranking distance metric.

\begin{problem}
[Computational]
\label{problem:comp}
The learner is given i.i.d. samples from the oracle $\calO$ and its goal is to \emph{efficiently} output a hypothesis $\wh{h} : \mathbb X \to \permSpace_k$ such that with high probability the error $\E_{\vec x \sim \calD_x}[d(\wh{h}(\vec x), h(\vec x) )]$ is small. 
\end{problem}

\begin{problem}
[Statistical]
\label{problem:stat}
Consider the median problem $h^\star = \argmin_{h} \E_{(\vec x, \sigma)}[d(h(\vec x), \sigma)]$ where $(\vec x, \sigma) \sim \calD_R$. The learner is given i.i.d. samples from the oracle $\calO$ and its goal is to output a hypothesis $\wh{h} : \mathbb X \to \permSpace_k$ from some hypothesis class $\calH$ such that with high probability the error $\Pr_{\vec x \sim \calD_x}[\wh{h}(\vec x) \neq h^\star(\vec x)]$ against the median $h^\star$ is small. 
\end{problem}
The main gap in the theoretical literature of LR was the lack of computational guarantees. \Cref{problem:comp} identifies this gap and offers, in combination with the generative models of the previous section, a natural and formal way to study the theoretical performance of practical methods for LR such as decision trees and random forests. We believe that this is the main conceptual contribution of our work. In \Cref{problem:comp}, the runtime should be polynomial in $d,k,1/\eps$.

While \Cref{problem:comp} deals with computational aspects of LR, \Cref{problem:stat} focuses on the statistical aspects, i.e., the learner may be computationally inefficient. This problem is extensively studied as Ranking Median Regression \citep{clemenccon2018ranking, vogel2020multiclass} and is closely related to Empirical Risk Minimization (and this is why it is ``statistical'', since NP-hardness barriers may arise). We note that the median problem is defined w.r.t. $\calD_R$ (over complete rankings) but the learner receives examples from $\calO$ (which may correspond to incomplete rankings).

We study the distribution-free nonparametric LR task from either theoretical or experimental viewpoints in three cases:

\paragraph{Noiseless Oracle with Complete Rankings.} In this setting, we draw samples from $\mathrm{Ex}(\vec m)$ (i.e., $\mathrm{Ex}(\vec m , \calE)$ with $\vec \xi = 0$). For this case, we resolve \Cref{problem:comp} (under mild assumptions) and provide theoretical guarantees for efficient algorithms that use decision trees and random forests, built greedily based on the CART empirical MSE criterion, to interpolate the correct ranking hypothesis. 
This class of algorithms is widely used in applied LR but theoretical guarantees were missing. For the analysis, we adopt the  \emph{labelwise decomposition} technique \citep{cheng2013labelwise},
where we generate one decision tree (or random forest) for each position of the ranking.
We underline that decision trees and random forests are the state-of-the-art techniques for LR. 

\paragraph{Contribution 1.}
We provide the first theoretical performance guarantees for these algorithms for Label Ranking, under mild conditions. We believe that our analysis and the identification of these conditions contributes towards a better understanding of the practical success of these algorithms.

\paragraph{Noisy Oracle with Complete Rankings.} We next
replace the noiseless oracle of \Cref{problem:comp} with $\mathrm{Ex}(\vec m, \calE)$. In this noisy setting, the problem becomes challenging for theoretical analysis; we provide experimental evaluation aiming to quantify how noise affects the capability of decision trees and random forests to interpolate the true hypothesis.

\paragraph{Contribution 2.}
Our experimental evaluation demonstrates that random forests and shallow decisions trees are robust to noise, not only in our noisy setting, but also in standard LR benchmarks.

\paragraph{Noisy Oracle with Incomplete Rankings.} We consider the oracle $\mathrm{Ex}(\vec m , \calE, \calM)$ with incomplete rankings. We resolve \Cref{problem:stat} for the Kendall tau distance, as in previous works (so, we resolve it for the weaker oracles too). Now, the learner is agnostic to the positions of the elements in the incomplete ranking
and so labelwise decomposition cannot be applied. Using \emph{pairwise decomposition}, we compute a ranking predictor 
that achieves low misclassification error compared to the optimal classifier $h^\star$ and obtain sample complexity bounds for this task.
\paragraph{Contribution 3.}
Building on the seminal results of \citet{korba2017learning, clemenccon2018ranking,clemenccon2018aggregation,vogel2020multiclass}, we give results for \Cref{problem:stat} for incomplete rankings under appropriate conditions.

    
\subsection{Related Work}
\label{sec:relatedWork}
LR has received significant attention over the years \citep{shalev2007online,hullermeier2008label, cheng2008instance, har2003constraint}, due to the large number of practical applications. There are multiple  approaches for tackling this problem \citep[see][and the references therein]{vembu2010label, zhou2014taxonomy}. Some of them are based on probabilistic models \citep{cheng2008instance,cheng2010plackett, grbovic2012learning, zhou2014label}. Others are tree and ensemble based, such as adaption of decision trees \citep{cheng2009decision}, entropy based ranking trees and forests \citep{de2017label}, bagging techniques \citep{aledo2017tackling}, random forests \citep{zhou2018random}, boosting \citep{dery2020boostlr}, achieving highly competitive results. There are also works focusing on supervised clustering \citep{grbovic2013supervised}. Decomposition techniques are closely related to our work; they mainly transform the LR problem into simpler problems, e.g., binary or multiclass \citep{hullermeier2008label, cheng2012probability, cheng2013labelwise, cheng2013nearest, gurrieri2014alternative}.

\paragraph{Comparison to Previous Work.} To the best of our knowledge, there is no previous theoretical work focusing  on the computational complexity of LR
%
(a.k.a. \Cref{problem:comp}). However, there are many important works that adopt a statistical viewpoint. Closer to ours are the following seminal works on the statistical analysis of LR: \citet{korba2017learning, clemenccon2018ranking,  vogel2020multiclass}. \citet{korba2017learning} introduced the statistical framework of consensus ranking (which is the unsupervised analogue of \Cref{problem:stat}) and identified crucial properties for the underlying distribution in order to get fast learning rate bounds for empirical estimators.
A crucial contribution of this work (that we also make use of) is to prove that when Strict Stochastic Transitivity holds, the set of Kemeny medians (solutions of \Cref{problem:stat} under the KT distance) is unique and has a closed form. 
\Cref{problem:stat} was introduced in \citet{clemenccon2018ranking}, where the authors provide fast rates (under standard conditions) when the learner observes complete rankings, which reveal the relative order, but not the positions of the labels in the correct ranking. 
The work of \citet{vogel2020multiclass} provides a novel multiclass classification approach to Label Ranking, where the learner observes the top-label with some noise, i.e, observes only the partial information $\sigma_{\vec x}^{-1}(1)$ in presence of noise, under the form of the random label $y$ assigned to $\vec x$. Our contribution concerning \Cref{problem:stat} is a natural follow-up of these works where the learner observes noisy incomplete rankings (and so has only information about the relative order of the alternatives). Our solution for \Cref{problem:stat} crucially relies on the conditions and the techniques developed in \citep{korba2017learning, clemenccon2018ranking, vogel2020multiclass}. In our setting we have to modify the key conditions in order to 
handle incomplete rankings. Finally, our labelwise decomposition approach to \Cref{problem:comp} is closely related to \citet{korba2018structured}, where many embeddings for ranking data are discussed.

\paragraph{Nonparametric Regression and CART.} Regression trees constitute a fundamental approach in order to deal with nonparametric regression. Our work is closely related to the one of \citet{syrgkanis2020estimation}, which shows that trees and forests, built greedily based on the CART empirical MSE criterion, provably adapt to sparsity in the high-dimensional regime.
Specifically, 
\citet{syrgkanis2020estimation} analyze two greedy tree algorithms (they can be found at the \Cref{appendix:previous-results}): $(a)$ in the Level Splits variant, in each level of the tree, the
same variable is greedily chosen at all the nodes in order to
maximize the overall variance reduction; $(b)$ in the Breiman's variant, which is the most popular in
practice, the choice of the next variable to split on is locally decided at each node of the tree. In general, regression trees \citep{breiman1984classification} and random forests \citep{breiman2001random} are one
of the most widely used estimation methods by ML practitioners \citep{loh2011classification, louppe2014understanding}. For further literature review and preliminaries on decision trees and random forests, we refer to the \Cref{appendix:previous-results}. 

\paragraph{Multiclass Prediction.} In multiclass prediction with $k$ labels, there are various techniques such as One-versus-All and One-versus-One \citep[see][]{shalev2014understanding}. We adopt the OVO approach for \Cref{problem:stat}, where we consider $\binom{k}{2}$ binary sub-problems \citep{hastie1998classification, moreira1998improved, allwein2000reducing, furnkranz2002round, wu2004probability} and we combine the binary predictions. A similar approach was employed for a variant of \Cref{problem:stat} by \citet{vogel2020multiclass}.

\subsection{Notation}
For vectors, we use lowercase bold letters $\vec x$; let $x_i$ be the $i$-th coordinate of $\vec x$. We write $\poly_{\square}$ to denote that the degree of the polynomial depends on the subscripted parameters. Also, $\wt{O}(\cdot)$ is used to hide logarithmic factors.
We denote the symmetric group over $k$ elements with $\permSpace_k$ and $\permSpace_{\leq k}$ for incomplete rankings. For $ i \in [k]$, we let $\sigma(i)$ denote the position of the $i$-th alternative. The \textbf{Kendall Tau} (KT) distance $d_{KT}(\pi, \sigma) = \sum_{i < j} \vec 1\{ (\pi(i) - \pi(j))(\sigma(i) - \sigma(j)) < 0 \}$ and the \textbf{Spearman} distance $d_2(\pi, \sigma) = \sum_{i \in [k]} (\pi(i) - \sigma(i))^2$. Also, $k_{\tau}$ stands for the \textbf{KT coefficient}, i.e., the normalization of Kendall tau distance to the interval $[-1, 1]$ which measures the proportion of the concordant pairs in two rankings.
The \textbf{Mean Squared Error} (MSE) of a function $f : \{0,1\}^d \to [0,1]$ is equal to 
\begin{equation}
\label{eq:mse}
\wt{L}(f,S) = \E_{\vec x \sim \calD_x}
    \!\!\Big[
    \big(
    f(\vec x) -\! 
    \E_{\vec w \sim \calD_x}[f(\vec w) | \vec w_S = \vec x_S] 
    \big)^2
    \Big] ,    
\end{equation}
where $\vec x_S$ is the sub-vector of $\vec x$, where we observe only the coordinates with indices in $S \subseteq [d]$ and $\vec x_S \in \{0,1\}^{|S|}$. The VC dimension $\mathrm{VC}(\calG)$ of a class $\mathcal{G} \subseteq \{-1,+1\}^{\xSpace}$ is the largest $n$ such that there exists a set $T \subset \xSpace, |T| = n$ and $\mathcal{G}$ shatters $T$ \citep{shalev2014understanding}. When $\mathrm{VC}(\calG) < \infty$, $\calG$ is said to be a \textbf{VC class}.

\section{Our Results}
\label{sec:nonparametric-theory}
We provide an overview of our contributions on 
distribution-free Label Ranking settings, as introduced in  \Cref{def:distr-free-lr}.
\subsection{Noiseless Oracle with Complete Rankings}
\label{section:full}
We begin with Label Ranking as \emph{noiseless nonparametric regression} \citep{tsybakov_book}.
This corresponds to the example oracle $\mathrm{Ex}(\vec m)$ of \Cref{def:distr-free-lr} which we recall now: For an underlying score hypothesis $\vec m : \xSpace \to [0,1]^k$, where $k$ is the number of labels and $m_i(\vec x)$ is the score of the alternative $i \in [k]$ with respect to $\vec x$. The learner observes a labeled example ($\vec x, \sigma) \sim \calD_R$. It holds that $\sigma = h(\vec x) = \mathrm{argsort}(\vec m(\vec x))$.

We resolve \Cref{problem:comp} for the $\mathrm{Ex}(\vec m)$ oracle: We provide the first theoretical guarantees in the LR setting for the performance of algorithms based on decision trees and random forests, when the feature space is the Boolean hypercube $\xSpace = \{0,1\}^d$ under mild assumptions, using the labelwise decomposition technique \citep{cheng2013labelwise}. 
We underline once again that this class of algorithms 
constitutes a fundamental tool for \emph{practical works} to solve LR; this heavily motivates the design of our theory. 
We focus on the performance of regression trees and forests in high dimensions. 
Crucially, \Cref{def:distr-free-lr} makes no assumptions on the structure of the underlying score hypothesis $\vec m$.
In order to establish our theoretical guarantees, we are going to provide a pair of structural conditions for the score hypothesis $\vec m$ and the features' distribution $\calD_x$. We will now state these conditions; for this we will need the definition of the mean squared error that can be found at \eqref{eq:mse}.
\begin{condition}
\label{condition:full}
Consider the feature space $\xSpace = \{0,1\}^d$ and the regression vector-valued function $\vec m  : \{0,1\}^d \to [0,1]^k$ with $\vec m = (m_1,\ldots,m_k)$. Let $\calD_x$ be the distribution over features. We assume that the following hold for any $j \in [k]$.
\begin{enumerate}
    \item \label{item:Sparsity}(Sparsity) The function $m_j : \{0,1\}^d \to [0,1]$ is $r$-sparse, i.e., it depends on $r$ out of $d$ coordinates.
    \item \label{item:ApproximateSubmodularity}(Approximate Submodularity) The mean squared error $\wt{L}_j$ of $m_j$   
    is $C$-approximate-submodular, i.e., for any $S \subseteq T \subseteq [d],$ $i \in [d]$, it holds that
    \[
    \wt{L}_j(T) - \wt{L}_j(T \cup \{i\})
    \leq 
    C \cdot \left (\wt{L}_j(S) - \wt{L}_j(S \cup \{i\}) \right)\,.
    \]
\end{enumerate}
\end{condition}
\noindent Some comments are in order: (1) 
The approximate submodularity condition for the mean squared error is the more technical condition, which however is provably \emph{necessary} \citep{syrgkanis2020estimation} to obtain meaningful results about the consistency of greedily grown trees in high dimensions.
(2) For the theoretical analysis, we constrain ourselves to the case where all features are binary. However, in the experimental part, we test the performance of the method with non-binary features too.
(3) Sparsity should be regarded as a way to parameterize the class of functions $\vec m$, rather than a restriction. Any function is $r$-sparse, for some value of $r$. However, our results are interesting when $r \ll d$ and establish that decision trees and random forests \emph{provably behave well under sparsity}.
As we will see in \Cref{infthm:main-label-rank}, the sample complexity has an $r^r$ dependence, which cannot be avoided, since the class of functions is nonparametric. Observe that both $m_i$ and $m_j$ are $r$-sparse but they are not constrained to depend on the same set of coordinates; the function $\vec m$ is at most $(k \cdot r)$-sparse, where $k \cdot r \ll d$, and we say that $\vec m$ is $\vec r$-sparse. These are the state-of-the-art conditions for the high-dimensional regime~\citep{syrgkanis2020estimation}.

Our algorithm for \Cref{problem:comp} uses decision trees via the Level Splits criterion. In this criterion, a set of splits $S \subseteq [d]$ is collected greedily and any
tree level has to split at the same direction $i \in [d]$.
Intuitively, the approximate submodularity condition captures the following phenomenon: ``If adding $i$ does not decrease the mean squared error significantly at some point (when having the set $S)$, then $i$ cannot decrease the mean squared error significantly in the future either (for any superset of $S$)''. Our main result for \Cref{problem:comp} using decision trees with Level Splits follows. Recall that $h(\vec x) = \mathrm{argsort}(\vec m(\vec x))$ and $d_2$ is the Spearman distance (i.e., $L_2$ squared over the rankings' positions).


\begin{theorem}[Noiseless LR (Informal)]
\label{infthm:main-label-rank}
Under \Cref{condition:full} with parameters $r, C$, 
there exists an algorithm (Decision Trees via Level-Splits - \Cref{algo:informal-level-split-rank}) 
that draws $n = \wt{O}\left( 
\log(d) \cdot \poly_{C,r}(k) \cdot (Cr /\epsilon)^{Cr + 2} \right)$ independent samples from $\mathrm{Ex}(\vec m)$ and, in $\poly_{C,r}(d,k,1/\epsilon)$ time, computes
a set of splits $S_n$ and
an estimate $h^{(n)}(\cdot~; S_n) : \{0,1\}^d \to \permSpace_k$ which, with probability $99\%$, satisfies
    $
    \E_{\vec x \sim \calD_x}
    \left [ d_{2}(
    h(\vec x), h^{(n)}(\vec x; S_n) ) 
    \right] 
    \leq \epsilon\,.
    $
\end{theorem}
See also \Cref{thm:main-label-rank}. This is the first sample complexity guarantee in LR for decision tree-based algorithms. We also provide results and algorithms for the Breiman's criterion (\Cref{sec:full-proof-dt}) and for Random Forests (\Cref{sec:full-proof-rf}). Our result can be read as: \emph{In practice, sparsity of the instance's ``score function'' is one of the reasons why such algorithms work well and efficiently in real world.} 

The description of \Cref{algo:informal-level-split-rank} follows.
Given $(\vec x, \sigma) \sim \calD_R$, we transform the ranking-label $\sigma$ to a vector $\vec y = \vec m_C(\sigma) \in [0,1]^k$, where $\vec m_C(\sigma)$ is the canonical representation of the ranking $\sigma$. Specifically, one can obtain the score vector $\vec y$ by setting $y_i = m_{C,i}(\sigma)$ equal to $\sigma(i)/k$, i.e.,
the position of the $i$-th alternative in the permutation $\sigma$, normalized by $k$, where $k$ is the length of the permutation (see Line 3 of \Cref{algo:informal-level-split-rank}).
Hence, we obtain a training set of the form $T = (\vec x^{(i)}, \vec y^{(i)})_{i \in [N]}$. Our goal is to fit the score vectors $\vec y^{(i)}$ using decision trees (or random forests depending on the black-box algorithm that we will choose to apply). 
During this step, we have to feed our training set $T$ into a learning algorithm that fits the function $\vec m_C \circ h : \xSpace \to [0,1]^k$, where $\circ$ denotes composition. We remark that since the regression function $\vec m$ is sparse, this vector-valued function is sparse too. We do this as follows. 

\begin{algorithm}[ht!] 
\caption{Algorithm of \Cref{infthm:main-label-rank}}
\label{algo:informal-level-split-rank}
\begin{algorithmic}[1]
%
%
\STATE Set $n \gets \wt{\Theta}(\log(d/\delta) \cdot \poly_{C,r}(k) \cdot (Cr/\eps)^{Cr+2})$
\STATE Draw $n$ samples $(\vec x^{(j)}, \sigma^{(j)}) \sim \calD_R, j \in [n]$
\STATE For any $j \in [n]$, set $\vec y^{(j)} \gets (\sigma^{(j)}(i)/k)_{i \in [k]}$ 
\STATE Create $k$ datasets $T_i = \{ (\vec x^{(j)}, y_i^{(j)}) \}_{j \in [n]}$
\STATE \textbf{for} $i \in [k]$ \textbf{do}
\STATE ~~~$m_i^{(n)}, S_n^{(i)} = \texttt{LevelSplits}(T_i, \wt{\Theta}(Cr\log(k)))$ 
\STATE \textbf{endfor}
\STATE Output $\mathrm{argsort} \circ (m_1^{(n)}(\cdot; S_n^{(1)}),...,m_k^{(n)}(\cdot; S_n^{(k)}))$
\end{algorithmic}
\end{algorithm}

We decompose the training set $T$ into $k$ data sets $T_i$, where the labels are no more vectors but real values (\emph{labelwise decomposition}, see Line 4 of \Cref{algo:informal-level-split-rank}). For each $T_i$, we apply the Level Splits method and finally we combine our estimates, where we have to aggregate our estimates into a ranking (see Line 8 of \Cref{algo:informal-level-split-rank}). \Cref{algo:informal-level-split-rank} uses the routine \texttt{LevelSplits}, which computes a decision tree estimate based on the input training set (see \Cref{algo:level-split} (with $h=0$) in \Cref{appendix:previous-results}). The second argument of the routine is the maximum number of splits $H$ (height of the tree) and in the above result $H = \wt{\Theta}(Cr\log(k))$. The routine iterates $H$ times, one for each level of the tree: at every level, we choose the direction $i \in [d]$ that minimizes the total empirical mean
squared error (greedy choice) and the space is partitioned recursively based on whether $x_i = 0$ or $1$. The routine outputs the estimated function and the set of splits $S \subseteq [d]$. For a proof sketch, see \Cref{sec:full-sketch}.

\subsection{Noisy Oracle with Complete Rankings}
We remark that the term `noiseless' in the above regression problem is connected with the output (the ranking), i.e., in the generative process given $\vec x \in \mathbb{X}$, we will constantly observe the same output ranking. Let us consider the oracle $\mathrm{Ex}(\vec m, \calE)$, that corresponds to \emph{noisy nonparametric regression}:
Draw $\vec x \sim \calD_x$ and independently draw $\vec \xi \in [-1/4,1/4]^k$ from a zero mean noise distribution $\calE$. Compute $\vec y = \vec m(\vec x) + \vec \xi$,
rank the alternatives $\sigma = \mathrm{argsort}(\vec y)$ and output $(\vec x, \sigma)$. Due to the noise vector $\vec \xi$, we may observe e.g., different rankings for the same $\vec x$ feature. 
We provide the following notions of inconsistency. 
\begin{definition}
[Output Inconsistency]
\label{definition:noise-consistency}
Let $\sigma = \mathrm{argsort}( \vec m(\vec x) + \vec \xi)$ denote the observed ranking of $\mathrm{Ex}(\vec m, \calE)$.
The noise distribution $\calE$ satisfies: \begin{enumerate}
    \item[(i)] the $\alpha$-inconsistency property if there exists $\alpha \in [0,1]$
so that $
\E_{\vec x \sim \calD_x} \left[\Pr_{\vec \xi \sim \calE}[ h(\vec x) \neq \sigma ] \right] \leq \alpha$, and
\item[(ii)] the $\beta$-$k_{\tau}$ gap property if there exists $\beta \in [-1,1]$ so that $\E_{\vec x \sim \calD_x}\E_{\vec \xi \sim \calE}[ k_{\tau}(h(\vec x), \sigma)) ] = \beta.$
\end{enumerate}
\end{definition}
Property (i) captures the phenomenon that the probability that the observed ranking $\sigma$ differs from the correct one $h(\vec x)$ is roughly $\alpha$ over the feature space. However, it does not capture the distance between these two rankings; this is why we need property (ii) which captures the expected similarity. The case $\alpha = 0$ (resp. $\beta = 1$) gives our noiseless setting. When $\alpha > 0$ (resp. $\beta < 1$), the structure of the problem changes and our theoretical guarantees fail. Interestingly, this is due to the fact that the geometry of the input noise is structurally different from the observed output. The input noise acts additively to the vector $\vec m(\vec x)$, while the output is computed by permuting the elements. Hence, the relation between the observed ranking and the expected one is no more linear and hence one needs to extend the standard `additive' nonparametric regression setting $y = f(x) + \xi$ to another geometry dealing with rankings $\sigma = \calE_{\theta} \circ f(x)$, where $f : \xSpace \to \permSpace_k$ is the regression function and $\calE_{\theta} : \permSpace_k \to \permSpace_k$ is a parameterized noise operator (e.g., a Mallows model). This change of geometry is interesting and to the best of our knowledge cannot be captured by existing theoretical results. Our experimental results aim to complement and go beyond our theoretical understanding of the capability of the decision trees and random forests to interpolate the correct underlying regression function with the presence of regression noise. To this end, we consider the oracle $\mathrm{Ex}(\vec m, \calE)$ where the noise distribution $\calE$ satisfies either property (i) or (ii) of \Cref{definition:noise-consistency}.

For the experimental evaluation, two synthetic data set families were used, namely LFN (Large Features Number) and SFN (Small Features Number), consisting of a single noiseless and 50 noisy data sets, respectively. For either data set family, a common $\vec m : \{0,1\}^d \to [0,1]^k$ was employed, accordingly. Each noiseless data set was created according to the oracle $\mathrm{Ex}(\vec m)$. It consists of 10000 samples $(\vec x, \sigma)$, where $\vec x \in \{0,1\}^{d}$ ($d =1000$ for LFN and $d = 100$  for SFN) and $\sigma \in \permSpace_5$ $(k=5)$, with $r=10$ informative binary features per label (sparsity). The noisy data sets were produced according to the generative process $\mathrm{Ex}(\vec m, \calE)$, each using a different zero-mean noise distribution $\calE$. We implemented modified versions of Algorithm \ref{algo:informal-level-split-rank}. The shallow trees (\textcolor[RGB]{34,136,51}{\textbullet},\textcolor[RGB]{204,187,68}{\textbullet}), fully grown decision trees (\textcolor[RGB]{68,119,170}{\textbullet},\textcolor[RGB]{102,204,238}{\textbullet}) and random forests (\textcolor[RGB]{170,51,119}{\textbullet},\textcolor[RGB]{238,102,119}{\textbullet}) were built greedily based on the CART empirical MSE criterion, using the Breiman's method instead of Level Splits. Results are obtained in terms of mean $k_{\tau}$, using the noisy data as training set and noiseless data as validation set. Figure \ref{fig:experiment1} summarizes the experimental results for different values of $\alpha \in [0,1]$ and $\beta \in [-1,1] $.

\begin{figure}[!ht]
\centerline{\includegraphics[width = 13cm]{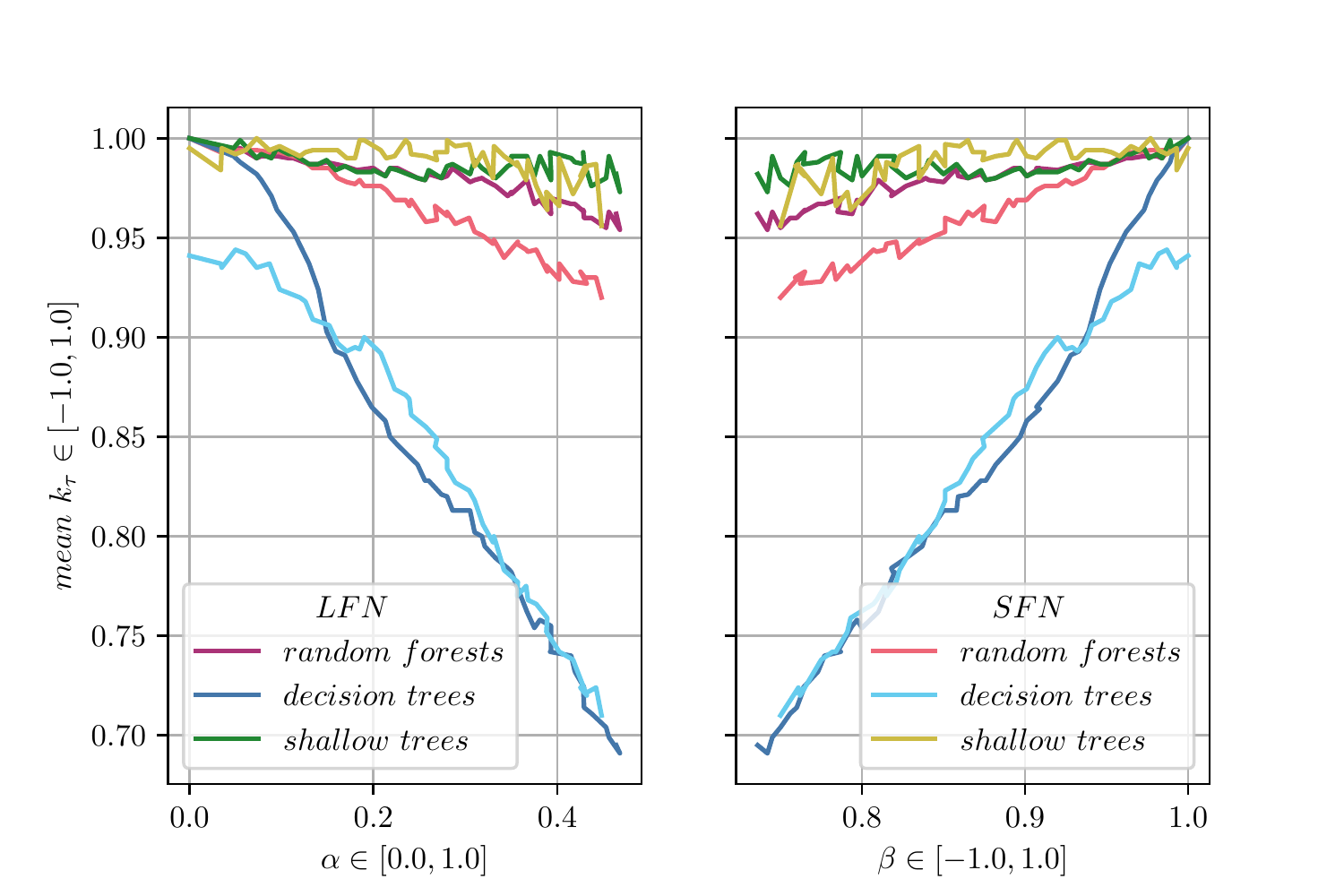}}
\caption{Experimental results in terms of mean  $k_{\tau}$ for different noise distributions $\calE$ w.r.t. (a) $\alpha$-inconsistency; (b) $\beta$-$k_{\tau}$ gap.}
\label{fig:experiment1}
\end{figure}
As expected, decision trees as well as random forests interpolate the $\vec m$ function successfully in the noiseless setting, since it is $\vec r$-sparse. Moreover, the increase of noise level leads to the decay of the decision trees' performance, indicated by the $\alpha$-inconsistency. However, the $\beta$-$k_{\tau}$ gap is a more appropriate noise level measure, because it quantifies the degree of deviation rather than the existence of it. The ratio of the performance in terms of mean $k_{\tau}$ over $\beta$-$k_{\tau}$ gap is approximately equal to one, revealing that decision trees fit the noise. On the contrary, shallow trees have better ability to generalize and avoid overfitting. Fully grown honest random forests are also resistant to overfitting due to bagging, and therefore are noise tolerant. The experimental results for different `additive' nonparametric noise settings and for LR standard benchmarks can be found in the \Cref{sec:experimental}. The code to reproduce our results is available  \href{https://anonymous.4open.science/r/LR-nonparametric-regression-BC28/}{here}.

\begin{remark}
We have encountered \Cref{problem:comp} with the oracles $\mathrm{Ex}(\vec m)$ and $\mathrm{Ex}(\vec m, \calE)$. A natural question is what happens if we use the incomplete oracle $\mathrm{Ex}(\vec m, \calE, \calM)$. In this setting, the position of each alternative is not correctly observed (we observe a ranking with size $\l \leq k$, see \Cref{sec:incomplete}). Hence, we cannot apply our \emph{labelwise} method for this oracle. One could obtain similar results for \Cref{problem:comp} for incomplete rankings using pairwise decomposition, but we leave it for future work. Next we focus on \Cref{problem:stat}.  
\end{remark}

\subsection{Noisy Oracle with Incomplete Rankings}
\label{sec:incomplete}
We now study \Cref{problem:stat}, where we consider a metric $d$ in $\permSpace_k$, a distribution $\calD_R$ over $\xSpace \times \permSpace_k$ that corresponds to the example oracle $\mathrm{Ex}(\vec m, \calE)$ and set up the task of finding
a measurable mapping $h : \xSpace \to \permSpace_k$ that minimizes the objective $R(h) = \E_{(\vec x, \sigma) \sim \calD_R}[d(h(\vec x), \sigma)]$. 
In this work, we focus on the Kendall tau distance $(d = d_{KT})$ and ask how well we can estimate the minimizer of the above population objective if we observe i.i.d. samples from the incomplete rankings' oracle $\mathrm{Ex}(\vec m, \calE, \calM)$.
We underline that in what follows whenever we refer to \Cref{problem:stat}, we have $d = d_{KT}$ in mind.
A natural question is: What is the optimal solution?
In binary classification, the learner aims to estimate the Bayes classifier, since it is known to minimize the misclassification error among all classifiers \citep{massart2006risk}.
\Cref{problem:stat} deals with rankings and is well-studied in previous works when: $d$ is the Kendall tau distance and the learner either observes complete rankings \citep{clemenccon2018ranking} or observes only the top element (under some BTL noise) \citep{vogel2020multiclass}. As we will see later, the optimal solution $h^\star$ of \Cref{problem:stat} is unique under mild conditions on $\calD_R$, due to \citet{korba2017learning}. Our goal will be to estimate $h^\star$ from labeled examples generated by $\mathrm{Ex}(\vec m, \calE, \calM)$.

 We are going to introduce the example oracle $\mathrm{Ex}(\vec m, \calE, \calM)$:
We are interested in the case where the mechanism $\calM$ generates incomplete rankings and captures
a general spectrum of ways to generate such rankings \citep[e.g.,][]{hullermeier2008label}. We begin with the incomplete rankings' mechanism $\calM$.
We assume that
there exists a survival probabilities vector  $\vec q: \xSpace \to [0,1]^k$, which is feature-dependent, i.e., the vector $\vec q$ depends on the input  $\vec x \in \xSpace$. Hence, for the example $\vec x$ and alternative $i \in [k]$, with probability $q_i(\vec x)$, we set the score $y_i$ equal to a noisy value of the score $m_i(\vec x)$ (the alternative $i$ survives)
and, otherwise, we set $y_i = \star$.
We mention that the events of observing $i$ and $j$ are not necessarily independent and so do not necessarily occur with probability $q_i(\vec x)q_j(\vec x)$. We denote the probability of the event ``Observe both $i$ and $j$ in $\vec x$'' by $q_{i,j}(\vec x).$
We modify the $\mathrm{argsort} : [0,1]^k \to \permSpace_k$ routine so that it will ignore the $\star$ symbol in the ranking, e.g., $\mathrm{argsort}(0.4, \star, 0.7, \star, 0.1) = (c \succ a \succ e).$ Crucially, we remark that another variant would preserve the $\star$ symbols: this problem is easier since it reveals the correct position of the non-erased alternatives. In our model, the information about the location of each alternative is not preserved. In order to model regression noise, we consider a noise distribution $\calE$ over the bounded cube $[-1/4, 1/4]^k$. Hence, we model $\mathrm{Ex}(\vec m, \calE, \calM)$ as follows: 

\begin{definition}
[Generative Process for Incomplete Data]
\label{def:regression-incomplete}
Consider an underlying score hypothesis $\vec m : \xSpace \to [0,1]^k$ and let $\calD_x$ be a distribution over features.
Let $\vec y : [k] \to [0,1] \cup \{\star\}$ and consider the survival probabilities vector $\vec q : \xSpace \to [0,1]^k$.
Each sample $(\vec x, \sigma) \sim \calD_R^{\vec q}$ is generated as follows: $(i)$ Draw $\vec x \in \xSpace$ from $\calD_x$ and $\vec \xi \in [-\frac{1}{4},\frac{1}{4}]^k$ from $\calE$; $(ii)$ draw $\vec q(\vec x)$-biased coins $\vec c \in \{-1,+1\}^k$; $(iii)$ if $c_i > 0$, set $y_i = m_i(\vec x) + \xi_i$, else $y_i = \star$; $(iv)$ compute $\sigma = \mathrm{argsort}(\vec y)$, ignoring the $\star$ symbol.
\end{definition} 

In what follows, we resolve \Cref{problem:stat} for the example oracle of \Cref{def:regression-incomplete} and $d = d_{KT}$. As in the complete ranking case, \Cref{def:regression-incomplete} imposes restrictions neither on the structure of the true score hypothesis $\vec m$ nor on the noise distribution $\calE$. In order to resolve \Cref{problem:stat}, we assume the following. Recall that $q_{i,j}(\vec x)$ is the probability of the event ``Observe both $i$ and $j$ in $\vec x$''.
\begin{condition}
\label{cond:incomplete}
Let $p_{ij}(\vec x) = \Pr_{\vec \xi \sim \calE}[m_i(\vec x) + \xi_i > m_j(\vec x) + \xi_j | \vec x]$ for $\vec x \in \reals^d$. For any $1 \leq i < j \leq k $, we assume that the following hold. 
\begin{enumerate}
    \item \label{item:ST} (Strict Stochastic Transitivity) For any $\vec x \in \reals^d$ and any $u \in [k]$, we have that $p_{ij}(\vec x) \neq 1/2$ and $(p_{iu}(\vec x) > 1/2 \land p_{uj}(\vec x) > 1/2) \Rightarrow p_{ij}(\vec x) > 1/2$.
    \item \label{item:ThybakovNoise}(Tsybakov's Noise Condition) There exists $a \in [0,1]$ and $B > 0$ so that the probability that a random feature $\vec x \sim \calD_x$ satisfies
$
\left| p_{ij}(\vec x) - 1/2 \right| < 2t\,,
$
is at most $B \cdot t^{a /(1-a)}$ for all $t \geq 0.$
    \item \label{item:deletionTolerance}(Deletion Tolerance) There exists $\phi  \in (0,1]$ so that $q_{i,j}(\vec x) \geq \phi$ for any $\vec x \in \reals^d$, where $q_{i,j}(\vec x)$ is the survival probability of the pair $i<j$ in $\vec x$.
\end{enumerate}
\end{condition}

\noindent\underline{Importance of \Cref{item:ST}:} Our goal is to find a good estimate for the minimizer $h^\star$ of the loss function ${R(h) = \E_{(\vec x, \sigma) \sim \calD_R}[d_{KT}(h(\vec x), \sigma)]}$. As observed in previous works \citep{korba2017learning,clemenccon2018ranking}, this problem admits a unique solution (with closed form) under mild assumptions on $\calD_R$ and specifically this is assured by Strict Stochastic Transitivity (SST) of the pairwise probabilities $p_{ij}(\vec x)$. The first condition guarantees \citep{korba2017learning} that the minimizer of $R(h)$ is almost surely unique and is given, with probability one and for any $i \in [k]$, by 
\begin{equation}
    \label{eq:bayes}
    h^\star(\vec x; i) = 1 + \sum_{j \neq i} \vec 1\{ p_{ij}(\vec x) <  1/2\}\,.
\end{equation}
This is the well-known Copeland rule and
we note that SST is satisfied by most natural probabilistic ranking models.\\

\noindent\underline{Importance of \Cref{item:ThybakovNoise}:} The Tsybakov's noise condition is standard in binary classification and corresponds to a realistic noise model \citep{boucheron2005theory}. 
In its standard form, this noise condition naturally requires that the regression function of a binary problem
$ \eta(\vec x) = \E[Y |X =\vec x] = 2\Pr[Y=+1 | X =  \vec x]-1
$
is close to the the critical value $0$ with low probability over the features, i.e., the labels are not completely random for a sufficiently large portion of the feature space. We consider that, for any two alternatives $i \neq j$, this condition must be satisfied by the true score function $\vec m$ and the noise distribution $\calE$ (specifically the functions $m_i$ and $m_j$ and the random variables $\xi_i, \xi_j$). This condition is very common in binary classification since it guarantees ``fast rates'' and has been previously applied to LR \citep{vogel2020multiclass}. \\

\noindent\underline{Importance of \Cref{item:deletionTolerance}:} The last condition is natural in the sense that we need to observe the pair $(i,j)$ at some frequency in order to achieve some meaningful results. 
Variants of this condition have already appeared in the incomplete rankings literature \citep[see e.g.,][]{fotakis2021aggregating} and in previous works in Label Ranking \citep[see][]{vogel2020multiclass}. We note that if we relax this condition to state that we only observe the pair $i,j$ only in a portion of $\reals^d$ (e.g., $q_{ij}(\vec x) = 0$ for 40\% of $\vec x$'s), then we will probably miss a crucial part of the structure of the underlying mapping. This intuitively justifies the reason that we need  deletion tolerance to hold for any $\vec x \in \reals^d$.

Having described our conditions, we continue with our approach. We first remind the reader that the labelwise perspective we adopted in the complete case now fails.
Second, since our data are incomplete, the learner cannot recover the optimal ranking rule $h^\star = \argmin_{h} \E[d_{KT}(h(\vec x), \sigma)]$ by simply minimizing an empirical version of this objective. To tackle this problem, we adopt a pairwise comparisons approach. In fact, the key idea is the closed form of the optimal ranking rule: From \eqref{eq:bayes}, we can write $h^\star(\vec x; i) = 1 + \sum_{j \neq i} \vec 1\{ h_{ij}^\star(\vec x) = -1\}$ where $h_{ij}^\star$ is the Bayes optimal classifier of the \emph{binary sub-problem of the pair $i \neq j$}. We provide our result based on the standard One-Versus-One (OVO) approach, reducing this complex problem into multiple binary ones: We reduce the ranking problem into $O(k^2)$ binary sub-problems and each sub-problem corresponds to a pairwise comparison between the alternatives $i $ and $j$ for any $1 \leq i < j \leq k$.
We solve each sub-problem separately by obtaining the Empirical Risk Minimizer $\wh{h}_{ij}$ (for a fixed VC class $\calG$, as in the previous works) whose risk is compared to the optimal $h^\star_{ij}$ and then we aggregate the $\binom{k}{2}$ binary classifiers into a single output hypothesis $\wh{h} : \reals^d \to \permSpace_k$. We compare the generalization of this empirical estimate $\wh{h}$ with the optimal predictor $h^\star$ of \eqref{eq:bayes}.
We set $L_{ij}(g) = \E[ g(\vec x) \neq \sgn(\sigma(i)-\sigma(j)) | \sigma \ni \{i,j\}]$ where $(\vec x, \sigma) \sim \calD_R^{\vec q}$. Our main result in this setting follows.
\begin{theorem}
[Noisy and Incomplete LR]
\label{thm:lr-main-incomplete}
Let $\epsilon, \delta \in (0,1).$
Consider a hypothesis class $\calG$ of binary classifiers with finite VC dimension. Under \Cref{cond:incomplete} with parameters $a,B, \phi$,
there exists an algorithm (\Cref{algo:ovo-inc}) that draws
\[
n = 
\wt{O} \left (\frac{k^{\frac{4(1-a)}{a}}}{\poly_{a}(\phi \cdot \eps)} \cdot \max \left \{ \log \left(\frac{k}{\delta}\right), \mathrm{VC}(\calG) \right \} \right)
\]
samples from $\calD_R^{\vec q}$, as in \Cref{def:regression-incomplete},
and computes an estimate $\wh{h} : \reals^d \to \permSpace_k$ so that $\Pr_{\vec x \sim \calD_x} [\wh{h}(\vec x) \neq h^\star(\vec x)]$ is, with probability $1-\delta$, at most
\begin{equation}
\label{eq:pac-result}
\frac{C_{a,B}}{ \phi^2}  \left( 2 \sum_{i < j} \left(\inf_{g \in \calG} L_{i,j}(g) - L_{i,j}^\star \right)^a \right) +
\eps\,, 
\end{equation}
where $h^\star$ is the optimal predictor of \eqref{eq:bayes} and $L_{i,j}^\star$ is the loss of the binary Bayes classifiers $h_{i,j}^\star$ for $1\leq i<j \leq k$, where $C_{a,B}$ is a constant depending on $a, B.$
\end{theorem}
Our result for incomplete rankings is a PAC result, in the sense that we guarantee that, when optimizing over a VC class $\calG$, 
the gap between the empirical estimate (the algorithm's output) and the optimal predictor of \eqref{eq:bayes} is at most $C \cdot \mathrm{OPT} + r_{n}(\delta)$, where
$\mathrm{OPT}$ is (a function of) the gap between the best classifier in the class $(\argmin_{g \in \calG} L(g))$ and the Bayes classifier and $r_n(\delta)$ is a function which tends to $0$ as the number of samples $n$ increases (see \eqref{eq:pac-result}).
We remark that the algorithm does not come with a computational efficiency guarantee, since the results are based on the computation of the ERM of each pairwise comparison $i < j$. In general, this is NP-hard but if the binary hypothesis class $\calG$ is ``simple'' then we also obtain computational guarantees.
\begin{algorithm}[ht!] 
\caption{Algorithm of \Cref{thm:lr-main-incomplete}}
\label{algo:ovo-inc}
\begin{algorithmic}[1]
\STATE $T \gets $ $n$ i.i.d. samples $(\vec x, \sigma) \sim \calD_R^{\vec q}$ (as in  \Cref{thm:lr-main-incomplete})
\STATE For any $i \neq j,$ set $T_{ij} = \emptyset$
\STATE \textbf{for} $1 \leq i < j \leq k$ \textbf{do}
\STATE ~~~~ \textbf{if} $(\vec x, \sigma) \in T$ and $\sigma \ni \{i,j\}$ \textbf{then}
\STATE ~~~~~~~~ Add $(\vec x, \mathrm{sgn}(\sigma(i) - \sigma(j)))$ to $T_{ij}$
\STATE ~~~~ \textbf{endif}
\STATE \textbf{endfor}
\STATE $\wh{\vec s} \gets $ \texttt{EstimateAggregate}($T_{ij}$ for $i < j$, $\calG$) 
\STATE On input $\vec x \in \reals^d,$ output $\mathrm{argsort}(\wh{\vec s}(\vec x))$ breaking arbitrarily possible ties.
\end{algorithmic}
\end{algorithm}

\Cref{algo:ovo-inc} works as follows: Given a training set $T$ of the form $(\vec x^{(i)}, \sigma^{(i)})$ with incomplete rankings, the algorithm creates $\binom{k}{2}$ datasets $T_{ij}$ with the following criterion: For any $i < j$, if $(\vec x, \sigma) \in T$ and $\sigma \ni \{i,j\}$, the algorithm adds to the dataset $T_{ij}$ the example $(\vec x, \mathrm{sgn}(\sigma(i) - \sigma(j)))$.
For any such binary dataset, the algorithm computes the ERM and aggregates the estimates to $\wh{\vec s}$ (these routines can be found as \Cref{algo:estim-aggr}). This aggregate rule is based on the structure of the optimal classifier $h^\star$ (that is valid due to the SST condition). The final estimator is the function $\wh{h}$ that, on input $\vec x \in \reals^d$, outputs the ranking $\wh{h}(\vec x) = \mathrm{argsort}(\wh{\vec s}(\vec x))$ (by breaking ties randomly).

\begin{remark}
(i) The oracle $\mathrm{Ex}(\vec m, 
\calE, \calM)$ and \Cref{def:regression-incomplete} can be adapted to handle \emph{partial} rankings (see \Cref{sec:partial}).
(ii) Theorem 2.4 directly controls the risk gap $R(\wh{h}) - R(h^\star) \leq \E_{\vec x}[d_{KT}(\wh{h}(\vec x), h^\star(\vec x))]$, since $d_{KT}(\pi, \sigma) \leq k^2 \vec 1\{\pi \neq \sigma\}$.
(iii) We studied \Cref{problem:stat} for $\mathrm{Ex}(\vec m, \calE, \calM)$. Our results can be transferred to the oracles $\mathrm{Ex}(\vec m, \calE)$ and $\mathrm{Ex}(\vec m)$ under \Cref{cond:incomplete} with $\phi = 1$.
(iv) This result is similar to \citet{vogel2020multiclass}, where the learner observes $(\vec x, y)$ where $y = \sigma_{\vec x}^{-1}(1)$ is the top-label. To adapt our incomplete setting to theirs, we must erase positions instead of alternatives. If $\wt{q}_i$ is the probability that the $i$-th position survives, then we have that, in  \citet{vogel2020multiclass}: $\wt{q}_1(\vec x) = 1$ and $\wt{q}_{i \neq 1}(\vec x) = 0$ for all $\vec x \in \reals^d$. 
Also, the generative processes of the two works are different (noisy nonparametric regression vs. Plackett-Luce based models). In general, our results and our analysis for \Cref{problem:stat} are very closely related and rely on the techniques of \citet{vogel2020multiclass}.
\end{remark}

\section{Technical Overview}
\textbf{Proof Sketch of  \texorpdfstring{\Cref{infthm:main-label-rank}}{Informal Theorem \ref{infthm:main-label-rank}}.}
\label{sec:full-sketch}
Our starting point is the work of \citet{syrgkanis2020estimation}, where they provide a collection of nonparametric regression algorithms based on decision trees and random forests. We have to provide a vector-valued extension of these results. For decision trees, we show (see \Cref{thm:score}) that if the learner observes i.i.d. samples $(\vec x, \vec m(\vec x) + \vec \xi)$ for some unknown target $\vec m$ satisfying \Cref{condition:full} with $r, C > 0$, then there exists an algorithm $\texttt{ALGO}$ that uses
$n = \wt{O}\left( \log(d)
    \cdot \poly_{C,r}(k C r/\eps) \right)$
samples and computes an estimate $\vec m^{(n)}$ which satisfies
    $
    \E_{\vec x \sim \calD_x}
    \left [
    \left\| \vec m(\vec x) - \vec m^{(n)}(\vec x) \right\|_2^2 
    \right] 
    \leq \eps
    $
    with probability $99\%$. In our setting, we receive examples from $\mathrm{Ex}(\vec m)$ and set $h(\vec x) = \mathrm{argsort}(\vec m(\vec x))$.
As described in \Cref{algo:informal-level-split-rank}, we design the training set $T = \{\vec x^{(i)}, \vec y^{(i)} \}$, where $\vec y^{(i)} = (h(\vec x^{(i)}; j))_{j \in [k]}$ (we make the ranking $h(\vec x)$ a vector).
We provide $T$ as input to $\texttt{ALGO}$ (with $\vec \xi = \vec 0)$ and get
a vector-valued estimation $\vec m^{(n)}$ that approximates the vector $(h(\cdot;1),...,h(\cdot;k))$. 
Our next goal is to convert the estimate $\vec m^{(n)}$ to a ranking by setting $\wh{h} = \mathrm{argsort} \circ \vec m^{(n)}.$ In Theorem \ref{thm:main-label-rank}, we
show that rounding our estimations back to permutations will yield bounds for the expected Spearman distance $\E_{\vec x \sim \calD_x}[d_{2}(h(\vec x), \wh{h}(\vec x))]$.\\

\noindent\textbf{Proof Sketch of \texorpdfstring{\Cref{thm:lr-main-incomplete}}{Theorem \ref{thm:lr-main-incomplete}}.}
\label{sec:incomplete-sketch}
Consider the VC class $\calG$ consisting of 
mappings $g : \reals^d \to \{-1,+1\}$. Let $\calG_{i,j} = \{ g_{i,j} : \reals^d \to \{-1,+1\}\}$ be a copy of $\calG$ for the pair $(i,j)$.
We let $\wh{g}_{i,j}$ and $g^\star_{i,j}$ be the algorithm's empirical classifier and the Bayes classifier respectively for the pair $(i,j)$.
The first key step is that the SST property (which holds thanks to \Cref{item:ST}) implies that the optimal ranking predictor $h^\star$ is unique almost surely and satisfies \eqref{eq:bayes}. Thanks to the structure of the optimal solution, we can compute the score estimates
$
\wh{s}(\vec x; i) = 1 + \sum_{j \neq i} \vec 1\{ \wh{g}_{i,j}(\vec x) = -1 \}
$
for any $i \in [k]$ and we will set $\wh{h}$ to be $\mathrm{argsort} \circ \wh{s}$. We first show that
$
\Pr_{\vec x \sim \calD_x}[\wh{h}(\vec x) \neq h^\star(\vec x)] \leq \sum_{i < j} \Pr_{\vec x \sim \calD_x} [\wh{g}_{i,j}(\vec x) \neq g^\star_{i,j}(\vec x)]\,.
$
Hence, we have reduced the problem of bounding the LHS error to a series of binary sub-problems. At this moment the problem is binary classification, we can use tools for generalization bounds \citep{boucheron2005theory}, where the population loss function for the pair $(i,j)$ of the classifier $g$ in the VC class is $
L_{i,j}(g) 
=
\E_{(\vec x, \sigma) \sim \calD_R^{\vec q}} [\vec 1\{ g(\vec x) \neq \mathrm{sgn}(\sigma(i) - \sigma(j)) \} | \sigma \ni \{i,j\} ]
$.
We can control the
incurred population loss of the ERM against the Bayes classifier exploiting \Cref{cond:incomplete}, similar to \citet{vogel2020multiclass}. We show that
for a training set $T_n$ with elements $(\vec x, y)$ with $y = \mathrm{sgn}(\sigma(i) - \sigma(j))$
where $(\vec x, \sigma) \sim \calD_R^{\vec q}$ conditioned that $\sigma \ni \{i,j\}$, it holds that:
$
L_{i,j}(\wh{g}_{i,j}) - L_{i,j}(g^\star) \leq 
2 \cdot \left(\inf_{g \in \calG} L_{i,j}(g) - L_{i,j}(g^\star) \right)
+ r_n\,, 
$
with high probability,
where $\wh{g}_{i,j} = \argmin_{g \in \calG} \wh{L}_{i,j}(g; T_n)$ and $g^\star$ is the Bayes classifier and $r_n = O(n^{-\frac{1}{2-a}} \cdot \mathrm{VC}(\calG)^{\frac{1}{2-a}})$, where $a$ is the parameter of the Tsybakov's noise; note that when $a = 1$, we obtain the fast rate $1/n$ and when $a = 0$, we get that standard rate $1/\sqrt{n}$. It remains to aggregate the above $O(k^2)$ binary classifiers into a single one (see \Cref{claim:aggregation}) in order to get the result of \Cref{thm:lr-main-incomplete}. For the sample complexity bound, see \Cref{claim:sc}. For the full proof, see \Cref{thm:main-inc}.

\bibliography{references}

\appendix
\onecolumn

\section{Main Theoretical Results: Statements and Proofs}
\label{sec:TheoreticalResults}

In this section, we provide our results formally: In particular, \Cref{sec:full-proof-dt} contains our results for LR with complete rankings and decision trees and \ref{sec:full-proof-rf} contains our results for LR with complete rankings and random forests. Finally, 
in the \Cref{section:incomplete}, our results for noisy LR with incomplete rankings are provided.

\subsection{Noiseless Oracle with Complete Rankings and Decision Trees (Level Splits \& Breiman)}
\label{sec:full-proof-dt}

\subsubsection{Definition of Properties for Decision Trees}
In this section, we study the Label Ranking problem in the complete rankings' setting. We first define some properties required in order to state our results.
In the high-dimensional regime, we assume that the target function is sparse. We remind the reader the following standard definition of sparsity of a real-valued Boolean function.
\begin{definition}
[Sparsity]
\label{def:sparsity}
We say that the target function $f : \{0,1\}^d \to \reals$ is r-sparse if and only if there exists a set $R \subseteq [d]$ with $|R| = r$ and a function $h : \{0,1\}^r \to \reals$ such that, for every $\vec z \in \{0,1\}^d$, it holds that $f(\vec z) = h(\vec z_R)$. The set $R$ is called the set of relevant features. Moreover, a vector-valued function $\vec m : \{0,1\}^d \to \reals^k$ is said to be $\vec r$-sparse if each coordinate $m_j : \{0,1\}^d \to \reals$ is $r$-sparse.\footnote{Note that each coordinate function $m_i$ can be sparse in a different set of indices.} 
\end{definition}

For intuition about the upcoming \Cref{cond:submodular} and \Cref{cond:diminish}, we refer the reader to the \Cref{sec:level-splits-algo} and \Cref{sec:breiman-algo} respectively \citep[and also the work of][]{syrgkanis2020estimation}.
Let us define the function $\wt{V}$ for a set $S \subseteq [d]$, given a function $f$ and a distribution over features $\calD_x$:
\begin{equation}
\label{eq:hetero-supp1}    
\wt{V}(S) := \E_{\vec z_S \sim \calD_{x,S}} \left[
\left ( \E_{\vec w \sim \calD_x}[f(\vec w) | \vec w_S = \vec z_S] 
\right)^2 \right]
\,,
\end{equation}
where $\calD_{x,S}$ is the marginal distribution $\calD_x$ conditioned on the index set $S$.
The function $\wt{V}$ can be seen as a measure of heterogeneity 
of the within-leaf mean values of the target function $f$, from the leafs created by the split $S$, as mentioned by \citet{syrgkanis2020estimation}.
\begin{condition}
[Approximate Submodularity]
\label{cond:submodular}
Let $C \geq 1$. We say that the function $\wt{V}$ with respect to $f$ (\Cref{eq:hetero-supp1}) is $C$-approximate submodular if and only if for any $T,S \subseteq [d]$, such that $S \subseteq T$ and any $i \in [d]$, it holds that
\[
\wt{V}(T \cup \{i\}) - \wt{V}(T)
\leq 
C \cdot (\wt{V}(S \cup \{i\})) - \wt{V}(S)).
\]
Moreover, a vector-valued function $\vec m : \{0,1\}^d \to \reals^k$ is said to be $\vec C$-approximate submodular if the function $\wt{V}$ with respect to each coordinate $m_j : \{0,1\}^d \to \reals$ of $\vec m$ is $C$-approximate submodular.
\end{condition}
The above condition will be used (and is necessary) in algorithms that use the Level Splits criterion. Let us also set
\begin{equation}
\label{eq:local}
\wt{V}_{\l}(A, \calP) = 
\E_{\vec x \sim \calD_x}
\left [
\left (\E_{\vec z \sim \calD_x}[f(\vec z) | \vec z \in \calP(\vec x)] \right)^2 \Big | \vec x \in A \right] \,,
\end{equation}
where  $\calP$ is a partition of the hypercube, $\calP(\vec x)$ is the cell of the partition in which $\vec x$ lies and $A$ is a cell of the partition. The next condition is the analogue of \Cref{cond:submodular} for algorithms that use the Breiman's criterion.
\begin{condition}
[Approximate Diminishing Returns]
\label{cond:diminish}
For $C \geq 1$, we say that the function $\wt{V}_{\ell}$ with respect to $f$ (\Cref{eq:local}) has the $C$-approximate diminishing returns property
if
for any cells
$A,A'$, any $i \in [d]$ and any $T \subseteq [d]$ such that
$A' \subseteq A$, 
it holds that
\[
\wt{V}_{\l}(A', T \cup \{i\})
- \wt{V}_{\l}(A', T)
\leq C \cdot 
(\wt{V}_{\l}(A, i) - \wt{V}_{\l}(A))\,.
\]
Moreover, a vector-valued function $\vec m : \{0,1\}^d \to \reals^k$ is said to have the $\vec C$-approximate diminishing returns property if the function $\wt{V}_{\ell}$ with respect to each coordinate $m_j : \{0,1\}^d \to \reals$ of $\vec m$ has the $C$-approximate diminishing returns property.
\end{condition}

\subsubsection{The Score Problem}
Having provided a list of conditions that will be useful in our theorems, we are now ready to provide our key results. In order to resolve \Cref{problem:comp}, we consider the following crucial problem. The solution of this problem will be used as a black-box in order to address \Cref{problem:comp}.
We consider the following general setting.
\begin{definition}
[Score Generative Process]
\label{def:regression-full}
Consider an underlying score hypothesis $\vec m : \xSpace \to [1/4,3/4]^k$ and let $\calD_x$ be a distribution over features. Each sample is generated as follows: 
\begin{enumerate}
    \item Draw $\vec x \in \xSpace$ from $\calD_x$.
    \item Draw $\vec \xi \in [-1/4,1/4]^k$ from the zero mean noise distribution $\calE$.
    \item Compute the score $\vec y = \vec m(\vec x) + \vec \xi$.
    \item Output $(\vec x, \vec y)$.
\end{enumerate}
We let $(\vec x, \vec y) \sim \calD.$
\end{definition} 
Under the \emph{score generative process} of \Cref{def:regression-full}, the following problem arises.

\begin{problem}
[Score Learning]
\label{problem:score-full}
Consider the \textbf{score} generative process of \Cref{def:regression-full} with underlying score hypothesis $\vec m : \xSpace \to [1/4,3/4]^k$, that outputs samples of the form $(\vec x, \vec y) \sim \calD$.
The learner is given i.i.d. samples from $\calD$ and its goal is to \emph{efficiently} output a hypothesis $\wh{\vec m} : \mathbb X \to \reals^k$ such that with high probability the error $\E_{\vec x \sim \calD_x}[ \| \wh{\vec m}(\vec x) - \vec m(\vec x) \|_2^2 ]$ is small. 
\end{problem}

\noindent In order to solve \Cref{problem:score-full}, we adopt the techniques and the results of \citet{syrgkanis2020estimation} concerning efficient algorithms based on decision trees and random forests (for an exposition of the framework, we refer the reader to \Cref{appendix:previous-results}). We provide the following vector-valued analogue of the results of \citet{syrgkanis2020estimation}, which resolves \Cref{problem:score-full}. Specifically, we can control the expected squared $L_2$ norm of the error between our estimate and the true score vector $\vec m$.

\begin{theorem}
[Score Learning with Decision Trees]
\label{thm:score}
For any $\eps, \delta > 0$,
under the score generative process of \Cref{def:regression-full} with underlying score hypothesis $\vec m : \{0,1\}^d \to [0,1]^k$ and given i.i.d. data $(\vec x, \vec y) \sim \calD$, the following hold:
\begin{enumerate}
    \item There exists an algorithm (\underline{Decision Trees via Level-Splits} - \Cref{algo:level-split-score}) with set of splits $S_n$ that 
    computes a score estimate $\vec m^{(n)}$ which
    satisfies
    \[
    \Pr_{(\vec x^1, \vec y^1),...,(\vec x^n, \vec y^n) \sim \calD^n} 
    \left [
    \E_{\vec x \sim \calD_x}
    \left [
    \left\| \vec m(\vec x) - \vec m^{(n)}(\vec x; S_n) \right\|_2^2 
    \right] 
    > \eps 
    \right ] \leq \delta\,,
    \]
    and for the number of samples $n$ and the number of splits $\log(t)$, we have that:
    \begin{enumerate}
         \item If $\vec m$ is $\vec r$-sparse as per \Cref{def:sparsity} and under the $\vec C$-submodularity condition ($m_i$ and $\calD_x$ satisfy \Cref{cond:submodular} for each alternative $i \in [k]$), it suffices to draw
    \[
    n = \wt{O}\left( \log(dk/\delta)
    \cdot k^{Cr+2}
    \cdot (Cr/\eps)^{Cr + 2} \right)
    \]
    samples and set the number of splits
    to be $\log(t) = \frac{C r}{C r + 2} (\log(n) - \log(\log(d/\delta)))$.
    
    \item If, additionally to 1.(a), the marginal over the feature vectors $\calD_x$ is a Boolean product probability distribution, it suffices to draw
    \[
    n = \wt{O}\left( \log(dk/\delta) \cdot 2^r \cdot k^2 \cdot (C/\eps)^{2} \right)
    \]
    samples and set the number of splits
    to be $\log(t) = r$.
    
    \end{enumerate}
    
    \item There exists an algorithm (\underline{Decision Trees via Breiman} - \Cref{algo:breiman-score}) that 
    computes a score estimate $\vec m^{(n)}$ which
    satisfies
     \[
    \Pr_{(\vec x^1, \vec y^1),...,(\vec x^n, \vec y^n) \sim \calD^n} 
    \left [
    \E_{\vec x \sim \calD_x}
    \left [
    \left\| \vec m(\vec x) - \vec m^{(n)}(\vec x; P_n) \right\|_2^2 
    \right] 
    > \eps 
    \right ] \leq \delta\,.
    \]
    and for the number of samples $n$ and the number of splits $\log(t)$, we have that:
    \begin{enumerate}
    \item If $\vec m$ is $\vec r$-sparse and under the $\vec C$-approximate diminishing returns condition ($m_i$ and $\calD_x$ satisfy \Cref{cond:diminish} for each alternative $i \in [k]$), it suffices to draw
    \[
    n = \wt{O} \left( \log(dk/\delta) \cdot k^{Cr+3} \cdot
    (Cr/\eps)^{Cr + 3} \right)
    \]
    samples and set $\log(t) \geq \frac{Cr}{Cr+3}(\log(n) - \log(\log(d/\delta)))$.
    
    \item If, additionally to 2.(a), the distribution $\calD_x$ is a Boolean product distribution, it suffices to draw
    \[
    n = \wt{O} \left( \log(dk/\delta) \cdot k^3 \cdot C^2 \cdot 2^r / \eps^3 \right)
    \]
    samples and set $\log(t) \geq r$.
\end{enumerate}
 \end{enumerate}
 The running time of the algorithms is $\poly_{C,r}(d,k,1/\epsilon)$.
\end{theorem}

\begin{algorithm}[ht!] 
\caption{
\color{limeGreen}
\textbf{Level-Splits Algorithm} \color{black} for Score Learning}
\label{algo:level-split-score}
\begin{algorithmic}[1]
\STATE \textbf{Input:} Access to i.i.d. examples of the form $(\vec x, \vec y) \sim \calD$.
\STATE \textbf{Model:} $\vec y = \vec m(\vec x) + \vec \xi$ (\Cref{def:regression-full}) with $\vec m : \{0,1\}^d \to [0,1]^k$.
\STATE \textbf{Output:} An estimate $\vec m^{(n)}(\cdot ; S_n)$ that, with probability $1-\delta$, satisfies 
\[
\E_{\vec x \sim \calD_x} \left [ \left \| \vec m(\vec x) - \vec m^{(n)}(\vec x; S_n) \right \|_2^2 \right] \leq \eps \,.
\]

\vspace{2mm}

\STATE \color{blue}\texttt{LearnScore}\color{black}($\eps, \delta$):
\STATE Draw $n = \wt{\Theta}(\log(dk/\delta) (Crk/\eps)^{O(Cr)})$ samples from $\calD$ \COMMENT{\emph{Under sparsity and \Cref{cond:submodular}.}}
\STATE $D^{(n)} \gets \{ \vec x^{(j)}, \vec y^{(j)}\}_{j \in [n]}$
\STATE \textbf{output} \texttt{LearnScore-LS}($D^{(n)}$)
\vspace{2mm}

\STATE  \color{blue}\texttt{LearnScore-LS} \color{black}($D^{(n)}$):
\STATE Set $\log(t) = \Theta(r \log(rk))$
\STATE Create $k$ datasets $D_i = \{ (\vec x^{(j)}, y_i^{(j)}) \}_{j \in [n]}$
\STATE \textbf{for} $i \in [k]$ \textbf{do}

\STATE ~~~~ $m_i^{(n)}, S_n^{(i)} = \texttt{LevelSplits-Algo}(0, D_i, \log(t))$ \COMMENT{\emph{Call \Cref{algo:level-split}.}}
\STATE \textbf{endfor}
\STATE Output $\vec m^{(n)}(\cdot; S_n^{(1)},...,S_n^{(k)}) = (m_1^{(n)}(\cdot; S_n^{(1)}),...,m_k^{(n)}(\cdot; S_n^{(k)}))$
\end{algorithmic}
\end{algorithm}

\begin{algorithm}[ht!] 
\caption{ \color{orange} \textbf{Breiman's Algorithm} \color{black} for Score Learning}
\label{algo:breiman-score}
\begin{algorithmic}[1]
\STATE \textbf{Input:} Access to i.i.d. examples of the form $(\vec x, \vec y) \sim \calD$.
\STATE \textbf{Model:} $\vec y = \vec m(\vec x) + \vec \xi$ (\Cref{def:regression-full}) with $\vec m : \{0,1\}^d \to [0,1]^k$.
\STATE \textbf{Output:} An estimate $\vec m^{(n)}(\cdot ; S_n)$ that, with probability $1-\delta$, satisfies 
\[
\E_{\vec x \sim \calD_x} \left [ \left \| \vec m(\vec x) - \vec m^{(n)}(\vec x; P_n) \right \|_2^2 \right] \leq \eps \,.
\]

\vspace{2mm}
\STATE  \color{blue}\texttt{LearnScore} \color{black}($\eps, \delta$):
\STATE Draw $n = \wt{\Theta}\left( \log(dk/\delta) (Crk/\eps)^{O(Cr)} \right)$ samples from $\calD$ \COMMENT{\emph{Under sparsity and \Cref{cond:diminish}.}}
\STATE $D^{(n)} \gets \{ \vec x^{(j)}, \vec y^{(j)}\}_{j \in [n]}$
\STATE \textbf{output} \texttt{LearnScore-Breiman}($D^{(n)}$)
\vspace{2mm}

\STATE  \color{blue} \texttt{LearnScore-Breiman} \color{black}($D^{(n)}$):
\STATE Set $\log(t) = \Theta(r \log(rk))$
\STATE Create $k$ datasets $D_i = \{ (\vec x^{(j)}, y_i^{(j)}) \}_{j \in [n]}$
\STATE \textbf{for} $i \in [k]$ \textbf{do}

\STATE ~~~~ $m_i^{(n)}, P_n^{(i)} = \texttt{Breiman-Algo}(0, D_i, \log(t))$ \COMMENT{\emph{Call \Cref{algo:breiman}.}}
\STATE \textbf{endfor}
\STATE Output $\vec m^{(n)}(\cdot; P_n^{(1)},...,P_n^{(k)}) = (m_1^{(n)}(\cdot; P_n^{(1)}),...,m_k^{(n)}(\cdot; P_n^{(k)}))$
\end{algorithmic}
\end{algorithm}

\begin{proof} (of \Cref{thm:score})
Let us set $J = [1/4, 3/4]$ and let $\vec m : \{0,1\}^d \to J^k$ be the underlying score vector hypothesis and consider a training set with $n$ samples of the form $(\vec x, \vec y) \in \{0,1\}^d \times J^k$ with law $\calD$, generated as in \Cref{def:regression-full}. 
We decompose the mapping as $m(\vec x) = (m_1(\vec x), \ldots, m_k(\vec x))$ and aim to learn each function $m_i : \{0,1\}^d \to J$ separately.
Note that since $\vec m$ is $\vec r$-sparse, then any $m_i$ is $r$-sparse for any $i \in [k]$.
We observe that each sample of \Cref{def:regression-full} can be equivalently generated as follows: 
\begin{enumerate}
    \item $\vec x \in \{0,1\}^d$ is drawn from $\calD_x$,
    \item For each $i \in [k]:$
    \begin{enumerate}
        \item Draw $\xi \in [-1/4,1/4]$ from the zero mean distribution marginal $\calE_i$.
        \item Compute $y_i = m_i(\vec x) + \xi$.
    \end{enumerate}
    \item Output $(\vec x, \vec y)$, where $\vec y = (y_i)_{i \in [k]}.$
\end{enumerate}
\noindent 
In order to estimate the coordinate $i \in [k]$, i.e., the function $m_i : \{0,1\}^d \to J$,
we have to make use of the samples $(\vec x, y_i) \in \{0,1\}^d \times [0,1]$.
We have that
\[
\begin{split}
    \Pr \left [ \E_{\vec x \sim \calD_x}
    \left [
    \left \| \vec m(\vec x) - \vec m^{(n)}(\vec x; S_n) \right\|_2^2 
    \right] 
    > \eps  \right] 
& =
\Pr \left [ \sum_{i \in [k]} \E_{\vec x \sim \calD_x}
    \left [
    (m_i(\vec x) - m_i^{(n)}(\vec x; S_n))^2 
    \right] 
    > \eps  \right] \\
& \leq 
\Pr[\exists i \in [k] : B_i]\,,
\end{split}
\]
where we consider the events
\[
B_i = \E_{\vec x \sim \calD_x} \left [ \left(m_i(\vec x) - m_{i}^{(n)}(\vec x; S_i^n) \right)^2 \right] > \eps/k\,,
\]
for any $i \in [k]$, whose randomness lies in the random variables used to construct the empirical estimate $m_i^{n}(\cdot ; S_i^n) = m_i^{n}(\cdot ; S_i^n, (\vec x^{(1)}, y_i^{(1)}),\ldots, (\vec x^{(n)}, y_i^{(n)}))$. We note that we have to split the dataset with examples $(\vec x, \vec y)$ into $k$ datasets $(\vec x, y_i)$ and execute each sub-routine with parameters $(\epsilon/k, \delta/k)$. We now turn to the sample complexity guarantees. Let us begin with the Level-Splits Algorithm.
\paragraph{Case 1a.} If each $m_i$ is $r$-sparse and under the submodularity condition, by \Cref{thm:level-split} with $f = m_i$,
we have that
\[
\Pr[B_i] \leq d \exp(-n/ (Cr \cdot k/\eps)^{Cr +2})\,.
\]
By the union bound, we have that
\[
\Pr[\exists i \in [k] : B_i] \leq \sum_{i \in [k]} \Pr[B_i]\,.
\]
In order to make this probability at most $\delta$, it suffices to make the probability of the bad event $B_i$ at most $\delta/k$, and, so, it suffices to draw
\[
n = \wt{O}(\log(d k/\delta) \cdot (Cr k/\eps)^{Cr +2} )\,.
\]
\paragraph{Case 1b.} If each $m_i$ is $r$-sparse and under the submodularity and the independence of features conditions, by \Cref{thm:level-split} with $f = m_i$,
we have that
\[
\Pr[B_i] \leq d \exp(- n/(2^r (C k /\eps)^2))\,.
\]
By the union bound and in order to make the probability $\Pr[\exists i \in [k] : B_i]$ at most $\delta$, it suffices to draw
\[
n = \wt{O}\left( \log(dk/\delta) \cdot 2^r \cdot (Ck/\eps)^   {2} \right)\,.
\]
Hence, in each one of the above scenarios, we have that
\[
\Pr \left[ \E_{\vec x \sim \calD_x} \left[ \left\| \vec m(\vec x) - \vec m^{(n)}(\vec x; S_n^1,...,S_n^k) \right\|_2^2 \right] > \eps \right] \leq \delta\,.
\]
We proceed with the Breiman Algorithm.
\paragraph{Case 2a.} If each $m_i$ is $r$-sparse and under the approximate diminishing returns condition (\Cref{cond:diminish}), by \Cref{thm:breiman} with $f = m_i$,
we have that
\[
\Pr[B_i] \leq d \exp(-n/(Cr \cdot k/\eps)^{Cr + 3})\,.
\]
By the union bound, we have that
\[
\Pr[\exists i \in [k] : B_i] \leq \sum_{i \in [k]} \Pr[B_i]\,.
\]
In order to make this probability at most $\delta$, it suffices to make the probability of the bad event $B_i$ at most $\delta/k$, and, so, it suffices to draw
\[
n = \wt{O}( \log(dk/\delta) \cdot (Cr \cdot k/\eps)^{Cr + 3} )\,.
\]
\paragraph{Case 2b.} If each $m_i$ is $r$-sparse and under the approximate diminishing returns condition (\Cref{cond:diminish}) and the independence of features conditions, by \Cref{thm:breiman} with $f = m_i$,
we have that
\[
\Pr[B_i] \leq d \exp(-n \eps^3/ (k^3 \cdot C^2 2^r))\,.
\]
By the union bound and in order to make the probability $\Pr[\exists i \in [k] : B_i]$ at most $\delta$, it suffices to draw
\[
n = \wt{O}(\log(dk/\delta) \cdot k^3 C^2 2^r/ \eps^3)\,.
\]
Hence, in each one of the above scenarios, we have that
\[
\Pr \left [ \E_{\vec x \sim \calD_x} \left[ \left \| \vec m(\vec x) - \vec m^{(n)}(\vec x; P_n^1,...,P_n^k) \right\|_2^2 \right] > \eps \right] \leq \delta\,.
\]
\end{proof}

\subsubsection{Main Result for Noiseless LR with Decision Trees}
We are now ready to address \Cref{problem:comp} for the oracle $\mathrm{Ex}(\vec m)$. Our main theorem follows.
We comment that \Cref{infthm:main-label-rank} corresponds to the upcoming case 1(a).

\begin{theorem}
[Label Ranking with Decision Trees]
\label{thm:main-label-rank}
Consider the example oracle $\mathrm{Ex}(\vec m)$ of \Cref{def:distr-free-lr} with underlying score hypothesis $\vec m : \{0,1\}^d \to [0,1]^k$,
where $k \in \nats$ is the number of labels.
Given i.i.d. data $(\vec x, \sigma) \sim \calD_R$, the following hold
for any $\eps > 0$ and $\delta > 0$:
\begin{enumerate}
    \item There exists an algorithm (\underline{Decision Trees via Level-Splits} - \Cref{algo:level-split-rank}) with set of splits $S_n$ that 
    computes an estimate $h^{(n)}(\cdot~; S_n) : \{0,1\}^d \to \permSpace_k$ which
    satisfies
    \[
    \Pr_{(\vec x^{1}, \sigma^1), ..., (\vec x^n, \sigma^n) \sim \calD_R^n} 
    \left [
    \E_{\vec x \sim \calD_x}
    \left [
    d_2(h(\vec x), h^{(n)}(\vec x; S_n)) 
    \right] 
    > 
    \epsilon \cdot k^2
    \right ] \leq \delta\,,
    \]
    and for the number of samples $n$ and the number of splits $\log(t)$, we have that:
    \begin{enumerate}
         \item If $\vec m$ is $\vec r$-sparse as per \Cref{def:sparsity} and under the $\vec C$-submodularity condition ($m_i$ and $\calD_x$ satisfy \Cref{cond:submodular} for each alternative $i \in [k]$), it suffices to draw
    \[
    n = \wt{O}\left( \log(dk/\delta)
    \cdot k^{Cr+2}
    \cdot (Cr/\eps)^{Cr + 2} \right)
    \]
    samples and set the number of splits
    to be $\log(t) = \frac{C r}{C r + 2} (\log(n) - \log(\log(d/\delta)))$.
    
    \item If, additionally to 1.(a), the distribution $\calD_x$ is a Boolean product distribution, it suffices to draw
    \[
    n = \wt{O}\left( \log(dk/\delta) \cdot 2^r \cdot k^2 \cdot (C/\eps)^{2} \right)
    \]
    samples and set the number of splits
    to be $\log(t) = r$.
    \end{enumerate}
    \item There exists an algorithm (\underline{Decision Tress via Breiman} - \Cref{algo:level-split-rank}) with $P_n$ that
    computes an estimate $h^{(n)}(\cdot~; P_n) : \{0,1\}^d \to \permSpace_k$ which
    satisfies
    \[
    \Pr_{(\vec x^{1}, \sigma^1), ..., (\vec x^n, \sigma^n) \sim \calD_R^n} 
    \left [
    \E_{\vec x \sim \calD_x}
    \left [
    d_2(h(\vec x), h^{(n)}(\vec x; P_n)) 
    \right] 
    >
    \epsilon \cdot k^2
    \right ] \leq \delta\,,
    \]
    and for the number of samples $n$ and the number of splits $\log(t)$, we have that:
    \begin{enumerate}
    \item If $\vec m$ is $\vec r$-sparse and under the $\vec C$-approximate diminishing returns condition ($m_i$ and $\calD_x$ satisfy \Cref{cond:diminish} for each alternative $i \in [k]$), it suffices to draw
    \[
    n = \wt{O} \left( \log(dk/\delta) \cdot k^{Cr+3} \cdot
    (C r/\eps)^{C \cdot r + 3} \right)
    \]
    samples and set $\log(t) \geq \frac{Cr}{Cr+3}(\log(n) - \log(\log(d/\delta)))$.
    
    \item If, additionally to 2.(a), the distribution $\calD_x$ is a Boolean product distribution, it suffices to draw
    \[
    n = \wt{O} \left( \log(dk/\delta) \cdot k^3 \cdot C^2 \cdot 2^r / \eps^3 \right)
    \]
    samples and set $\log(t) \geq r$.
\end{enumerate}
    \end{enumerate}
     The running time of the algorithms is $\poly_{C,r}(d,k,1/\epsilon)$.
\end{theorem}

\begin{algorithm}[ht!] 
\caption{Algorithms for Label Ranking with Complete Rankings}
\label{algo:level-split-rank}
\begin{algorithmic}[1]

\STATE \textbf{Input:} Access to i.i.d. examples of the form $(\vec x, \sigma) \sim \calD_R$.
\STATE \textbf{Model:} 
Oracle $\mathrm{Ex}(\vec m)$ with $\vec m : \{0,1\}^d \to [0,1]^k$ and $h(\vec x) = \mathrm{argsort}(\vec m(\vec x)).$ (\Cref{def:distr-free-lr})
\STATE \textbf{Output:} An estimate $h^{(n)}(\cdot ; S_n)$. 

\vspace{2mm}

\STATE \color{blue}\texttt{LabelRank}\color{black}($\eps, \delta$):
\STATE \color{limeGreen} \textbf{Level-Splits Case} \color{black}
\STATE Draw $n = \wt{\Theta}(\log(dk/\delta) (rk/\eps)^r)$ samples from $\calD_R$ \COMMENT{\emph{Under sparsity and \Cref{cond:submodular}}.}
\STATE For any $j \in [n]$, compute $\vec y^{(j)} \gets \vec m_C(\sigma^{(j)})$ 
\COMMENT{\emph{See \Cref{eq:can-repr}.}}
\STATE $D^{(n)} \gets \{ \vec x^{(j)}, \vec y^{(j)}\}_{j \in [n]}$
\STATE \textbf{output} $\mathrm{argsort}$(\texttt{LearnScore-LS}($D^{(n)}$)) \COMMENT{\emph{Call \Cref{algo:level-split-score}}.}

\STATE \color{orange} \textbf{Breiman Case} \color{black}
\STATE Draw $n = \wt{\Theta}\left( \log(dk/\delta) (r k/\eps)^{r} \right)$ samples from $\calD$ \COMMENT{\emph{Under sparsity and \Cref{cond:diminish}.}}
\STATE For any $j \in [n]$, compute $\vec y^{(j)} \gets \vec m_C(\sigma^{(j)})$ 
\COMMENT{\emph{See \Cref{eq:can-repr}.}}
\STATE $D^{(n)} \gets \{ \vec x^{(j)}, \vec y^{(j)}\}_{j \in [n]}$
\STATE \textbf{output}
$\mathrm{argsort}$(
\texttt{LearnScore-Breiman}($D^{(n)}$)) \COMMENT{\emph{Call \Cref{algo:breiman-score}}.}

\end{algorithmic}
\end{algorithm}

\begin{proof} (of \Cref{thm:main-label-rank})
We let the mapping $\vec m_C : \permSpace_k \to [0,1]^k$ be the canonical representation of a ranking, i.e., for $\sigma \in \permSpace_k$, we define 
\begin{equation}
    \label{eq:can-repr}
    \vec m_C(\sigma) = (\sigma(i)/k)_{i \in [k]}\,.
\end{equation}
We reduce this problem to a score problem: for any sample $(\vec x, \sigma) = (\vec x, \mathrm{argsort}(\vec m(\vec x) )) \sim \calD_R, $ we create the tuple $(\vec x, \vec y') = (\vec x, \vec m_C(\sigma)),$ where $\vec m_C$ is the canonical representation of the ranking $\sigma$. 
Hence, any permutation of length $k$ is mapped to a vector whose entries are integer multiples of $1/k$.
Let us fix $\vec x \in \{0,1\}^d.$

So, the tuple $(\vec x, \vec y')$
falls under the setting of the score variant of the regression setting of
\Cref{def:regression} with regression function equal to $\vec m'$ (which is equal to $\vec m_C \circ \mathrm{argsort} \circ \vec m$) and noise vector $\vec \xi' = 0$, i.e., $\vec y' = \vec m'(\vec x) + \vec \xi'$.
Recall that our goal is to use the transformed samples $(\vec x, \vec y')$ in order to estimate the true label ranking mapping $h : \xSpace \to \permSpace_k$.
Let us set $h'^{(n)}$ be the label ranking estimate using $n$ samples. We will show that $h'^{(n)} = \mathrm{argsort}(\vec m'^{(n)})$ is close to $h' = \mathrm{argsort}(\vec m')$ in Spearman's distance, where $\vec m'^{(n)}$ is the estimation of $\vec m'$ using \Cref{thm:score}. We have that
\[
\E_{\vec x \sim \calD_x} \left\| \vec m'(\vec x) - \vec m'^{(n)}(\vec x; S_n^1,...,S_n^k) \right\|_2^2 \leq \eps 
\]
with high probability using the vector-valued tools developed in \Cref{thm:score}.
By choosing an appropriate method, we obtain each one of the items $1(a), 1(b), 2(a)$ and $2(b)$ (each sample complexity result is in full correspondence with \Cref{thm:score}).
Hence, our estimate is, by definition, close to $\vec m_C \circ \mathrm{argsort} \circ \vec m$, i.e.,
\[
\E_{\vec x \sim \calD_x} \left\| \vec m_C(\mathrm{argsort}(\vec m(\vec x))) - \vec m'^{(n)}(\vec x; S_n^1,...,S_n^k) \right\|_2^2 \leq \eps \,,
\]
thanks to the structure of the samples $(\vec x, \vec y').$ We can convert our estimate $\vec m'^{(n)}$ to a ranking by setting $h' = \mathrm{argsort}(\vec m'^{(n)}).$
For any $i \in [k]$,
let us set $m_i$ and $\wh{m}_i$ for the true and the estimation quantities for simplicity;
intuitively (without the expectation operator), a gap of order $\epsilon/k$ to the estimate $(m_i - \wh{m}_i)^2$ yields a bound $|m_i - \wh{m}_i| \leq \sqrt{\eps/k}.$ 
Recall that $m_i = \sigma(i)/k$ and so this implies that, in integer scaling, $|\sigma(i) - k \cdot \wh{m}_i| \leq  O (\sqrt{\eps \cdot k})$. We now have to compute $\wh{\sigma}(i)$, that is the rounded value of $k \cdot \wh{m}_i$. When turning the values $k \cdot \wh{m}_i$ into a ranking, the distortion of the $i$-th element from the correct value $\sigma(i)$ is at most the number of indices $j \neq i$ that lie inside the estimation radius. So, any term of the Spearman distance is on expectation of order $O(\eps \cdot k)$.
This is due to the fact that
$\E[|m_i - k \cdot \wh{m}_i|] \leq \sqrt{\E[(m_i - k \cdot \wh{m}_i)^2]} = O(\sqrt{\eps \cdot k}).$
To conclude, we get that $
\E_{\vec x \sim \calD_x} d_2 (\mathrm{argsort}(\vec m(\vec x))), h'(\vec x))  \leq 
O(\epsilon) \cdot k^2\,.
$
\end{proof}

\subsection{Noiseless Oracle with Complete Rankings and Random Forests (Level Splits \& Breiman)}
\label{sec:full-proof-rf}
In this section, we provide similar algorithmic results for Fully Grown Honest Forests based on the Level-Splits and the Breiman's criteria.

\subsubsection{Definition of Properties for Random Forests}
For algorithms that use Random Forests via Level Splits, we need the following condition.
\begin{condition}
[Strong Sparsity]
\label{cond:strong-sparse}
A target function $f : \{0,1\}^d \to [-1/2, 1/2]$ is $(\beta, r)$-strongly sparse if $f$ is $r$-sparse with relevant features $R$ (see \Cref{def:sparsity})
and
the function $\wt{V}$ (see \Cref{eq:hetero-supp1}) satisfies
\[
\wt{V}(T \cup \{j\}) - \wt{V}(T) + \beta
\leq 
\wt{V}(T \cup \{i\}) - \wt{V}(T)\,,
\]
for all $i \in R, j \in [d]\setminus R$ and $T \subset [d] \setminus \{i\}.$ Moreover, a vector-valued function $\vec m : \{0,1\}^d \to [-1/2, 1/2]^k$ is $(\vec \beta, \vec r)$-strongly sparse if each $m_j$ is $(\beta, r)$-strongly sparse for any $j \in [k]$.
\end{condition}

For algorithms that use Random Forests via Breiman's criterion, we need the following condition.
\begin{condition}
[Marginal Density Lower Bound]
\label{cond:marg-lb}
We say that the density $\calD_x$ is $(\zeta, q)$-lower bounded if, for every set $Q \subset [d]$ with size $|Q| = q$, for every $\vec w \in \{0,1\}^q$, it holds that 
\[
\Pr_{\vec x \sim \calD_x}[\vec x_Q = \vec w] \geq \zeta/2^q\,.
\]
\end{condition}

\subsubsection{Main Result for Noiseless LR with Random Forests}

Our theorem both for Score Learning and Label Ranking for Random Forests with Level-Splits follows.
\begin{theorem}
[Label Ranking with \underline{Fully Grown Honest Forests via Level-Splits}]
\label{thm:level-split-forest-lr}
Let $\eps, \delta > 0$.
Let $H > 0$.
Under \Cref{def:regression-full} with underlying score hypothesis $\vec m : \{0,1\}^d \to [0,1]^k$, where $k \in \nats$ is the number of labels and given access to i.i.d. data $(\vec x, \sigma) \sim \calD_R$, the following hold.
For any $i \in [k]$, we have that:
let $m^{(n,s)}_{i}$ be the forest estimator for alternative $i$ that is built with sub-sampling of size $s$ from the 
training set and where every
tree $m_{i}(\vec x, D_s)$ is built using \Cref{algo:level-split}, with inputs: $\log(t)$ large enough so that every leaf has two or three samples and $h=1.$ Under the strong sparsity condition (see \Cref{cond:strong-sparse}) for any $i \in [k]$, if $R$ is the set of relevant features and for every $\vec w \in \{0,1\}^r$, it holds for the marginal probability that $\Pr_{\vec z \sim \calD_x}(\vec z_R = \vec w) \notin (0, \zeta/2^r)$ and if $s = \wt{\Theta}( 2^r \cdot  ( \log(dk/\delta)/ \beta^2  + \log(k/\delta)/ \zeta) )$, then it holds that
\[
\Pr_{(\vec x^{1}, \sigma^1), ..., (\vec x^n, \sigma^n) \sim \calD_R^n} \left [\E_{\vec x \sim \calD_x} \left[ \left\| \vec m(\vec x) - \vec m^{(n,s)}(\vec x) \right\|_2^2 \right] \geq \eps \right] \leq \delta
\]
using a training set of size $n = \wt{O} \left( 
\frac{2^r k \log(k/\delta)}{\eps} \left( \frac{\log(d)}{\beta^2} + \frac{1}{\zeta} \right)
\right)$.
Moreover, under the generative process of $\mathrm{Ex}(\vec m)$ of \Cref{def:distr-free-lr}, it holds that there exists a $\mathrm{poly}(d,k,1/\eps)$-time algorithm with the same sample complexity that computes
an estimate $h^{(n,s)} : \{0,1\}^d \to \permSpace_k$ which
    satisfies
    \[
    \Pr_{(\vec x^{1}, \sigma^1), ..., (\vec x^n, \sigma^n) \sim \calD_R^n} 
    \left [
    \E_{\vec x \sim \calD_x}
    \left [
    d_2(h(\vec x), h^{(n,s)}(\vec x)) 
    \right] 
    > 
    \epsilon \cdot k^2
    \right ] \leq \delta\,.
    \]
\end{theorem}

\begin{proof} (of \Cref{thm:level-split-forest-lr})
Recall that each sample of \Cref{def:regression-full} can be equivalently generated as follows: 
\begin{enumerate}
    \item $\vec x \in \{0,1\}^d$ is drawn from $\calD_x$,
    \item For each $i \in [k]:$
    \begin{enumerate}
        \item Draw $\xi \in [-1/4,1/4]$ from the zero mean distribution marginal $\calE_i$.
        \item Compute $y_i = m_i(\vec x) + \xi$.
    \end{enumerate}
    \item Output $(\vec x, \vec y)$, where $\vec y = (y_i)_{i \in [k]}.$
\end{enumerate}
\noindent 
In order to estimate the coordinate $i \in [k]$, i.e., the function $m_i : \{0,1\}^d \to [1/4, 3/4]$,
we have to make use of the samples $(\vec x, y_i) \in \{0,1\}^d \times [0,1]$.
We have that
\[
\begin{split}
    \Pr \left [ \E_{\vec x \sim \calD_x}
    \left [
    \left\| \vec m(\vec x) - \vec m^{(n,s)}(\vec x) \right\|_2^2 
    \right] 
    > \eps  \right] 
& =
\Pr \left [ \sum_{i \in [k]} \E_{\vec x \sim \calD_x}
    \left [
    (m_i(\vec x) - m_i^{(n,s)}(\vec x))^2 
    \right] 
    > \eps  \right] \\
& \leq 
\Pr[\exists i \in [k] : B_i]\,,
\end{split}
\]
where we consider the events
\[
B_i = \E_{\vec x \sim \calD_x} \left [\left (m_i(\vec x) - m_{i}^{(n,s)}(\vec x)\right)^2 \right] > \eps/k\,,
\]
for any $i \in [k]$, whose randomness lies in the random variables used to construct the empirical estimate $m_i^{(n,s)} = m_i^{(n,s)}(\cdot ; (\vec x^{(1)}, y_i^{(1)}),\ldots, (\vec x^{(n)}, y_i^{(n)})),$ where $m_i^{(n,s)}$ is the forest estimator for the alternative $i$ using subsampling of size $s$.
We are ready to apply the result of \citet{syrgkanis2020estimation} for random forest with Level-Splits (Theorem 3.4) with accuracy $\eps/k$ and confidence $\delta/k$.
Fix $i \in [k]$.
Let $m_i^{(n,s)}$ be the forest estimator that is built with sub-sampling of size $s$ from the 
training set and where every
tree $m_i(\vec x, D_s)$ is built using \Cref{algo:level-split}, with inputs: $\log(t)$ large enough so that every leaf has two or three samples and $h=1.$ Under the strong sparsity condition for $m_i$ (see \Cref{cond:strong-sparse}), if $R$ is the set of relevant features and for every $\vec w \in \{0,1\}^r$, it holds for the marginal probability that $\Pr_{\vec z \sim \calD_x}(\vec z_R = \vec w) \notin (0, \zeta/2^r)$ and if $s = \wt{\Theta}( 2^r \cdot  ( \log(dk/\delta)/ \beta^2  + \log(k/\delta)/ \zeta) )$, then it holds that
\[
\Pr_{D_n \sim \calD^n} \left[ \E_{\vec x \sim \calD_x} \left[ \left (m_i(\vec x) - m_i^{(n,s)}(\vec x)\right)^2 \right] \geq \eps/k \right] \leq \delta/k\,,
\]
using a training set of size $n = \wt{O} \left( 
\frac{2^r k \log(k/\delta)}{\eps} \left( \frac{\log(d)}{\beta^2} + \frac{1}{\zeta} \right)
\right)$.
Aggregating the $k$ random forests, we get the desired result using the union bound. The Spearman's distance result follows using the canonical vector representation, as in \Cref{thm:main-label-rank}.
\end{proof}

The result for the Breiman's criterion is the following.
\begin{theorem}
[Label Ranking with \underline{Fully Grown Honest Forests via Breiman}]
\label{thm:breiman-forest-lr}
Let $\eps, \delta > 0$.
Under \Cref{def:regression-full} with underlying score hypothesis $\vec m : \{0,1\}^d \to [0,1]^k$, where $k \in \nats$ is the number of labels and given access to i.i.d. data $(\vec x, \sigma) \sim \calD_R$, the following hold.
Suppose that $\calD_x$ is $(\zeta, r)$-lower bounded (see \Cref{cond:marg-lb}). 
For any $i \in [k]$, let $m_i^{(n,s)}$ be the forest estimator for the $i$-th alternative that is built with sub-sampling of size $s$ from the training set and where every tree $m_i(\vec x, D_s)$ is built using the \Cref{algo:breiman}, with inputs: $\log(t)$ large enough
so that every leaf has two or three samples, training set $D_s$ and $h=1$. Then, using
$s = \wt{\Theta}(\frac{2^r \log(dk/\delta)}{\zeta \beta^2})$ and under \Cref{cond:strong-sparse} for any $i \in [k]$, we have that
\[
\Pr_{(\vec x^{1}, \sigma^1), ..., (\vec x^n, \sigma^n) \sim \calD_R^n} \left[ \E_{\vec x \sim \calD_x} \left[ \left \| \vec m(\vec x) - \vec m^{(n,s)}(\vec x) \right\|_2^2 \right] \geq \eps \right] \leq \delta\,,
\]
using $n = \wt{O} \left (\frac{2^r k \log(dk/\delta)}{\eps \zeta \beta^2} \right)$.
Moreover, under the generative process of $\mathrm{Ex}(\vec m)$ of \Cref{def:distr-free-lr}, it holds that there exists a $\mathrm{poly}(d, k, 1/\eps)$-time algorithm with the same sample complexity that computes 
an estimate $h^{(n,s)} : \{0,1\}^d \to \permSpace_k$ which
    satisfies
    \[
    \Pr_{(\vec x^{1}, \sigma^1), ..., (\vec x^n, \sigma^n) \sim \calD_R^n} 
    \left [
    \E_{\vec x \sim \calD_x}
    \left [
    d_2(h(\vec x), h^{(n,s)}(\vec x)) 
    \right] 
    > \epsilon \cdot k^2
    \right ] \leq \delta\,.
    \]
\end{theorem}

\begin{proof} (of \Cref{thm:breiman-forest-lr})
The proof is similar as in \Cref{thm:level-split-forest-lr} (by modifying the sample complexity and the value of subsampling $s$) and is omitted.
\end{proof}

\subsection{Noisy Oracle with Incomplete Rankings}
\label{section:incomplete}
In this section, we study the Label Ranking problem with incomplete rankings and focus on \Cref{problem:stat}. 
In \Cref{def:regression-incomplete}, we describe how a noisy score vector $\vec y$ and its associated ranking $\mathrm{argsort}(\vec y)$ is generated. In order to resolve \Cref{problem:stat}, we consider an One-Versus-One (OVO) approach. In fact, we consider a VC class $\calG$ of binary class and our goal is to
use the incomplete observations and
output a collection of $\binom{k}{2}$ classifiers from $\calG$ so that, for a testing example $\vec x \sim \calD_x$ with $\vec x \in \reals^d$, the estimated ranking $\wh{\sigma}_x$, based on our selected hypotheses from $\calG$,
will be close to the optimal one with high probability. We propose the following algorithm (\Cref{algo:ovo-inc-supp}).
\begin{algorithm}[ht!] 
\caption{Algorithm for Estimation and Aggregation for a VC class}
\label{algo:estim-aggr}
\begin{algorithmic}[1]
\STATE \textbf{Input:} A collection of training sets $D_{i,j}$ for $1 \leq i < j \leq k$, VC class $\calG$.
\STATE \textbf{Output:} An estimate $\wh{\vec s} : \xSpace \to \mathbb{N}^{k}$ for the optimal score vector $\vec s^\star$.

\vspace{2mm}

\STATE \color{blue}\texttt{EstimateAggregate}\color{black}($D_{i,j}$ for all $i < j$, $\calG$):
\STATE \textbf{for} $1 \leq i < j \leq k$ \textbf{do}
\STATE ~~~~ Find $\wh{g}_{i,j} = \argmin_{g \in \calG} \frac{1}{|D_{i,j}|} \sum_{(\vec x, y) \in D_{i,j}} \vec 1\{ g(\vec x) \neq y \}$
\STATE \textbf{endfor}

\STATE \textbf{for} $1 \leq i \leq k$ \textbf{do}
\STATE ~~~~ $\wh{s}(\vec x; i) = 1 + \sum_{j \neq i} \vec 1\{ \wh{g}_{i,j}(\vec x) = -1 \}$ \COMMENT{\emph{Due to the Strict Stochastic Transitivity property (see \Cref{cond:incomplete}).}}
\STATE \textbf{endfor}
\STATE Break ties randomly
\STATE \textbf{output} $\wh{\vec s}(\cdot) = (\wh{s}(\cdot;1), ..., \wh{s}(\cdot; k))$
\end{algorithmic}
\end{algorithm}

\begin{algorithm}[ht!] 
\caption{Algorithm for Label Ranking with Incomplete Rankings}
\label{algo:ovo-inc-supp}
\begin{algorithmic}[1]

\STATE \textbf{Input:} Sample access to i.i.d. examples of the form $(\vec x, \sigma) \sim \calD_R^{\vec q}$, VC class $\calG$.
\STATE \textbf{Model:} Incomplete rankings are generated as in \Cref{def:regression-incomplete}.
\STATE \textbf{Output:} An estimate $\wh{\sigma} : \reals^d \to \permSpace_k$ of the optimal classifier $\sigma^\star$ that satisfies 
\[
\Pr_{\vec x \sim \calD_x} [\wh{\sigma}(\vec x) \neq \sigma^\star(\vec x)] \leq \frac{C_{a,B}}{\phi^2} \cdot \mathrm{OPT}(\calG) + \eps \,.
\]

\vspace{2mm}

\STATE \color{blue}\texttt{LabelRankIncomplete}\color{black}($\eps, \delta$):
\STATE Set $n = \wt{\Theta} \left (k^{\frac{4(1-a)}{a}} \max\{ \log(k/\delta), \mathrm{VC}(\calG)) \} / \poly_{a}(\phi \cdot \eps) \right)$ \COMMENT{\emph{See \Cref{thm:main-inc}.}}
\STATE Draw a training set $D$ of $n$ independent samples from $\calD_R^{\vec q}$
\STATE For any $i \neq j,$ set $D_{i,j} = \emptyset$
\STATE \textbf{for} $1 \leq i < j \leq k$ \textbf{do}
\STATE ~~~~ \textbf{if} $(\vec x, \sigma) \in D$ and $\sigma \ni \{i,j\}$ \textbf{then}
\STATE ~~~~~~~~ Add $(\vec x, \mathrm{sgn}(\sigma(i) - \sigma(j)))$ to $D_{i,j}$
\STATE ~~~~ \textbf{endif}
\STATE \textbf{endfor}

\vspace{2mm}

\STATE \texttt{Training Phase}: $ \wh{\vec s} \gets $ \texttt{EstimateAggregate}($D_{i,j}$ for all $i < j$, $\calG$) \COMMENT{\emph{See \Cref{algo:estim-aggr}.}}
\STATE \texttt{Testing Phase}: On input $\vec x \in \reals^d,$ output $\mathrm{argsort}(\wh{\vec s}(\vec x))$


\end{algorithmic}
\end{algorithm}

In order to resolve \Cref{problem:stat} under \Cref{cond:incomplete}, we will make use of the Kemeny embedding and the OVO approach.
Let $D$ be the training set with labeled examples of the form $(\vec x, \sigma) \sim \calD_R^{\vec q}$, where $\sigma$ corresponds to an incomplete ranking generated as in \Cref{def:regression-incomplete}. Our algorithm proceeds as follows:
\begin{enumerate}
    \item As a first step, for any pair of alternatives $i < j$ with $i,j \in [k]$, we create a dataset $D_{i,j} = \emptyset$.
    
    \item For any $i < j$ and for any feature $\vec x \in D$ whose incomplete ranking $\sigma$ contains both $i$ and $j$, we add in the dataset $D_{i,j}$ 
    the example $(\vec x, y) := (\vec x, \sgn(\sigma(i) - \sigma(j)))$.
    
    \item For any $i < j,$ we compute the ERM solution (see \Cref{algo:estim-aggr}) to the binary classification problem $\wh{g}_{i,j} = \argmin_{g \in \calG} \wh{L}_{i, j}(g)$ where
    \[
    \wh{L}_{i, j}(g) = \frac{1}{|D_{i,j}|} \sum_{ (\vec x, y) \in D_{i,j}} \vec 1\{ g(\vec x) \neq y \}\,.
    \]
    
    \item We aggregate the binary classifiers (see \Cref{algo:estim-aggr}) using the score function:
    \[
    \wh{s}(\vec x; i) = 1 + \sum_{j \neq i} \vec 1\{ \wh{g}_{i, j} (\vec x) = -1 \}\,.
    \]
    The structure of this score function comes from the SST property.
    
    \item Break the possible ties randomly and output the prediction $\mathrm{argsort}(\wh{s}(\vec x)).$
\end{enumerate}

Let us consider a binary classification problem with labels $-1, +1.$ Let the regression function be $\eta(\vec x) = \Pr_{(\vec x,y)}[y = +1 | \vec x]$ and define the mapping $g^\star(\vec x) = \vec 1\{ \eta(\vec x) \geq 1/2 \}.$ If the distribution over $(x,y)$ were known, the problem of finding an optimal classifier would be solved by simply outputting the Bayes classifier $g^\star$, since it is known to minimize the misclassication probability $\Pr_{(x,y)}[y \neq g(x)]$ over the collection of all classifiers. In particular, for any $g \in \calG$, it holds that $$ L(g) - L(g^\star) = 2 \E_{\vec x \sim \calD_x}\left[ | \eta(\vec x) - 1/2| \cdot  \vec 1\{ g(\vec x) \neq g^\star(\vec x)\} \right]\,. $$

In \Cref{problem:stat}, our goal is to estimate the solution of the ranking median regression problem with respect to the KT distance
\[
\sigma^\star = \argmin_{h : \xSpace \to \permSpace_k} \E_{(\vec x, \sigma) \sim \calD_R}[d_{KT}(h(\vec x), \sigma)]\,.
\]
When the probabilities $p_{ij}(\vec x) = \Pr[ \sigma(i) > \sigma(j) | \vec x]$ satisfy the SST property, the solution is unique almost surely and has a closed form (see \eqref{eq:bayes}). Hence, we can use estimate the $O(k^2)$ binary optimal classifiers in order to estimate it.
This is exactly what we will do in \Cref{algo:ovo-inc-supp}.
Our main result is the following.

\begin{theorem}
[Label Ranking with Incomplete Rankings]
\label{thm:main-inc}
Let $\eps, \delta \in (0,1)$ and 
assume that \Cref{cond:incomplete} holds, i.e., the Stochastic Transitivity property holds, the Tsybakov's noise condition holds with $a \in (0,1), B > 0$ and the deletion tolerance condition holds for the survival probability vector with parameter $\phi \in (0,1)$.
Set $C_{a,B} = B^{1-a}/((1-a)^{1-a} a^a)$ and consider a hypothesis class $\calG$ of binary classifiers with finite VC dimension.
There exists an algorithm (\Cref{algo:ovo-inc-supp}) that computes an estimate $\wh{\sigma} : \reals^d \to \permSpace_k$ so that
\[
\Pr_{\vec x \sim \calD_x} [\wh{\sigma}(\vec x) \neq \sigma^\star(\vec x)] \leq 
\frac{C_{a,B}}{ \phi^{2}}  \left( 2 \sum_{i < j} \left(\inf_{g \in \calG} L_{i,j}(g) - L_{i,j}^\star \right)^a \right) +
\eps
\,, 
\]
with probability at least $1-\delta$,
where $\sigma^\star : \reals^d \to \permSpace_k$ is the mapping (see \Cref{eq:bayes}) induced by the aggregation of the $\binom{k}{2}$ Bayes classifiers $g_{i,j}^\star$ with loss $L_{i,j}^\star$,
using $n$ independent samples from $\calD_R^{\vec q}$ (see \Cref{def:regression-incomplete}), with
\[
n = O \left( 
\frac{C_{a,B}}{\phi^{4-2a} \cdot \binom{k}{2} } 
\cdot 
\left( \frac{C_{a,B} \binom{k}{2}}{\eps \cdot \phi} \right)^{\frac{2-a}{a}}
\cdot 
M
\right)\,,
\]
where
\[
M = 
\max 
\left \{ 
\log(k/\delta),
\mathrm{VC}(\calG) \cdot \log\left( \frac{C_{a,B} \mathrm{VC}(\calG)}{\phi^{3-2a}} \cdot 
\left( \frac{C_{a,B} \binom{k}{2}}{\eps \cdot \phi} \right)^{\frac{2-a}{a}}\right)
\right \}\,.
\]
\end{theorem}

In \Cref{tab:inc-vc-samples}, we present our sample complexity results (concerning \Cref{thm:lr-main-incomplete} (and \Cref{thm:main-inc})) for various natural candidate VC classes, including halfspaces and neural networks. We let $a \lor b := \max\{a,b\}$.

\begin{table}[ht]
    \centering
 \caption{The table depicts the sample complexity for \Cref{problem:stat} and \Cref{thm:main-inc} for various concept classes. In the sample complexity column, we set $N_0 = \poly_{a,B} \left(\frac{k}{\phi \cdot \eps} \right)$. The VC dimension bounds for halfspaces and axis-aligned rectangles can be found in \citet{shalev2014understanding} and the VC dimension of $L_2$-balls can be found in \citet{dudley1979balls}.    For the Neural Networks cases, $M$ and $N$ are the number of parameters and of neurons respectively and the corresponding VC dimension bounds are from \citet{baum1989size} and \citet{karpinski1997polynomial}.}
 \vskip 0.15in
 \begin{center}
\begin{small}
\begin{sc}
\begin{tabular}{lccc}
\toprule
Concept Class & VC Dimension & Sample Complexity \\
\midrule
Halfspaces in $\reals^d$ & $d+1$ & $N_0 \cdot O(\log(k/\delta) \lor d\log(d))$\\
Axis-aligned Rectangles in $\reals^d$ & $2d$ & $N_0 \cdot O(\log(k/\delta) \lor d\log(d))$\\
$L_2$-balls in $\reals^d$ & $d+1$ & $N_0 \cdot O(\log(k/\delta) \lor d\log(d))$\\
NN with Sigmoid Activation & $O(M^2 N^2)$ & $N_0 \cdot O(\log(k/\delta) \lor M^2 N^2\log(M \cdot N))$\\
NN with Sign Activation & $O(M\log (M))$ & $N_0 \cdot O(\log(k/\delta) \lor M \log^2(M)) $\\
\bottomrule
\end{tabular}
\end{sc}
\end{small}
\end{center}






\label{tab:inc-vc-samples}
 \end{table}



\begin{remark}
We remark that, in the above results for the noisy nonparametric regression, we only focused on the sample complexity of our learning algorithms. Crucially, the  runtime depends on the complexity of the Empirical Risk Minimizer and this depends on the selected VC class. Hence, the choice of the VC class involves a trade-off between computational complexity and expressivity/flexibility.
\end{remark}


We continue with the proof. In order to obtain fast learning rates for general function classes, the well-known Talagrand’s
inequality \citep[see \Cref{fact:talagrand} and][]{boucheron2005theory} is used combined with an upper bound on the variance of the loss (which is given by the noise condition) and convergence bounds on Rademacher averages \citep[see e.g.,][]{bartlett2005local}.

\begin{proof}
(of \Cref{thm:main-inc})
We decompose the proof into a series of claims.
Consider the binary hypothesis class $\calG$ consisting of 
mappings $g : \reals^d \to \{-1,+1\}$
of finite VC dimension.
We let $g^\star$ be the Bayes classifier. 
We consider $\binom{k}{2}$ copies of this class, one for each unordered pair $(i,j)$ and let $\calG_{i,j} = \{ g_{i,j} : \reals^d \to \{-1,+1\}\}$ be the corresponding class.
We let $\wh{g}_{i,j}$ and $g^\star_{i,j}$ be the algorithm's empirical classifier and the Bayes classifier respectively for the pair $(i,j)$.

\begin{claim}
It holds that
\[
\Pr_{\vec x \sim \calD_x}[\wh{\sigma}(\vec x) \neq \sigma^\star(\vec x)] \leq \sum_{i < j} \Pr_{\vec x \sim \calD_x} [\wh{g}_{i,j}(\vec x) \neq g^\star_{i,j}(\vec x)]\,.
\]
\end{claim}
\begin{proof}
The following hold due to the SST condition (\Cref{cond:incomplete}.i), which implies \eqref{eq:bayes}. Let $\wh{g}_{i,j} = \wh{g}_{i,j}(D_n)$ be the output estimator for the pair $(i,j).$
We have that
\[
\bigcap_{i < j} \{ \vec x \in \reals^d : \wh{g}_{i,j}(\vec x) = g^\star_{i,j}(\vec x)\}
\subset
\{ \vec x \in \reals^d : \wh{\sigma}(\vec x) = \sigma^\star(\vec x)\}\,,
\]
where $\wh{\sigma}, \sigma^\star : \reals^d \to \permSpace_k$ are the mappings generated by aggregating the estimators $\{\wh{g}_{i,j}\}, \{g_{i,j}^\star\}$ respectively. Hence, we get that
\[
\{ \vec x \in \reals^d : \wh{\sigma}(\vec x) \neq \sigma^\star(\vec x)\} \subset \bigcup_{i < j}
\{ \vec x \in \reals^d : \wh{g}_{i,j}(\vec x) \neq g^\star_{i,j}(\vec x)\}\,.
\]
So, we have that the desired probability is controlled by
\[
\Pr_{\vec x \sim \calD_x}[\wh{\sigma}(\vec x) \neq \sigma^\star(\vec x)] \leq \sum_{i < j} \Pr_{\vec x \sim \calD_x} [\wh{g}_{i,j}(\vec x) \neq g^\star_{i,j}(\vec x)]\,,
\]
where the above probabilities also depend on the input training set $D_n$.
\end{proof}

Thanks to the union bound, it suffices to control the error probability of a single binary classifier. Note that the empirical estimator $\wh{g}_{i,j} : \reals^d \to \{-1, 1\}$ is built from a random number of samples $(\vec x, \sigma)$ (those that satisfy $\sigma \ni \{i,j\}$, see also \Cref{algo:ovo-inc-supp}). Let us fix a pair $(i,j)$ and, for $\sigma \ni \{i,j\}$, we set $y_{i,j} = \mathrm{sgn}(\sigma(i) - \sigma(j)).$ For each classifier $g \in \calG_{i,j}$, we introduce the risk
\[
L_{i,j}(g) 
=
\E_{(\vec x, \sigma) \sim \calD_R^{\vec q}} [\vec 1\{ g(\vec x) \neq y_{i,j} \} | \sigma \ni \{i,j\} ]
=
\frac{ \E_{(\vec x, \sigma) \sim \calD_R^{\vec q}} [\vec 1\{ g(\vec x) \neq y_{i,j} \cap \sigma \ni \{i,j\} \}] }{ \E_{(\vec x,\sigma) \sim \calD_R^{\vec q}} [\vec 1\{ \sigma \ni \{i,j\} \}] } \,,
\]
and the empirical risk that is obtained using $n$ i.i.d. samples 
\[
\wh{L}_{i,j}(g) =  \frac{ \sum_{i \in [n]} \vec 1\{ g(\vec x) \neq y_{i,j} \cap \sigma \ni \{i,j\} \} }{ \sum_{i \in [n]} \vec 1\{ \sigma \ni \{i,j\} \} }\,.
\]
We can control these quantities using the following result, which is a modification of a result that appears in \citet{vogel2020multiclass}. In our case, the estimator for the pair $(i,j)$ is built from the samples $(\vec x, \sigma) \sim \calD_R^{\vec q}$ which contain both $i$ and $j$.
\begin{claim}
\label{claim:loss-gap}
Let $ \delta > 0$ and $i \neq j$ with $i,j \in [k]$.
Assume that the Tsybakov condition (\Cref{cond:incomplete}.ii) holds with parameters $a, B$ for the pair $(i,j)$ and that the survival probabilities vector $\vec q$ satisfies \Cref{cond:incomplete}.iii with parameter $\phi \in (0,1)$.  
Set $C_{a,B} = \frac{B^{1-a}}{(1-a)^{1-a}a^a}$.
Then,
for a training set $T_n$ with elements $(\vec x, y)$ with $y = \mathrm{sgn}(\sigma(i) - \sigma(j))$
where $(\vec x, \sigma) \sim \calD_R^{\vec q}$ conditioned that $\sigma \ni \{i,j\}$,
it holds that
\[
L_{i,j}(\wh{g}_{i,j}) - L_{i,j}(g^\star) \leq 
2 \cdot \left(\inf_{g \in \calG} L_{i,j}(g) - L_{i,j}(g^\star) \right)
+ r_n(\delta)\,, 
\]
with probability at least $1-\delta$,
where $\wh{g}_{i,j} = \argmin_{g \in \calG} \wh{L}_{i,j}(g; T_n)$ and $g^\star$ is the Bayes classifier and
\[
r_n(\delta) 
= 
\max \left \{ 2 \left( \frac{16 C_{a,B}}{n \cdot \phi^{3-2a}}\right)^{\frac{1}{2-a}} \cdot
\left(  (C \cdot  \mathrm{VC}(\calG) \cdot \log(n))^{\frac{1}{2-a}} + \left( 32 \log(2/\delta) \right)^{\frac{1}{2-a}} \right), 
\frac{C' \cdot \log(2/\delta)}{\phi \cdot n}
\right \}
\,,
\]
where $C, C'$ are constants. Moreover, if the size of the training set $T_n$ is at least
\[
n \geq n_{\mathrm{PC}} := \max \left \{ (2/\delta)^{\frac{1}{2 C^2 \cdot \mathrm{VC}(\calG)}}, \log(2/\delta) \left( \frac{6^{2-a} \phi^{1-a}}{ C_{a,B} } \right)^{\frac{1}{1-a}} \right \} \,,
\]
we have that
\[
r_n(\delta) 
= 
2 \left( \frac{16 C_{a,B}}{n \cdot \phi^{3-2a}}\right)^{\frac{1}{2-a}} \cdot
\left(  (C \cdot  \mathrm{VC}(\calG) \cdot \log(n))^{\frac{1}{2-a}} + \left( 32 \log(2/\delta) \right)^{\frac{1}{2-a}} \right)\,.
\]

\end{claim}
\begin{proof}
Let us fix $i \neq j$.
Consider the binary class $\calG = \calG_{i,j}$ and let us set $L^\star_{i,j} = L_{i,j}(g^\star),$ where $g^\star$ is the Bayes classifier.
Let us set $y_{i,j} = \mathrm{sgn}(\sigma(i) - \sigma(j)) \in \{-1,+1\}$.
Consider the loss function for the classifier $g \in \calG$:
\[
c_{i,j}(g; \vec x, \sigma) = \vec1 \{ g(\vec x) \neq y_{i,j} \cap \sigma \ni \{i,j\} \}\,.
\]
We introduce the class of loss functions $\calF_{i,j}$ associated with $\calG_{i,j}$, where
\[
\calF_{i,j} =  \left \{ (\vec x, \sigma) \mapsto \vec 1\{ \sigma \ni \{i,j\}\} \cdot (c_{i,j}(g; \vec x, \sigma) - \vec 1\{ g^\star_{i,j}(\vec x) \neq y_{i,j} \})  : g \in \calG_{i,j}  \right \}\,.
\]
Let $\calF_{i,j}^\star$ be the star-hull of $\calF_{i,j}$ with $\calF_{i,j}^\star = \{ a \cdot f : a \in [0,1], f \in \calF_{i,j} \}.$
For any $f \in \calF_{i,j}$, we introduce the function $T$ which controls the variance of the function $f$ as follows: First, we have that
\[
\Var(f) \leq \Pr_{\vec x \sim \calD_x}[g(\vec x) \neq g_{i,j}^\star(\vec x)] \leq \frac{C_{a,B}}{\phi}(L_{i,j}(g) - L_{i,j}^\star)^a\,.
\]
We remark that the first inequality follows from the binary structure of $f$ and the second inequality follows from the fact that the Tsybakov's noise condition implies (see \Cref{fact:tsybakov}) that
\[
\Pr_{(\vec x, \sigma)} [g(\vec x) \neq g_{i,j}^\star(\vec x) | \sigma \ni \{i,j\}]
\leq C_{a,B} (L_{i,j}(g) - L^\star_{i,j})^a\,,
\]
and since $\Pr[\sigma \ni \{i,j\}|\vec x] > \phi$ for all $i \neq j$ and $\vec x \in \reals^d$. Next, we can write that
\[
\frac{C_{a,B}}{\phi }(L_{i,j}(g) - L_{i,j}^\star)^a = \frac{C_{a,B}}{\phi \Pr[\sigma \ni \{i,j\}]^a} (\E f)^a
=: T^2(f)\,.
\]

We will make use of the following result of \citet{boucheron2005theory}. In order to state this result, we have to define the functions $\psi$ and $w$, (we also refer the reader to \citet{boucheron2005theory} for further intuition on the definition of these crucial functions). We set
\[
\psi(r) = \E R_n \{ f \in \calF_{i,j}^\star : T(f) \leq r \} \,,
\]
where $R_n$ is the Rademacher average\footnote{The Rademacher average of a set $A$ is $R_n(A) := \frac{1}{n} \E \sup_{a \in A} |\sum_{i = 1}^n \sigma_i a_i|$, where $\sigma_1,...,\sigma_n$ are independent Rademacher random variables.} of a subset of the star-hull of the loss class $\calF_{i,j}$ whose variance is controlled by $r^2$ and
\[
w(r) = \sup_{f \in \calF_{i,j}^\star : \E f \leq r} T(f)\,,
\]
which captures the largest variance (i.e., the value $\sqrt{\Var(f)}$) among all loss functions $f$ in the star hull of the loss class with bounded expectation.

\begin{theorem}
[Theorem 5.8 in \citet{boucheron2005theory}]
\label{thm:boucheron}
Consider the class $\calG$ of classifiers $g : \xSpace \to \{-1,+1\}.$
For any $\delta > 0,$ let $r^\star(\delta)$ denote the solution of
\[
r = 4 \psi(w(r)) + 2w(r) \sqrt{\frac{2 \log(2/\delta)}{n}} + \frac{16 \log(2/\delta)}{3n}
\]
and $\eps^\star$ the positive solution of the equation $r = \psi(w(r))$. Then, for any $\theta > 0,$ with probability at least $1-\delta$, the empirical risk minimizer $g_n$ satisfies
\[
L(g_n) - \inf_{g \in \calG} L(g)
\leq \theta \left (\inf_{g \in \calG}L(g) - L(g^\star) \right) + \frac{(1+\theta)^2}{4\theta} r^\star(\delta)\,.
\]
\end{theorem}
\noindent In our binary setting with $i \neq j$, the risk $L_{i,j}$ is conditioned on the event $\sigma \ni \{i,j\}$. Hence, with probability at least $1-\delta$, we get that
\[
L_{i,j}(g_n) - \inf_{g \in \calG} L_{i,j}(g)
\leq \theta \left (\inf_{g \in \calG}L_{i,j}(g) - L_{i,j}(g^\star) \right) + \frac{(1+\theta)^2}{4\theta} \frac{r^\star(\delta)}{\Pr[\sigma \ni \{i,j\} ]}\,,
\]
where $\calG = \calG_{i,j}$. We set $\theta = 1$ and, by adding and subtracting the Bayes error $L_{i,j}^\star = L_{i,j}(g^\star)$ in the left hand side, we obtain
\[
L_{i,j}(\wh{g}_{i,j}) - L^\star_{i,j} \leq 2 \left (\inf_{g \in \calG} L_{i,j}(g) - L_{i,j}^\star \right) + \frac{r^\star(\delta)}{\Pr[\sigma \ni \{i,j\}]}\,.
\]
The result follows by using the fact that $\Pr[\sigma \ni \{i,j\}] = \Omega(\phi)$ and
by the analysis of \citet{vogel2020multiclass}
for $r^\star(\delta)$ (see the proof of Lemma 14 by \citet{vogel2020multiclass} and replace $\eps$ by $\phi)$. Finally, we set $r_n(\delta) = r^\star(\delta)/\Pr[\sigma \ni \{i,j\}]$.
\end{proof}

Crucially, observe that the above result does not depend on $k$, since it focuses on the pairwise comparison $i,j.$

\begin{claim}
\label{claim:aggregation}
For $C_{a,B}$ and $\phi$ as defined in \Cref{claim:loss-gap},
it holds that
\[
\Pr_{\vec x \sim \calD_x}[\wh{\sigma}(\vec x) \neq \sigma^\star(
\vec x)] \leq 
\frac{C_{a,B}}{\phi} \left( 2 \sum_{i <  j} \left( \inf_{g \in \calG} L_{i,j}(g) - L_{i,j}^\star  \right)^a + \binom{k}{2} r_n^a \left(\delta/\binom{k}{2} \right)  \right)\,,
\]
with probability at least $1-\delta$.
\end{claim}
\begin{proof}
Recall that $(\vec x, \sigma) \sim \calD_R^{\vec q}$ and fix the training examples given that $\sigma \ni \{i,j\}$. The Tsybakov's noise condition over the marginal over $\vec x$ given that $\sigma \ni \{i,j\}$ implies (see \Cref{fact:tsybakov}) that
\[
\Pr_{(\vec x, \sigma)} [g_{i,j}(\vec x) \neq g_{i,j}^\star(\vec x) | \sigma \ni \{i,j\}]
\leq C_{a,B} (L_{i,j}(g) - L^\star_{i,j})^a\,.
\]
We have that
\begin{align*}
\Pr_{\vec x \sim \calD_x}[g_{i,j}(\vec x) \neq g_{i,j}^\star(\vec x) | \sigma \ni \{i,j\}] &=
\int_{\reals^d} \vec 1\{ g_{i,j}(\vec x) \neq g_{i,j}^\star(\vec x) \} \calD_x(\vec x|\sigma \ni \{i,j\}) d\vec x =\\ 
& = \E_{\vec x \sim \calD_x} \left[\frac{\calD_x(\vec x | \sigma \ni \{i,j\})}{\calD_x(\vec x)} \vec 1\{ g_{i,j}(\vec x) \neq g_{i,j}^\star(\vec x) \}\right]\,,
\end{align*}
where $\calD_x(\cdot | \sigma \ni \{i,j\})$ is the conditional distribution of $\vec x$ given that the label permutation contains both $i$ and $j$.
Then, using the deletion tolerance property, 
we can obtain that $\calD_x(\vec x|\sigma \ni \{i,j\}) = \Omega(\phi \cdot \calD_x(\vec x))$. So, we have that
\[
\Pr_{\vec x}[ \wh{g}_{i,j}(\vec x) \neq g^\star_{i,j}(\vec x)]
\leq \frac{C_{a,B}}{\phi} (L_{i,j}(\wh{g}_{i,j}) - L_{i,j}^\star)^a\,.
\]
Using \Cref{claim:loss-gap} for the pair $i \neq j$ and Minkowski inequality, we have that
\[
\Pr_{\vec x}[ \wh{g}_{i,j}(\vec x) \neq g^\star_{i,j}(\vec x)]
\leq \frac{C_{a,B}}{\phi} \left( 2 \left( \inf_{g \in \calG} L_{i,j}(g) - L_{i,j}^\star \right )^a + r_n^a (\delta) \right ) \,.
\]
Hence, via the union bound, we have that
\[
\Pr_{\vec x \sim \calD_x}[\wh{\sigma}(\vec x) \neq \sigma^\star(
\vec x)] \leq 
\frac{C_{a,B}}{\phi} \left( 2 \sum_{i <  j} \left( \inf_{g \in \calG} L_{i,j}(g) - L_{i,j}^\star  \right)^a + \binom{k}{2} r_n^a \left(\frac{\delta}{\binom{k}{2}} \right)  \right)\,,
\]
with probability at least $1-\delta$.
\end{proof}

\begin{claim}
\label{claim:sc}
Let $C_{a,B}$ and $\phi$ as defined in $\Cref{claim:loss-gap}.$
For any $\eps > 0$, it suffices to draw
\[
n = O \left (\frac{n_{\mathrm{PC}}}{\phi \binom{k}{2}} \right)
\]
samples from $\calD_R^{\vec q}$ in order to get
\[
\frac{C_{a,B}}{\phi} \binom{k}{2} r_{n_{\mathrm{PC}}}^a \left(\delta/\binom{k}{2} \right) \leq \eps\,.
\]
\end{claim}
\begin{proof}
Note that each pair $(i,j)$ requires a training set of size at least $n_{\mathrm{PC}}(\delta)$ in order to make the failure probability at most $\delta/\binom{k}{2}$. These samples correspond the rankings $\sigma$ so that $\sigma \ni \{i,j\}$.
Hence, the sample complexity of the problem is equal to the number of (incomplete) permutations drawn from $\calD_R^{\vec q}$ in order to obtain the desired number of pairwise comparisons. With high probability, the sample complexity is 
\[
n = O \left (\frac{n_{\mathrm{PC}}}{\binom{k}{2}\min_{i \neq j} \Pr[\sigma \ni \{i,j\} ]} \right) =
O \left (\frac{n_{\mathrm{PC}}}{\phi \binom{k}{2}} \right)\,,
\]
since the random variable that corresponds to the number of pairwise comparisons provided by each sample $(\vec x, \sigma) \sim \calD_R^{\vec q}$ is at least $\phi \cdot \binom{k}{2}$, with high probability. The desired sample complexity bound requires a number of pairwise comparisons $n_{\mathrm{PC}}$ so that
\[
\frac{C_{a,B}}{\phi} \binom{k}{2} r_{n_{\mathrm{PC}}}^a \left(\delta/\binom{k}{2} \right) \leq \eps\,.
\]
The number of pairwise comparisons should be
\[
r_{n_{\mathrm{PC}}} \left(\delta/ \binom{k}{2} \right) \leq \left (\frac{\eps \phi}{C_{a,B} \binom{k}{2}}\right)^{1/a}\,.
\]
Let us set $m = n_{\mathrm{PC}}$. It remains to control the function $r_m.$ 
Recall that (see \Cref{claim:loss-gap})
\[
r_m(\delta) 
= 
\max \left \{ 2 \left( \frac{16 C_{a,B}}{m \cdot \phi^{3-2a}}\right)^{\frac{1}{2-a}} \cdot
\left(  (C \cdot  \mathrm{VC}(\calG) \cdot \log(m))^{\frac{1}{2-a}} + \left( 32 \log(2/\delta) \right)^{\frac{1}{2-a}} \right), 
\frac{C' \cdot \log(2/\delta)}{\phi \cdot m}
\right \}
\,,
\]
If the second term is larger in the above maximum operator, we get that
\[
\frac{c_1 \cdot \log(k/\delta)}{\phi \cdot m} \leq 
\left (\frac{\eps \phi}{C_{a,B} \binom{k}{2}}\right)^{1/a}\,,
\]
and so
\[
m \geq \Omega \left( \frac{\log(k/\delta)}{\phi} \cdot 
\left (\frac{C_{a,B} \binom{k}{2}}{\eps \phi}\right)^{1/a} \right) =: N_1\,.
\]
On the other side, we get that
\[
\left( \frac{c_1 C_{a,B} \cdot \mathrm{VC}(\calG) \cdot \log(m)}{m \cdot \phi^{3-2a}}\right)^{\frac{1}{2-a}}
+
\left(\frac{c_2 C_{a,B} \log(k/\delta)}{m \cdot \phi^{3-2a}}\right)^{\frac{1}{2-a}}
\leq 
\frac{1}{2} \left (\frac{\eps \phi}{C_{a,B} \binom{k}{2}}\right)^{1/a}\,,
\]
and so we should take the maximum between the terms
\[
m \geq \Omega \left( \frac{C_{a,B} \log(k/\delta)}{\phi^{3-2a}} \cdot 
\left( \frac{C_{a,B} \binom{k}{2}}{\eps \cdot \phi} \right)^{\frac{2-a}{a}} 
\right) =: N_2\,,
\] 
and
\[
\frac{m}{\log(m)} \geq \Omega \left( 
\frac{C_{a,B} \mathrm{VC}(\calG)}{\phi^{3-2a}} \cdot 
\left( \frac{C_{a,B} \binom{k}{2}}{\eps \cdot \phi} \right)^{\frac{2-a}{a}}
\right) =: M_3 \,.
\]
Let $m = e^y$ and set $y e^{-y} = 1/M_3$. So, we have that
$-y e^{-y} = -1/M_3$ and hence $-y = W(-1/M_3)$, where $W$ is the Lambert $W$ function. Let $N_3$ be the value of $m$ that corresponds to the Lambert equation, i.e., $-\log(m) = W(-1/M_3)$ and so $N_3 \approx e^{-W(-1/M_3)}$ with $1/M_3 \in [0, 1/e]$. 
Hence, the number of samples that suffice to draw from $\calD_R^{\vec q}$ is
\[
n \geq \frac{1}{\phi \binom{k}{2}} \max \{ N_1, N_2, N_3\} = \frac{1}{\phi \binom{k}{2}} \max \{ N_2, N_3\}\,.
\]
We remark that, in the third case, it suffices to take 
\[
m \geq N_3 \approx 2M_3 \log(M_3) = \wt{\Omega} \left ( \frac{C_{a,B} \mathrm{VC}(\calG)}{\phi^{3-2a}} \cdot 
\left( \frac{C_{a,B} \binom{k}{2}}{\eps \cdot \phi} \right)^{\frac{2-a}{a}} \right)\,,
\]
since $2M \log(M) \geq M\log(2M\log(M))$ for all $M$ sufficiently large.
\end{proof}
\noindent These claims complete the proof.
\end{proof}



\section{Additional Experimental Results for Noisy Oracle with Complete Rankings}
\label{sec:experimental}

\paragraph{Experimental Setting and Evaluation Metrics.}
We follow the setting that was proposed by \citet{cheng2008instance} (it has been used for the empirical evaluation of LR since then). For each data set, we run five repetitions of a ten-fold cross-validation process. Each data set is divided randomly into ten folds five times. For every division, we repeat the following process: every fold is used exactly one time as the validation set, while the other nine are used as the training set (i.e., ten iterations for every repetition of the ten-fold cross-validation process) \citep[see][p.181]{james2013introduction}. 
For every test instance, we compute the Kendall tau coefficient between the output and the given ranking. In every iteration, we compute the mean Kendall tau coefficient of all the test instances. Finally, we compute the mean and standard deviation of every iteration's aggregated results. This setting is used for the evaluation of both our synthetic data sets and the LR standard benchmarks. 

\paragraph{Algorithm's Implementation.} The algorithm's implementation was in Python. We decided to use the version of our algorithms with Breiman's criterion. Therefore there was no need to implement decision trees and random forests from scratch. We used \texttt{scikit-learn} implementations. The code for reproducibility of our results can be found in the \href{https://anonymous.4open.science/r/LR-nonparametric-regression-BC28/}{anonymized repository}.

\paragraph{Data sets.} The code for the creation of the Synthetic bencmarks, the Synthetic benchmarks and the standard LR benchmarks that were used in the experimental evaluation can be found in the \href{https://anonymous.4open.science/r/LR-nonparametric-regression-BC28/}{anonymized repository} 



\paragraph{Experimental Results for Different Noise Settings.}
\label{exp-noise}
In our noisy nonparametric regression setting, the input noise acts additively to the vector $\vec  m(\vec x)$ for some input $\vec x$ and the output is computed by permuting the labels. We aim to understand if the proposed algorithms' performance differs in case the added noise is applied directly to the output ranking, by a parameterized noise operator, instead of being added to the vector $\vec m(\vec x)$. 

Hence, we resort to a popular distance-based probability model introduced by Mallows \citep{mallows1957non}. The standard Mallows model $\calM(\sigma_0, \theta)$ is a two-parameter model supported on $\permSpace_k$ with density function for a permutation $\sigma$ equal to
\begin{equation*}
    \Pr_{\sigma \sim \calM(\sigma_0, \theta)}[\sigma | \theta, \sigma_0] = \frac{e^{-\theta D(\sigma,  \sigma_0)}}{Z_k(\theta, \sigma_0)}\,.
\end{equation*}
The ranking $ \sigma_0$ is the central ranking (location parameter) and $\theta > 0$ is the spread (or dispersion) parameter. In this dispersion regime, the ranking $\sigma_0$ is the mode of the distribution. The probability of any other permutation decays exponentially as the distance from the center permutation increases. The spread parameter controls how fast this occurs. Finally, $Z_k(\theta,\sigma_0) := \sum_{\sigma \in \permSpace_k} \exp(-\theta D(\sigma, \sigma_0))$ is the partition function of the Mallows distribution. In what follows, we work with the KT distance (when $D = d_{KT}$, the partition function only depends on the dispersion $\theta$ and $k$ and not on the central ranking i.e., $Z_k = Z_k(\theta)$).

\begin{figure}[ht!]
\begin{subfigure}[b]{0.49\textwidth}
\centerline{\includegraphics[width = \linewidth]{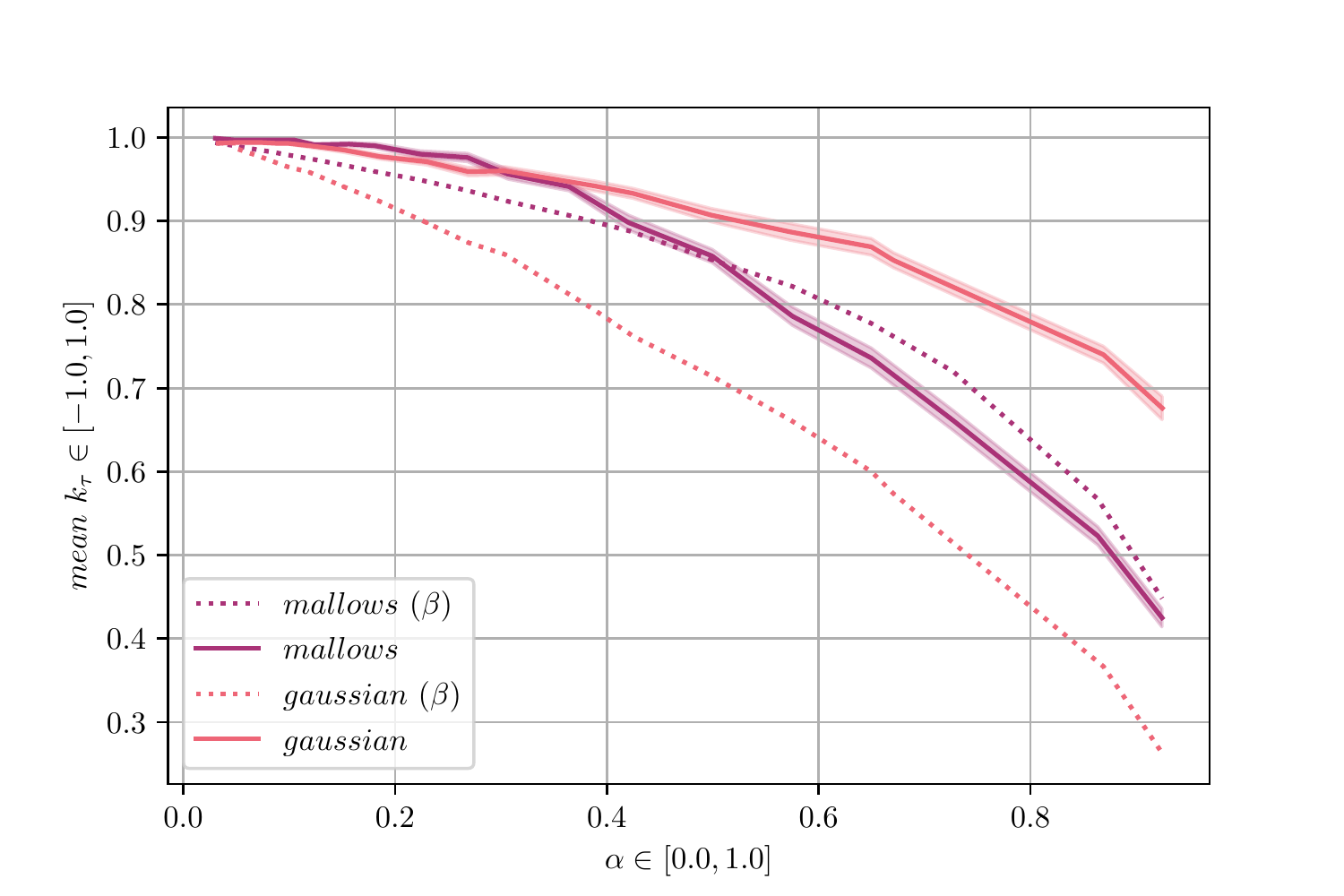}}
\caption{SFN - fully grown random forests}
\label{subfig:experiment2_SFN_rf}
\end{subfigure}
\begin{subfigure}[b]{0.49\textwidth}
\centerline{\includegraphics[width = \linewidth]{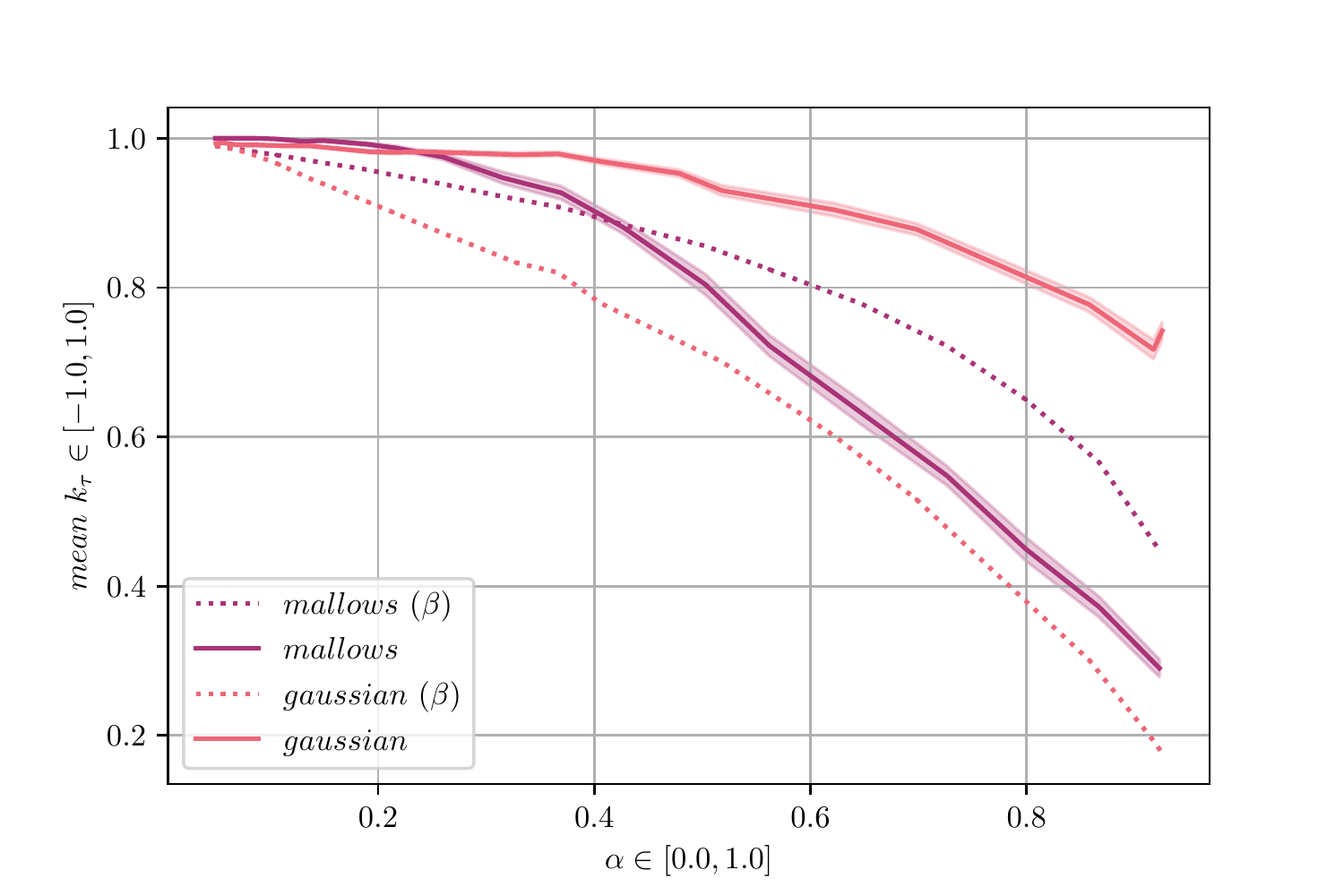}}
\caption{LFN - fully grown random forests}
\label{subfig:experiment2_LFN_rf}
\end{subfigure}

\begin{subfigure}[b]{0.49\textwidth}
\centerline{\includegraphics[width = \linewidth]{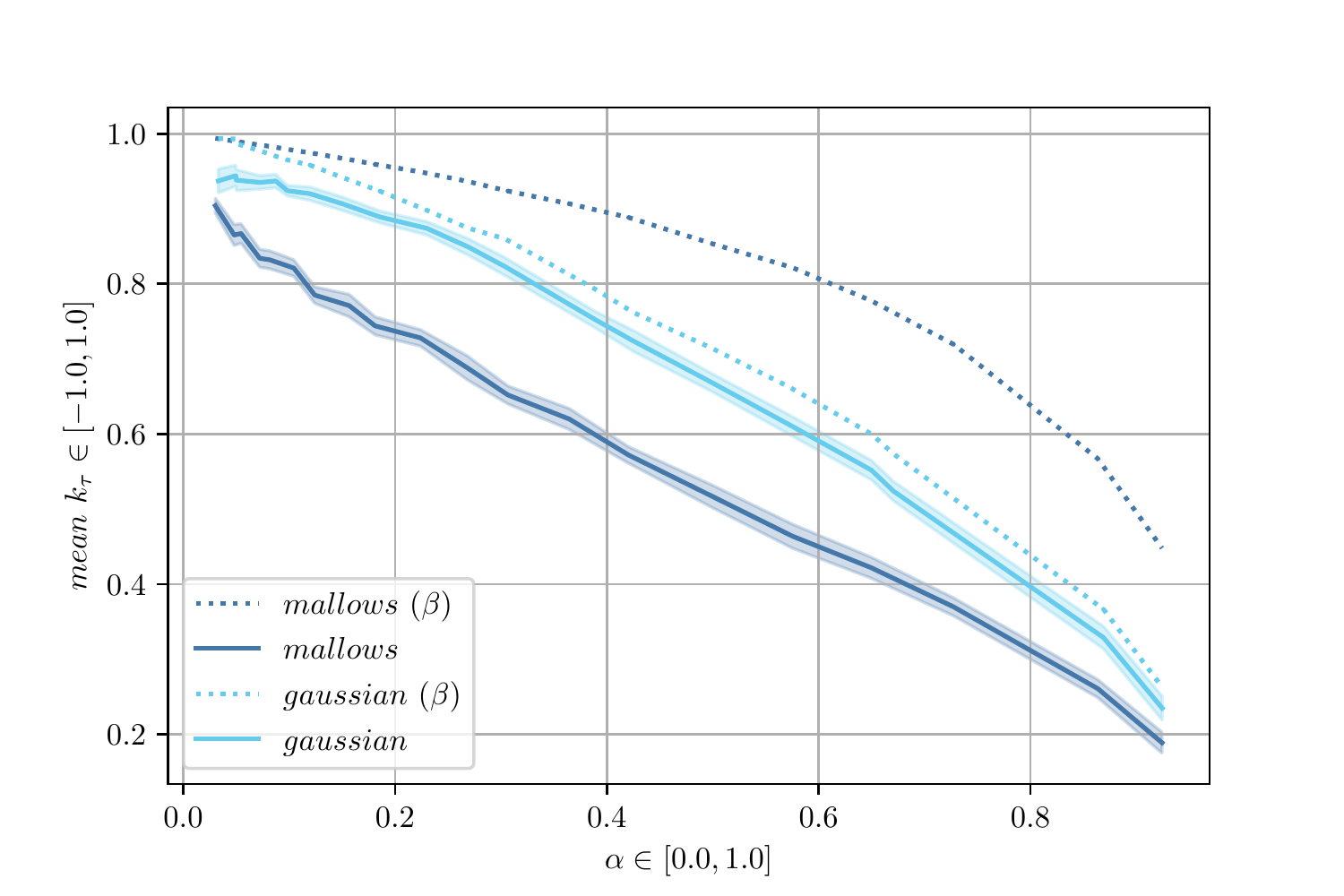}}
\caption{SFN - fully grown decision trees}
\label{subfig:experiment2_SFN_dt}
\end{subfigure}
\begin{subfigure}[b]{0.49\textwidth}
\centerline{\includegraphics[width = \linewidth]{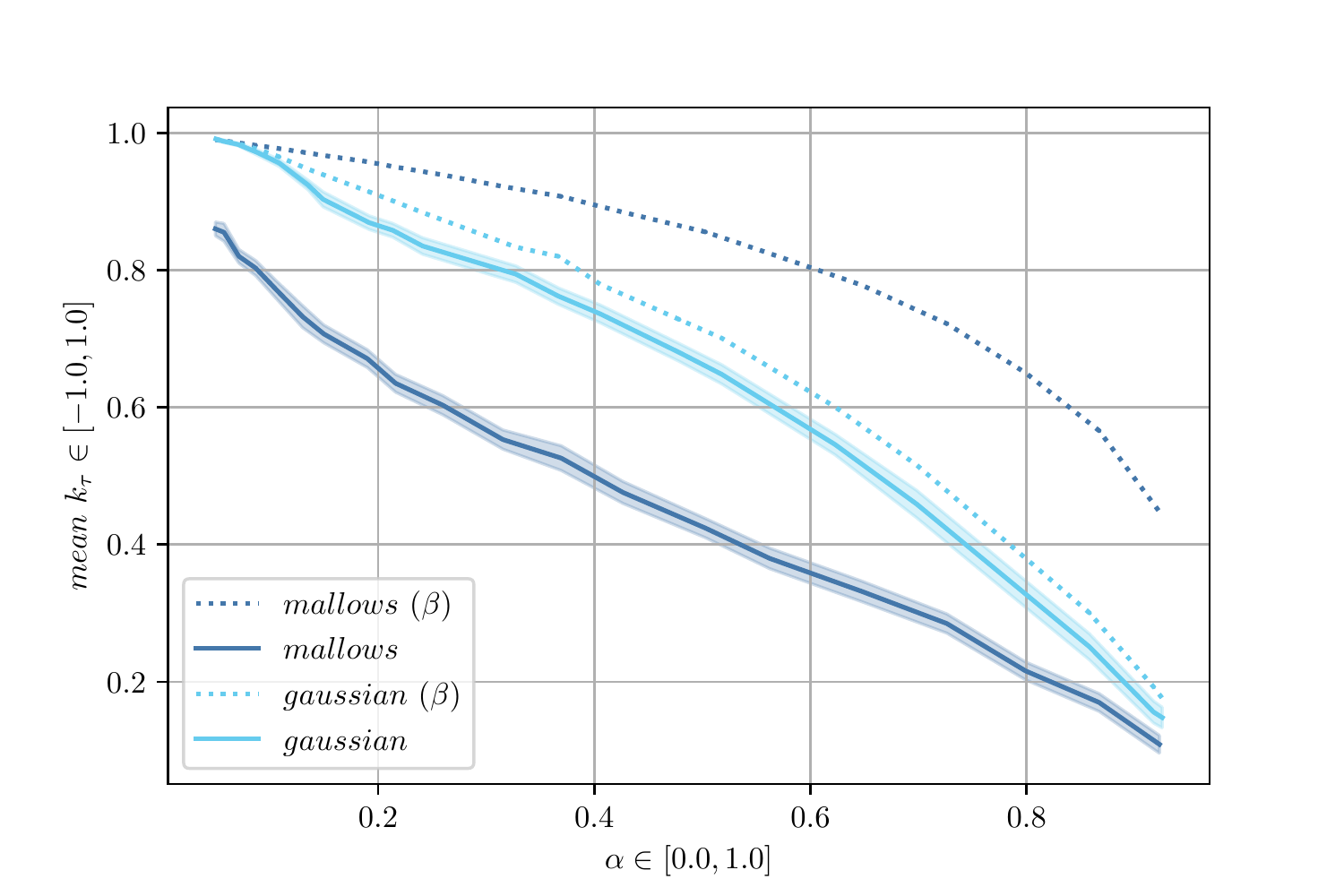}}
\caption{LFN - fully grown decision trees}
\label{subfig:experiment2_LFN_dt}
\end{subfigure}

\begin{subfigure}[b]{0.49\textwidth}
\centerline{\includegraphics[width = \linewidth]{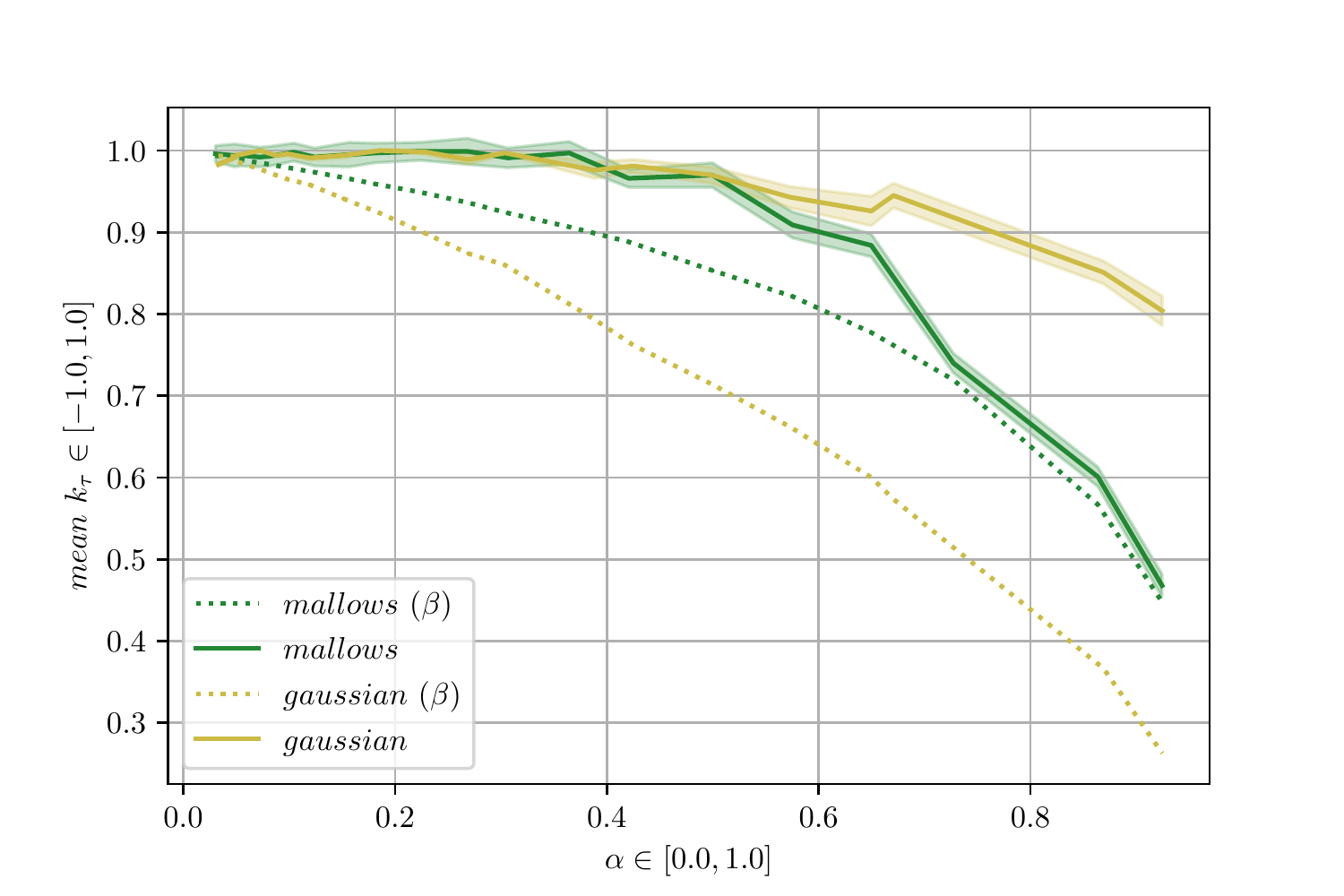}}
\caption{SFN - shallow decision trees}
\label{subfig:experiment2_SFN_dts}
\end{subfigure}
\begin{subfigure}[b]{0.49\textwidth}
\centerline{\includegraphics[width = \linewidth]{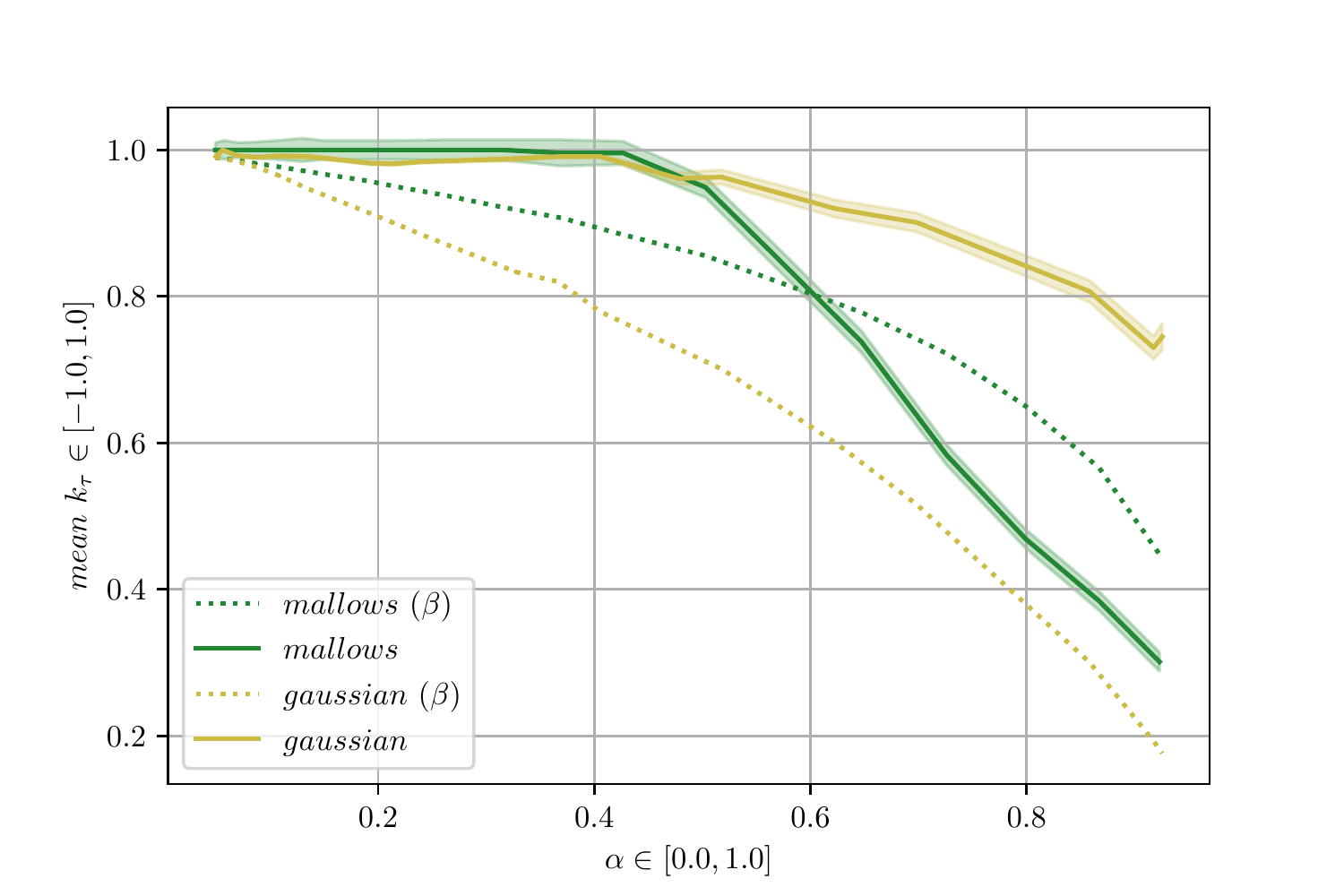}}
\caption{LFN - shallow decision trees}
\label{subfig:experiment2_LFN_dts}
\end{subfigure}
\caption{Illustration of the experimental results in terms of mean $k_{\tau}$ for different  noise operators $\calE$ with respect to $\alpha$-inconsistency and with $\beta$-$k_{\tau}$ gap as a reference point. The Gaussian operator gives a ranking $\sigma = \mathrm{argsort}(\vec m(\vec x) + \vec \xi)$ with $\vec \xi \in [-1/4, 1/4]^k$ being a zero mean truncated Gaussian random variable. The Mallows operator gives a permutation $\sigma \sim \calM( \mathrm{argsort}(\vec m(\vec x)), \theta)$ where $\theta \in [0.6,4.4] $. }
\label{fig:experiment2}
\end{figure}

We once again use the noiseless data sets of the two families, namely LFN and SFN. We create 20 noisy data sets for each family using the Mallows Model as a noise operator (we will refer to these data sets as MN - Mallows Noisy). For each noiseless sample, we generate the corresponding output ranking 
by drawing a permutation from a Mallows distribution with mode equal to
the noiseless ranking and dispersion parameter $\theta \in [0.6, 4.4]$\footnote{We use a different dispersion parameter for each MN dataset. Specifically we use $\theta \in [0.6, 4.4]$ with 0.2 step.} .


We also generate 20 additional noisy data sets for each family (following the generative process $\mathrm{Ex}(\vec m, \calE)$) with different zero mean Gaussian noise distributions (we will refer to these data sets as GN - Gaussian Noisy), matching the $\alpha$-inconsistency property of the MN data sets. Since the corresponding MN and GN data sets do not also share the $\beta$-$k_\tau$ property, the results are not directly comparable. Therefore we compare the MN and GN data sets, for each algorithm and data set family separately, with respect to $\alpha$-inconsistency and using $\beta$-$k_{\tau}$ gap as a reference point. The results are obtained again in terms of mean $k_\tau$ from five repetitions of a ten-fold cross validation, using the noisy output as training data set and the noiseless output data as validation set, and are visualized in \Cref{fig:experiment2}. The mean $k_\tau$ values are shown as solid lines, the shaded area corresponds to the standard deviation, and the dotted lines reveal the $\beta$-$k_{\tau}$ gap. 

We come to similar conclusions regarding the noise tolerance of each algorithm, e.g., shallow trees and fully grown random forests are noisy tolerant while fully grown decision trees' performance decays almost linearly. What is really interesting, is that in spite of the $\beta$-$k_\tau$ gap being higher in the MN data sets, i.e., the mean error in the input is smaller, the performance of the algorithms is worse, in comparison to the GN data sets. This reveals that neither $\alpha$-inconsistency nor $\beta$-$k_\tau$ gap can be strictly correlated to the ability of the models to interpolate the correct underlying function. 

Some comments are in order. The geometry change of the input noise is not the one we should blindly blame for the performance drop. It has to do more with the generative process used for the creation of the data sets. Specifically, in the GN, for every $\vec x$, $k$ samples from a zero mean Gaussian distribution truncated to $[-0.25,0.25]$ were drawn ($\vec \xi \in [-1/4,1/4]^k$) and subsequently resulted to a permutation of the elements. The permutation of the elements depends on the noiseless regression values and the additive noise of each label. This results to irregular changes, which in the presence of a big number of samples `cancel' each other. On the other hand, the MN data sets were created by drawing a sample from a Mallows distribution, which means that certain permutations are more probable for each ranking than others. The noise is no longer unbiased and it is highly likely  that the same ranking almost constantly appears as output for a given center permutation. This could also be linked to strategic addition of noise. With these in mind, we can safely conclude that the noise tolerance of our models is directly affected not only by the error in terms of permutation distance, but also, quite naturally, by the frequency of the repetitive errors, i.e., the frequency of the same erroneous noisy output ranking that arises from the noiseless output ranking, as we have to observe the correct permutation at some frequency to achieve meaningful results.

\paragraph{Evaluation on LR Standard Benchmarks.}
\label{exp-standard-benchmarks}
We also evaluate our algorithms on standard Label Ranking Benchmarks. 
Specifically on sixteen semi-synthetic data sets and on five real world LR data sets. The semi-synthetic ones are considered standard benchmarks for the evaluation of LR algorithms, ever since they were proposed in \citet{cheng2008instance}. They were created from the transformation of multi-class (Type A) and regression (Type B) data sets from UCI repository and the Statlog collection into Label Ranking data \citep[see][]{cheng2008instance}. A summary of these data sets and their characteristics are given in  \Cref{tab:semi-synthetic}. The real-world ones are genetic data, where the genome consists of 2465 genes and each gene is represented by an associated phylogenic profile of length, i.e., 24 features. In these data sets we aim to predict a `qualitative' representation of an expression profile. The expression profile of a gene is an ordered sequence of real-valued measurements, converted into a ranking (e.g., $(2.1, 3.5, 0.7, -2.5)$ is converted to $(2,1,3,4)$) \citep{hullermeier2008label}. A summary of the real-world data sets is given in \Cref{tab:real}.

\begin{table}[ht]
\caption{Properties of the Semi-Synthetic Benchmarks}
\label{tab:semi-synthetic}
\vskip 0.15in
\begin{center}
\begin{small}
\begin{sc}
\begin{tabular}{lcccc}
\toprule
Benchmark &  Type & Number of Instances & Number of attributes & Number of Labels \\
\midrule
authorship & A & 841 & 70 & 4\\
bodyfat & B&4522& 7& 7\\
calhousing &B& 37152& 4&4\\
cpu-small &B& 14744& 6&5\\
elevators & B& 29871 & 9 &9\\
fried &B & 73376& 9&5\\
glass & A& 214 & 9 & 6\\
housing&B&906&6&6\\
iris &A& 150 & 4 & 3\\
pendigits &A& 10992 & 16 & 10\\
segment &A& 2310 & 18 & 7\\
stock & B& 1710&5&5\\
vehicle &B& 846 & 18 & 14\\
vowel &A& 528 & 10 & 11\\
wine &A& 178 & 13 & 3\\
wisconsin &b& 346& 16&16\\
\bottomrule
\end{tabular}
\end{sc}
\end{small}
\end{center}
\end{table}

\begin{table}[ht]
\caption{Properties of the Real-Word Benchmarks}
\label{tab:real}
\vskip 0.15in
\begin{center}
\begin{small}
\begin{sc}
\begin{tabular}{lccc}
\toprule
Benchmark & Number of Instances & Number of Features & Number of Labels \\
\midrule
spo                    & 2465          & 24         & 11\\
heat                   & 2465          & 24         & 6\\
ddt                    & 2465          & 24         & 4\\
cold                   & 2465          & 24         & 4\\
diau                   & 2465          & 24         & 7\\
\bottomrule
\end{tabular}
\end{sc}
\end{small}
\end{center}
\end{table}

We follow the same experimental setting as before, i.e., five repetitions of a ten-fold cross validation process for each data set. In \Cref{tab:our-results} we summarize our results. RF stands for the algorithm using Random Forest, DT for Decision Trees, SDT for Shallow Decision Trees. These are the vanilla versions of the proposed algorithms, i.e., the parameters of the regressors were note tuned (only MSE criterion and maximum depth were defined, while the other parameters were the default). RFT and DTT stand for Random Forests Tuned and Decision Trees Tuned (each decision tree was tuned on the training data). The parameters of the regressor were tuned in a five folds inner c.v. for each training set. The parameter grids are reported in the \href{https://anonymous.4open.science/r/LR-nonparametric-regression-BC28/}{anonymized repository}. 

Random Forests have the best result overall. Interestingly the tuning does not always lead to better results. In the majority of the benchmarks RF and RFT have comparable results. 
\begin{table}[ht]
\caption{Performance in terms of Kendall's tau coefficient - Semi Synthetic Benchmarks}
\label{tab:our-results}
\vskip 0.15in
\begin{center}
\begin{small}
\begin{sc}
\begin{tabular}{lccccc}
\toprule
Benchmark & RF & DT & DTS  & RFT & DTT\\
\midrule
authorship & 0.85±0.04&0.78±0.05&0.81±0.05& \textbf{0.87±0.04}&0.77±0.05\\
bodyfat & \textbf{0.12±0.05}&0.05±0.06&0.09±0.07&0.11±0.06&0.07±0.07\\
calhousing & 0.32±0.01&0.24±0.01&0.17±0.02&\textbf{0.33±0.01}&0.16±0.03\\
cpu-small & 0.29±0.01&0.21±0.02&0.28±0.01 &\textbf{0.30±0.02}&0.28±0.01\\
elevators&0.60±0.01&0.47±0.01&0.50±0.01&\textbf{0.61±0.01}&0.55±0.02\\
fried &\textbf{ 0.96±0.00}&0.90±0.00&0.65±0.01& \textbf{0.96±0.00}&0.90±0.00\\
glass & \textbf{0.88±0.06}&0.80±0.06&0.79±0.06&0.80±0.07&0.80±0.07\\
housing&\textbf{0.44±0.07}&0.40±0.06&0.39±0.06&0.37±0.07&0.31±0.08\\
iris & \textbf{0.95±0.07}&0.91±0.09&0.92±0.08&\textbf{0.95±0.07}&0.90±0.10\\
pendigits & 0.86±0.01&0.77±0.018&0.63±0.01&\textbf{0.86±0.01}&0.78±0.01\\
segment & 0.90±0.02&0.87±0.02&0.82±0.02&\textbf{0.91±0.02}&0.87±0.02\\
stock&\textbf{0.80±0.03}&0.76±0.04&0.76±0.04&0.78±0.03&0.73±0.05\\
vehicle &\textbf{ 0.84±0.03}&0.78±0.04&0.79±0.04&0.83±0.03&0.78±0.04\\
vowel & 0.67±0.04&0.63±0.05&0.55±0.04&\textbf{0.68±0.04}&0.58±0.06\\
wine & 0.90±0.09&0.84±0.12&0.86±0.10&\textbf{0.91±0.08}&0.84±0.11\\
wisconsin & 0.14±0.04&0.08±0.04&0.1±0.04& \textbf{0.14±0.05}&0.09±0.04\\
\bottomrule
\end{tabular}
\end{sc}
\end{small}
\end{center}
\end{table}

\begin{table}[ht]
\caption{Performance in terms of Kendall's tau coefficient- Real World Benchmarks}
\label{tab:our-results-r}
\vskip 0.15in
\begin{center}
\begin{small}
\begin{sc}
\begin{tabular}{lccccc}
\toprule
Benchmark & RF & DT & DTS  & RFT & DTT\\
\midrule
cold&\textbf{0.10±0.03}&0.06±0.03&0.07±0.03&0.09±0.03&0.07±0.04\\
diau&\textbf{0.15±0.03}&0.12±0.02&0.12±0.02&0.14±0.03&0.12±0.04\\
dtt&\textbf{0.13±0.04}&0.10±0.04&0.09±0.04 &\textbf{ 0.13±0.03}&0.10±0.03\\
heat&0.07±0.02&0.05±0.02&0.05±0.02&\textbf{0.08±0.03}&0.05±0.02\\
spo&\textbf{0.05±0.02}&0.05±0.02&0.04±0.02 &0.01±0.01 & 0.01±0.01\\
\bottomrule
\end{tabular}
\end{sc}
\end{small}
\end{center}
\end{table}

In \Cref{tab:comparison}, we compare our Random Forest results (RF and RFT) with other previously proposed methods, as in \citet{cheng2013labelwise} Labelwise Decomposition (LWD),  \citet{hullermeier2008label} Ranking with Pairwise Comparisons (RPC), \citet{cheng2008instance} Label Ranking Trees (LRT), \citet{zhou2018random} Label Ranking Random Forests (LR-RF), \citet{korba2018structured} k-NN Kemeny regressor (kNN Kemeny) and \citet{dery2020boostlr} Boosting-based Learning Ensemble for LR (BoostLR). 

For the semi-synthetic data sets, our results are competitive in comparison to \citet{cheng2013labelwise} LWD, \citet{hullermeier2008label} RPC and \citet{cheng2008instance} LRT results, but with no systematic improvements. As expected, in comparison to the most recent and state-of-the-art results, the experimental results are comparable but cannot compete the highly optimized applied algorithms. But we believe that the insights gained using this technique may be valuable to a variate of other Label Ranking methods. 

For the real-world data sets there are not so many methods to compare our results with. Therefore we compare them with the RPC method and BoostLR. The results are summarized in \Cref{tab:comparison-r}. The performance of our algorithm in the real-word benchmarks is worse than in the other data sets. We suspect that the ``non-sparsity'' of these data sets (genome data) is one of the main reasons that this pattern in the performance is observed.
A more thorough investigation is left for future work.

\begin{table}[ht]
\caption{Evaluation in terms of Kendall's tau coefficient - Semi-Synthetic Benchmarks}
\label{tab:comparison}
\vskip 0.15in
\begin{center}
\begin{small}
\begin{sc}
\begin{tabular}{lcccccccc}
\toprule
Benchmark & RF & RFT & LWD  & RPC & LRT  & LR-RF & kNN Kemeny & BoostLR\\
\midrule
authorship & 0.86±0.04& 0.87±0.04&0.91±0.01& 0.91& 0.88& 0.92&0.94±0.02&0.92\\
bodyfat & 0.12±0.05&0.11±0.06&-&0.28& 0.11 &0.19&0.23±0.06&0.20\\
calhousing & 0.32±0.01&0.33±0.01&-&0.24& 0.36 &0.37&0.33±0.01&0.44\\
cpu-small & 0.29±0.01 &0.3±0.02 &-&0.45& 0.42 & 0.52&0.51±0.00&0.50\\
elevators&0.60±0.01&0.61±0.01    &-&0.75&0.76&0.76&-&0.77\\
fried & 0.96±0.00&0.96±0.00 & -&1.00& 0.89 &1.00&0.89±0.00&0.94\\
glass & 0.83±0.06&0.80±0.07&0.88±0.4&0.88&0.88 &0.89&0.85±0.06&0.89\\
housing &0.44±0.07& 0.37±0.07 & -& 0.67 & 0.80 & 0.80 & - & 0.83\\
iris & 0.95±0.07&0.95±0.07&0.93±0.06 &0.89& 0.95 & 0.97&0.95±0.04&0.83\\
pendigits & 0.86±0.01&0.86±0.01&-&0.93& 0.94 &0.94&0.94±0.00&0.94\\
segment & 0.90±0.02&0.91±0.02&0.94±0.01& 0.93   & 0.95 &0.96&0.95±0.01&0.96\\
stock&0.80±0.03&0.78±0.03&-&0.78&0.90&0.92&-&0.93\\
vehicle & 0.84±0.03&0.83±0.03&0.87±0.02&0.85& 0.83 &0.86&0.85±0.03&0.86\\
vowel & 0.67±0.04&0.68±0.04&0.67±0.02&0.65&0.79 &0.97&0.85±0.03&0.84\\
wine & 0.90±0.09&0.91±0.08&0.91±0.06&0.92&0.88 &0.95&0.94±0.06&0.95\\
wisconsin & 0.14±0.04&0.14±0.05 &-&0.63&0.34 &0.48&0.49±0.04&0.45\\
\bottomrule
\end{tabular}
\end{sc}
\end{small}
\end{center}
\end{table}

\begin{table}[ht]
\caption{Evaluation in terms of Kendall's tau coefficient - Real Word Benchmarks}
\label{tab:comparison-r}
\vskip 0.15in
\begin{center}
\begin{small}
\begin{sc}
\begin{tabular}{lcccc}
\toprule
Benchmark & RF & RFT & RPC & BoostLR \\
\midrule
spo &0.05±0.02  & 0.01±0.01  &0.14±0.02 & 0.14\\
heat&0.07±0.02&  0.08±0.03   &0.13±0.2 & 0.13\\
dtt &0.13±0.04  & 0.13±0.03  &0.17±0.3 & 0.17\\
cold &0.10±0.03 & 0.09±0.0  &0.22±0.03 & 0.21\\
diau &0.15±0.03 & 0.14±0.03  &0.33±0.02& 0.33\\
\bottomrule
\end{tabular}
\end{sc}
\end{small}
\end{center}
\end{table}


\section{Results on the Noisy Oracle with Partial Rankings}
\label{sec:partial}
In this setting, we consider a distribution of partitions of the interval (of positive integers) $[1..k]$, which depends on the feature $\vec x \in \reals^d.$ Before a formal definition, we provide an intuitive example: for some feature $\vec x$, let the noisy score vector be equal to $\vec y = [0.2, 0.4, 0.1, 0.3, 0.5]$ (where $\vec y = \vec m(\vec x) + \vec \xi$, as in  \Cref{def:regression-incomplete}). Then, it holds that
$\sigma = \mathrm{argsort}(\vec y) = (e \succ b \succ d \succ a \succ c).$ In the partial setting, we additionally draw an \emph{increasing}\footnote{A partition $I$ of the space $[k]$ is called increasing if there exists an increasing sequence $i_1 < i_2 < ... < i_m$ of indices of $[k]$ so that $I = [1..i_1][i_1..i_2]...[i_m+1..k]$. For instance, the partition $[1,2][3,4][5]$ is increasing, but the partition $[1,4][2,3][5]$ is not.} partition $I$ of the label space $[k]$ from the distribution $\vec p(\vec x)$. Assume that $I = [1...2][3...3][4...5]$, whose size is $3$. Then, the partial ranking is defined as $\mathrm{PartialRank}(\sigma; I)$ and is equal to $e = b \succ d \succ a = c$.

\begin{definition}
[Generative Process for Partial Data]
\label{def:regression-partial}
Consider an instance of the Label Ranking problem
with underlying score hypothesis $\vec m : \xSpace \to [0,1]^k$ and let $\calD_x$ be a distribution over features. 
Consider the partial partition distribution $\vec p$. Each sample is generated as follows: 
\begin{enumerate}
    \item Draw $\vec x \in \xSpace$ from $\calD_x$ and $\vec \xi \in [-1/4,1/4]^{k}$ from the distribution $\calE$.
    \item Compute $\vec y = \vec m(\vec x) + \vec \xi$.
    \item Set 
    $\wt{\sigma} = \mathrm{argsort}(\vec y).$
    \item Draw a partition $I = [1...i_1][i_1+1...i_2]...[i_L+1...k]$ of size $L+1$ from a distribution $\vec p(\vec x)$, for some $L \in [0..k]$.
    \item Set $\sigma = \mathrm{PartialRank}(\wt{\sigma}; I)$
    \item Output $(\vec x, \sigma).$
\end{enumerate}
We let $(\vec x, \sigma) \sim \calD_R^{\vec p}.$
\end{definition}
Note that when $L = 0$, we observe a complete ranking $\sigma$. We remark that our generative model is quite general: It allows arbitrarily complex distributions over partitions, which depend on the feature space instances. 
We aim to address \Cref{problem:stat} when dealing with partial observations. We are going to address this question in a similar fashion as in the incomplete regime. Specifically, we assume that \Cref{cond:incomplete} still holds, but instead of the \Cref{item:deletionTolerance} (Deletion tolerance), we assume that the property of \emph{partial tolerance} holds.
\begin{condition}
[]
\label{cond:partial}
For any $1 \leq i < j \leq k$, we assume that the following hold: The Stochastic Transitivity and the Tsybakov's noise condition with parameters $a,B$ (see \Cref{cond:incomplete}) are satisfied and the following condition holds:
\begin{enumerate}
    \item (Partial tolerance): There exists $\xi \in (0,1)$ so that $p_{i,j}(\vec x) \geq \xi$, where $p_{i,j}(\vec x)$ is the probability that the pair $i<j$ is not in the same subset of the partial partition for $\vec x \in \reals^d$.
\end{enumerate}
\end{condition}

Using similar technical tools as in \Cref{thm:lr-main-incomplete}, we obtain the following result.
\begin{theorem}
[Label Ranking with Partial Permutations]
\label{thm:partial-main}
Let $\eps, \delta \in (0,1)$ and 
assume that \Cref{cond:partial} holds, i.e., the Stochastic Transitivity property holds and both the Tsybakov's noise condition holds with $a \in (0,1), B > 0$ and the partial tolerance condition holds for the partition probability vector $\vec p$ with parameter $\xi \in (0,1)$.
Set $C_{a,B} = B^{1-a}/((1-a)^{1-a} a^a)$ and consider a hypothesis class $\calG$ of binary classifiers with finite VC dimension.
There exists an algorithm (\Cref{algo:ovo-partial}) that computes an estimate $\wh{\sigma} : \reals^d \to \permSpace_k$ so that
\[
\Pr_{\vec x \sim \calD_x} [\wh{\sigma}(\vec x) \neq \sigma^\star(\vec x)] \leq 
\frac{C_{a,B}}{ \xi^{2}}  \left( 2 \sum_{i < j} \left(\inf_{g \in \calG} L_{i,j}(g) - L_{i,j}^\star \right)^a \right) +
\eps
\,, 
\]
with probability at least $1-\delta$,
where $\sigma^\star : \reals^d \to \permSpace_k$ is the mapping (see \Cref{eq:bayes}) induced by the aggregation of the $\binom{k}{2}$ Bayes classifiers $g_{i,j}^\star$ with loss $L_{i,j}^\star$,
using $n$ independent samples from $\calD_R^{\vec p}$ (see \Cref{def:regression-partial}), with
\[
n = O \left( 
\frac{C_{a,B}}{\xi^{4-2a} \cdot \binom{k}{2} } 
\cdot 
\left( \frac{C_{a,B} \binom{k}{2}}{\eps \cdot \xi } \right)^{\frac{2-a}{a}}
\cdot 
M \right)\,,
\]
where
\[
M = 
\max 
\left \{ 
\log(k/\delta),
\mathrm{VC}(\calG) \cdot \log\left( \frac{C_{a,B} \mathrm{VC}(\calG)}{\xi^{3-2a}} \cdot 
\left( \frac{C_{a,B} \binom{k}{2}}{\eps \cdot\xi} \right)^{\frac{2-a}{a}}\right)
\right \}\,.
\]

\end{theorem}

\begin{algorithm}[ht!] 
\caption{Algorithm for Label Ranking with Partial Rankings}
\label{algo:ovo-partial}
\begin{algorithmic}[1]

\STATE \textbf{Input:} Sample access to i.i.d. examples of the form $(\vec x, \sigma) \sim \calD_R^{\vec p}$, VC class $\calG$.
\STATE \textbf{Model:} Partial rankings are generated as in \Cref{def:regression-partial}.
\STATE \textbf{Output:} An estimate $\wh{\sigma} : \reals^d \to \permSpace_k$ of the optimal estimator $\sigma^\star$ that satisfies 
\[
\Pr_{\vec x \sim \calD_x} [\wh{\sigma}(\vec x) \neq \sigma^\star(\vec x)] \leq \frac{C_{\alpha,B}}{\xi^2} \cdot \mathrm{OPT}(\calG) + \eps \,.
\]

\vspace{2mm}

\STATE \color{blue}\texttt{LabelRankPartial}\color{black}($\eps, \delta$):
\STATE Set $n = \wt{\Theta} \left (k^{\frac{4(1-\alpha)}{\alpha}} \max\{ \log(k/\delta), \mathrm{VC}(\calG)) \} / \poly_{\alpha}(\xi \cdot \eps) \right)$ \COMMENT{\emph{See \Cref{thm:partial-main}.}}
\STATE Draw a training set $D$ of $n$ independent samples from $\calD_R^{\vec p}$
\STATE For any $i \neq j,$ set $D_{i,j} = \emptyset$
\STATE \textbf{for} $1 \leq i < j \leq k$ \textbf{do}
\STATE ~~~~ \textbf{if} $(\vec x, \sigma) \in D$ and $\sigma(i) \neq \sigma(j)$ \textbf{then} \COMMENT{\emph{$i$ and $j$ do not lie in the same partition.}}
\STATE ~~~~~~~~ Add $(\vec x, \mathrm{sgn}(\sigma(i) - \sigma(j)))$ to $D_{i,j}$
\STATE ~~~~ \textbf{endif}
\STATE \textbf{endfor}

\vspace{2mm}

\STATE \texttt{Training Phase}: $\wh{\vec s} \gets$ 
\texttt{EstimateAggregate}($D_{i,j}$ for all $i < j$, $\calG$) \COMMENT{\emph{See \Cref{algo:estim-aggr}.}}
\STATE \texttt{Testing Phase}: On input $\vec x \in \reals^d,$ output $\mathrm{argsort}(\wh{\vec s}(\vec x))$
\end{algorithmic}
\end{algorithm}

\begin{proof}
The proof is similar to the incomplete rankings case and the difference lies in the following steps
\begin{enumerate}
    \item For any $i \neq j$, the conditioning on the event $\sigma \ni \{i,j\}$ should be replaced with the event that $i$ does not lie in the same partition ($\sigma(i) \neq \sigma(j)$). This implies that any $\phi$ term (which corresponds to a lower bound on the probability $\Pr[\sigma \ni \{i,j\}])$ should be replaced by the term $\xi \in (0,1)$. 
    \item For the final sample complexity, the following holds: With high probability, the number of pairwise comparisons induced by a partial ranking is at least $\xi \cdot \binom{k}{2}.$
    Hence, with high probability, a number of $n_{\mathrm{PC}}/(\xi \cdot \binom{k}{2})$ independent draws from the distribution $\calD_R^{\vec p}$ suffices to obtain the desired bound.
\end{enumerate}

\end{proof}

\section{Background on Regression with Trees and Forests}
\label{appendix:previous-results}

In this section, we provide a discussion on decision trees and forests. We refer to \citet{shalev2014understanding}, \citet{syrgkanis2020estimation} and \citet{breiman1984classification} for further details.

\subsection{Further Previous Work}
From an information-theoretic viewpoint, 
sample complexity bounds of decision trees and other data-adaptive partitioning
estimators have been established \citep{nobel1996histogram, lugosi1996consistency, mansour2000generalization}. 
However, from a
computational aspect, the problem of choosing the optimal tree is NP-complete \citep[e.g.,][]{laurent1976constructing} and, hence, from a practical standpoint, trees and forests
are built greedily (e.g., Level-Splits, Breiman) by identifying the most empirically informative split
at each iteration \citep{breiman1984classification, breiman2001random}. Advances have shown that such greedily constructed trees are asymptotically consistent \citep{biau2012analysis, denil2014narrowing, scornet2015consistency} in the low-dimensional regime. On the other side, the high-dimensional regime, where the number of features can grow exponentially with the number of samples, is studied by \citet{syrgkanis2020estimation}. The literature related to the CART criterion \citep{breiman1984classification} is vast; however, there are  various other strands of research dealing with the problem of sparse nonparametric regression (that we consider in the noiseless regression settings of our work). On the one side, several  heuristic methods have been proposed \citep{friedman2001elements,friedman1991multivariate,george1997approaches,smola2001sparse}. On the other side, various works, such as the ones of \citet{lafferty2008rodeo, liu2009nonparametric, comminges2012tight, yang2015minimax}, design and theoretically analyze  greedy algorithmic approaches that exploit
the sparsity of the regression function in order to get around with the curse of dimensionality of the input feature data.

\subsection{Preliminaries on Regression Trees}

\paragraph{Decision Trees.} A decision tree is a predictor $h : \xSpace \to \labSpace$, which, on the input feature $\vec x,$ predicts the label associated with the instance by following a decision path from a root node of a tree to a leaf. At each node on the
root-to-leaf path, the successor child is chosen on the basis of a splitting of the input space. The splitting may be based on a specific feature of $\vec x$ or on a predefined set of splitting rules and a leaf always contains a specific label or value, depending on the context (classification or regression).   

\paragraph{Nonparametric Regression.} In the nonparametric regression problem, we consider that we observe independent samples $(\vec x, y)$, generated as $ y = f(\vec x) + \xi$, where $\xi$ corresponds to bounded zero mean noise. Specifically, we have the following generative process: 
\begin{definition}
[Standard Nonparametric Regression]
\label{def:regression}
Consider the underlying regression function $ f: \xSpace \to [1/4,3/4]$ and let $\calD_x$ be a distribution over features. Each sample is generated as follows: 
\begin{enumerate}
    \item Draw $\vec x \in \xSpace$ from $\calD_x$.
    \item Draw $\xi \in [-1/4,1/4]$ from the zero mean distribution $\calE$.
    \item Compute $y = f(\vec x) + \xi$.
\end{enumerate}
We let $(\vec x, y) \sim \calD$.
\end{definition} 
Note that the noise random variable does not depend on the feature $\vec x.$ In the high-dimensional regime, we assume that the target function is sparse (recall \Cref{def:sparsity}). 

\paragraph{Regression Tree Algorithms.} Let us briefly describe how regression tree algorithms work in an abstract fashion. We will focus on binary trees. In general, the regression tree algorithms operate in two phases, which are the following: first the algorithm finds a partition $\calP$ of the hypercube $\{0,1\}^d.$ Afterwards, it
assigns a single value to every cell of the partition $\calP$, which defines the estimation function $f^{(n)}.$ Finally, the algorithm outputs this estimate. More concretely, we have that:

\paragraph{Phase 1 (Partitioning the space).} First, a depth $0$ tree contains a single cell $\{ \{0,1\}^d \}$. If this single node splits based on whether $x_1 = 0$ or $x_1 = 1$, we obtain a depth $1$ tree with two cells $\{ \{0\} \times \{0,1\}^{d-1}, \{1\} \times \{0,1\}^{d-1} \}$. In general, this procedure generates a partition $\calP$ of the space $\{0,1\}^d$. We let $\calP(\vec x)$ denote the unique cell of $\calP$ that contains
$\vec x$.

\paragraph{Phase 2 (Computing the estimation).} Let $D$ denote the training set that contains examples of the form $(\vec x, y)$, generated as $y = f(\vec x) + \xi$.
For any cell $c \in \calP$, we create the dataset $D_c$ of all the training examples $(\vec x, y) \in D$ that are contained in the cell $c$, i.e., $\vec x \in c$. Then, we compute the value of the cell as
\[
f^{(n)}(c; \calP) := \frac{1}{|D_c|}\sum_{(\vec x,y) \in D_c} y\,.
\]

\noindent The main question not covered in the above discussion is the following:
\begin{center}
    \emph{How is the split of Phase 1 chosen?}
\end{center}

There are various splitting rules in order to partition the space in Phase 1. We discuss two such rules: the Breiman's Algorithm and the Level-Splits Algorithm. In Breiman's algorithm, every node can choose a different direction to split, by using the following greedy criterion: every node chooses the direction
that minimizes its own
empirical mean squared error.
On the other side, in the Level-Splits algorithm, every node at the same level has to split in the same direction, by using the greedy criterion: at every level, we choose the direction that minimizes the total empirical mean squared error. In the upcoming sections, we elaborate on the algorithms based on the Level-Splits (\Cref{sec:level-splits-algo}) and Breiman's (\Cref{sec:breiman-algo}) criterion. 

\subsection{Level-Splits Algorithm}
\label{sec:level-splits-algo}
We define $\vec x_S$ 
as the sub-vector of $\vec x$, where we observe only the coordinates with indices in $S \subseteq [d]$. 
Recall that in the Level-Splits algorithm with set of splits $S$, any level has to split at the same direction. Hence, each level provides a single index to the set $S$ and the size of the set $S$ is the depth of the decision tree.
Given a set of splits $S$, we define the expected mean squared error of $S$ as follows:
\[
\wt{L}(S)
=
\E_{\vec x \sim \calD_x}
\left[
\left(
f(\vec x) - 
\E_{\vec w \sim \calD_x}[f(\vec w) | \vec w_S = \vec x_S] 
\right)^2
\right]
=  
\E_{\vec x \sim \calD_x}
\left[f^2(\vec x)\right]
-
\E_{\vec z_S \sim \calD_{x,S}}
\left(\E_{\vec w \sim \calD_x}[f(\vec w) | \vec w_S = \vec z_S] 
\right)^2\,.
\]
This function quantifies the population version of the mean squared error between the actual value of $f$ at $\vec x$ and the mean value of $f$ constrained at the cell of $\calP$ that contains $\vec x,$ i.e., the subspace of $\{0,1\}^d$ that is equal to $\calP(\vec x)$.
Observe that $\wt{L}$ depends only on $f$ and $\calD_x$. We set 
\begin{equation}
\label{eq:hetero-supp2}    
\wt{V}(S) := \E_{\vec z_S \sim \calD_{x,S}}
\left(\E_{\vec w \sim \calD_x}[f(\vec w) | \vec w_S = \vec z_S] 
\right)^2
\,.
\end{equation}
The function $\wt{V}$ can be seen as a measure of heterogeneity 
of the within-leaf mean values of the target function $f$, from the leafs created by split $S$. The following condition is required.
\begin{condition}
[Approximate Submodularity]
\label{cond:submodular-app}
Let $C \geq 1$. We say that the function $\wt{V}$ is $C$-approximate submodular if and only if for any $T,S \subseteq [d]$, such that $S \subseteq T$ and any $i \in [d]$, it holds that
\[
\wt{V}(T \cup \{i\}) - \wt{V}(T)
\leq 
C \cdot (\wt{V}(S \cup \{i\})) - \wt{V}(S)).
\]
\end{condition}

We can equivalently write this condition as (this is the formulation we used in \Cref{condition:full}):
\[
\wt{L}(T) - \wt{L}(T \cup \{i\})
\leq 
C \cdot (\wt{L}(S) - \wt{L}(S \cup \{i\})).
\]
The reduction of the
mean squared error when the coordinate $i$ is added to a set of splits $T$ is upper bounded by
the reduction when adding $i$ to a subset of $T$, i.e., \emph{if adding $i$ does not decrease the mean squared error significantly at some point (when having the set $S)$, then $i$ cannot decrease the mean squared error significantly in the future either (for any superset of $S$)}.
We remark that this condition is necessary for any greedy algorithm to work \citep[see][]{syrgkanis2020estimation}.

\paragraph{Empirical MSE.} For the algorithm, we will use the empirical version of the mean square error. Provided a set of splits $S$, we have that
\begin{align}
L_n(S) &= \frac{1}{n} \sum_{j \in [n]} \left(y^{(j)} - f^{(n)}(\vec x^{(j)}; S) \right)^2 
= \frac{1}{n} \sum_{j \in [n]} (y^{(j)})^2 - 
\frac{1}{n} \sum_{j \in [n]} f^{(n)}(\vec x^{(j)};S)^2\\
\label{eq:hetero-emp}
& =: \frac{1}{n} \sum_{j \in [n]} (y^{(j)})^2 - V_n(S)\,. 
\end{align}

\begin{algorithm}[ht!] 
\caption{Level-Splits Algorithm \citep[see][]{syrgkanis2020estimation}}
\label{algo:level-split}
\begin{algorithmic}[1]
\STATE \textbf{Input:} honesty flag $h$, training dataset $D_n$, maximum number of splits $\log(t)$.


\STATE \textbf{Output:} Tree approximation of $f$.

\vspace{2mm}

\STATE \color{blue}\texttt{LevelSplits-Algo}\color{black}($h, D_n, \log(t)$):
\STATE $\calV \gets D_{n,x}$ \COMMENT{\emph{Keep only training features $\vec x$}.}

\STATE \textbf{if} $h=1$ \textbf{then} Split randomly $D_n$ in half; $D_{n/2}, D'_{n/2}, n \gets n/2, \calV \gets D'_{n, x}$

\STATE Set $\calP_0 = \{ \{0,1\}^d \}$ \COMMENT{\emph{The partition that corresponds to the root.}}

\STATE $\calP_{\l} = \emptyset$ for any $\l \in [n]$

\STATE $\mathrm{level} \gets -1, S \gets \emptyset$

\STATE \textbf{while} $\mathrm{level} < \log(t)$ \textbf{do}
\STATE ~~~~ $\mathrm{level} \gets \mathrm{level} + 1$
\STATE ~~~~ Select the direction $i \in [d]$ that maximizes $V_n(S \cup \{i\})$ 
\COMMENT{\emph{For $V_n$, see \Cref{eq:hetero-emp}.}}
\STATE ~~~~ \textbf{for all} $C \in \calP_{level}$ \textbf{do}
\STATE ~~~~~~~~ Partition the cell $C$ into the cells
$C_k^{i} = \{ \vec x : \vec x \in C \land x_i = k \}, k \in 0,1$
\STATE ~~~~~~~~ 
\textbf{if} $|\calV \cap C_0^i| \geq 1 \land |\calV \cap C_1^i| \geq 1$ \textbf{then}

\STATE ~~~~~~~~~~~~ $\calP_{level+1} \gets \calP_{level+1} \cup \{ C_0^i, C_1^i\}$

\STATE ~~~~~~~~ \textbf{else}

\STATE ~~~~~~~~~~~~ 
$\calP_{level+1} \gets \calP_{level+1} \cup \{ C \}$

\STATE ~~~~~~~~ \textbf{endif}

\STATE ~~~~ \textbf{endfor}

\STATE ~~~~ $S \gets S \cup \{i\}$
\STATE \textbf{endwhile}

\STATE \textbf{Output} $(\calP_n, f^{(n)}) = (\calP_{level+1}, \vec x \mapsto f^{(n)}(\vec x; S))$

\end{algorithmic}
\end{algorithm}

The following result summarizes the theoretical guarantees for algorithms with Decision Trees via the Level-Splits criterion \citep[see][]{syrgkanis2020estimation}.

\begin{theorem}
[Learning with Decision Trees via Level-Splits (see \cite{syrgkanis2020estimation})]
\label{thm:level-split}
Let $\eps, \delta > 0$.
Let $H > 0$.
Let $\mathrm{D}_n$ be i.i.d. samples from the nonparametric regression model $y = f(\vec x) + \xi$, where $f(\vec x) \in [-1/2, 1/2], \xi \sim \calE, \E_{\xi \sim \calE} [\xi]  = 0 $ and
$\xi \in [-1/2, 1/2].$
Let also $S_n$ be the set of splits chosen by the Level-Splits algorithm (see \Cref{algo:level-split}), with input $h = 0$. The following statements hold.
\begin{enumerate}
    \item Given
    $
    n = \wt{O}\left( \log(d/\delta) \cdot (Cr/\eps)^{Cr + 2} \right)
    $
    samples, if $f$ is $r$-sparse as per \Cref{def:sparsity} and under the submodularity \Cref{cond:submodular-app}, and if we set the number of splits
    to be $\log(t) = \frac{C r}{C r + 2} (\log(n) - \log(\log(d/\delta)))$, then it holds that
    \[
    \Pr_{\mathrm{D}_n \sim \calD^n} 
    \left [
    \E_{\vec x \sim \calD_x}
    \left [
    (f(\vec x) - f^{(n)}(\vec x; S_n))^2 
    \right] 
    > \eps 
    \right ] \leq \delta\,.
    \]
    \item If $f$ is $r$-sparse as per \Cref{def:sparsity} and under the submodularity \Cref{cond:submodular-app} and the independence of features condition, given
    $
    n = \wt{O}\left( \log(d/\delta) \cdot 2^r \cdot (C/\eps)^   {2} \right)
    $
    samples and if we set the number of splits
    to be $\log(t) = r$, then it holds that
    \[
    \Pr_{\mathrm{D}_n \sim \calD^n} 
    \left [
    \E_{\vec x \sim \calD_x}
    \left [
    (f(\vec x) - f^{(n)}(\vec x; S_n))^2 
    \right] 
    > \eps 
    \right ] \leq \delta\,.
    \]
    
\end{enumerate}
\end{theorem}

\paragraph{Fully Grown Honest Forests with Level-Splits Algorithm.}
We first explain the term Fully Grown Honest Forests: The term \emph{Fully Grown} or (deep) means that we split every node until every leaf has exactly $1$ training sample.
The term \emph{Honest} \citep[see][]{wager2018estimation} corresponds to the following: the regression tree algorithms operate in two phases where in both stages we use the same set of training examples. In honest trees, we split randomly the training set and use half of the dataset $(D_{n/2})$ in find a partition of $\{0,1\}^d$ and the other half $(D'_{n/2})$ to assign the values in the cells. Finally, the term \emph{Forest} is used when we subsample $s$ out of $n$ samples and use them in order to build independent trees; we then output the average of these trees and this function is denoted by $f^{(n,s)}.$ 


In the work of \citet{syrgkanis2020estimation},
a result about Fully Grown Forests via the Level-Splits criterion is provided under \Cref{cond:strong-sparse}. Shortly, it holds that, using a training set of size $n = \wt{O} \left( 
\frac{2^r \log(1/\delta)}{\eps} \left( \frac{\log(d)}{\beta^2} + \frac{1}{\zeta} \right)
\right)$, and if for every $\vec w \in \{0,1\}^r$, it holds for the marginal probability that $\Pr_{\vec z \sim \calD_x}(\vec z_R = \vec w) \notin (0, \zeta/2^r)$ and if $s = \wt{\Theta}( 2^r \cdot  ( \log(d/\delta)/ \beta^2  + \log(1/\delta)/ \zeta))$, then it holds that
$$
\Pr_{D_n \sim \calD^n} \left( \E_{\vec x \sim \calD_x} [(f(\vec x) - f^{(n,s)}(\vec x))^2] \geq \eps \right) \leq \delta.
$$ We remark that every
tree $f(\vec x, D_s)$ is built using \Cref{algo:level-split}, with inputs: $\log(t)$ large enough so that every leaf has two or three samples and $h=1.$


\subsection{Breiman's Algorithm}
\label{sec:breiman-algo}
We now turn our attention to the Breiman's criterion. We define the total expected mean square error that is achieved by a partition $\calP$ of $\{0,1\}^d$ in the population model as follows:
\[
\begin{split}
\wt{L}(\calP) 
= 
\E_{\vec x \sim \calD_x}
\left[\left(f(\vec x) - \E_{\vec z \sim \calD_x}[f(\vec z)| \vec z \in \calP(\vec x)]\right)^2\right]
= 
\E_{\vec x \sim \calD_x}[f^2(\vec x)]
-
\E_{\vec x \sim \calD_x}
\left(\E_{\vec z \sim \calD_x}[f(\vec z) | \vec z \in \calP(\vec x)]\right)^2\,.
\end{split}
\]
As in the Level-Splits criterion, we set
\[
\wt{V}(\calP) = 
\E_{\vec x \sim \calD_x}
\left(\E_{\vec z \sim \calD_x}[f(\vec z) | \vec z \in \calP(\vec x)]\right)^2\,.
\]
In order to define the splitting criterion of the algorithm (due to the local nature of Breiman), one has to introduce the local version of the expected MSE for the cell $A$:
\[
\begin{split}
\wt{L}_{\l}(A, \calP) 
&= 
\E_{\vec x \sim \calD_x}
\left[\left(f(\vec x) - \E_{\vec z \sim \calD_x}[f(\vec z)| \vec z \in \calP(\vec x)]\right)^2 \Big | \vec x \in A\right]\\
&= 
\E_{\vec x \sim \calD_x}[f^2(\vec x) | \vec x \in A]
-
\E_{\vec x \sim \calD_x}\left[
\left(\E_{\vec z \sim \calD_x}[f(\vec z) | \vec z \in \calP(\vec x)]\right)^2 \Big | \vec x \in A\right]\,.
\end{split}
\]
We set
\[
\wt{V}_{\l}(A, \calP) = 
\E_{\vec x \sim \calD_x}\left[
\left(\E_{\vec z \sim \calD_x}[f(\vec z) | \vec z \in \calP(\vec x)] \right)^2 \Big | \vec x \in A\right]\,.
\]

The following condition is required for decision tree-based algorithms that use the Breiman's criterion.
\begin{condition}
[Approximate Diminishing Returns]
\label{cond:diminish-app}
For $C \geq 1$, we say that the function $\wt{V}$ has the $C$-approximate diminishing returns property
if
for any cells
$A,A'$, any $i \in [d]$ and any $T \subseteq [d]$ such that
$A' \subseteq A$, 
it holds that
\[
\wt{V}_{\l}(A', T \cup \{i\})
- \wt{V}_{\l}(A', T)
\leq C \cdot 
(\wt{V}_{\l}(A, i) - \wt{V}_{\l}(A))\,.
\]
\end{condition}

For the algorithm, we need the empirical mean squared error, conditional on a cell $A$ and a potential split direction $i$, which is defined as follows: let $N_n(A)$ be the number of training points in the cell $A$.
Recall that $A_z^i = \{ \vec x \in A | x_i = z \}$ for $z \in \{0,1\}$. Also, set $f^{(n)}(\vec x; \calP) = g^{(n)}(\calP(\vec x))$.
Then, we have that
\begin{align}
L_n^{\l}(A, i) 
&= \sum_{z \in \{0,1\}} \frac{N_n(A_z^i)}{N_n(A)} \sum_{j : \vec x^{(j)} \in A_z^i} \frac{1}{N_n(A_z^i)}(y^{(j)} - f^{(n)}(\vec x^{(j)}; \calP(\vec x^{(j)})))^2 \\
& = \frac{1}{N_n(A)} \sum_{j : \vec x^{(j)} \in A} (y^{(j)})^2
- \sum_{z \in \{0,1\}} \frac{N_n(A_z^i)}{N_n(A)} (g^{(n)}(A_z^i))^2 \\
\label{eq:emp-breiman}
& =: \frac{1}{N_n(A)} \sum_{j : \vec x^{(j)} \in A} (y^{(j)})^2
- V_n^{\l}(A,i)\,. 
\end{align}

\begin{algorithm}[ht!] 
\caption{Breiman's Algorithm \citep[see][]{syrgkanis2020estimation}}
\label{algo:breiman}
\begin{algorithmic}[1]
\STATE \textbf{Input:} honesty flag $h$, training dataset $D_n$, maximum number of splits $t$.


\STATE \textbf{Output:} Tree approximation of $f.$

\vspace{2mm}

\STATE \color{blue}\texttt{Breiman-Algo}\color{black}($h, D_n, t$):
\STATE $\calV \gets D_{n,x}$
\STATE \textbf{if} $h=1$ \textbf{then} Split randomly $D_n$ in half; $D_{n/2}, D'_{n/2}, n \gets n/2, \calV \gets D'_{n,x}$

\STATE Set $\calP_0 = \{ \{0,1\}^d \}$ \COMMENT{\emph{The partition that corresponds to the root.}}

\STATE $\calP_{\l} = \emptyset$ for any $\l \in [n]$

\STATE $\mathrm{level} \gets 0, n_{nodes} \gets 1, \mathrm{queue} \gets \calP_0$

\STATE \textbf{while} $n_{nodes} < t$ \textbf{do}
\STATE ~~~~ \textbf{if} $\mathrm{queue} = \emptyset$ \textbf{do}

\STATE ~~~~~~~~ $\mathrm{level} \gets \mathrm{level}+1, \mathrm{queue} \gets \calP_{level}$

\STATE ~~~~ \textbf{endif}
\STATE ~~~~ Pick $A$ the first element in $\mathrm{queue}$
\STATE ~~~~ 
\textbf{if} $|\calV \cap A| \leq 1$ \textbf{then}
\STATE ~~~~~~~~ $\mathrm{queue} \gets \mathrm{queue} \setminus \{A\}, \calP_{level+1} \gets \calP_{level+1} \cup \{A\}$

\STATE ~~~~ \textbf{else}

\STATE ~~~~~~~~ 
Select $i \in [d]$ that maximizes $V_n^{\l}(A,i)$ \COMMENT{\emph{See \Cref{eq:emp-breiman}.}}
\STATE ~~~~~~~~ Cut the cell $A$ to cells $A_k^i = \{ \vec x | \vec x \in A \land x_i = k \}, k =0,1$
\STATE ~~~~~~~~
$\mathrm{queue} \gets \mathrm{queue} \setminus \{A\}, $
$\calP_{level+1} \gets \calP_{level+1} \cup \{ A_0^i, A_1^i \}$
\STATE ~~~~ \textbf{endif}
\STATE \textbf{endwhile}
\STATE $\calP_{level+1} \gets \calP_{level+1} \cup \mathrm{queue}$
\STATE \textbf{Output} $(\calP_n, f^{(n)}) = (\calP_{level+1}, \vec x \mapsto f^{(n)}(\vec x; \calP_{level+1}))$
\end{algorithmic}
\end{algorithm}

\begin{theorem}
[Learning with Decision Trees via Breiman (see \cite{syrgkanis2020estimation}]
\label{thm:breiman}
Let $\eps, \delta > 0$.
Let $H > 0$.
Let $\mathrm{D}_n$ be i.i.d. samples from the nonparametric regression model $y = f(\vec x) + \xi$, where $f(\vec x) \in [-1/2, 1/2], \xi \sim \calE, \E_{\xi \sim \calE} [\xi]  = 0 $ and
$\xi \in [-1/2, 1/2].$
Let also $P_n$ be the partition that the algorithm (see \Cref{algo:breiman}) returns with input $h = 0$. The following statements hold.
\begin{enumerate}
    \item If $f$ be $r$-sparse as per \Cref{def:sparsity} and if the approximate diminishing returns \Cref{cond:diminish-app} holds,
    then given
    $
    n = \wt{O} \left( \log(d/\delta) (C \cdot r/\eps)^{C \cdot r + 3} \right)
    $
    samples and if we set $\log(t) \geq \frac{Cr}{Cr+3}(\log(n) - \log(\log(d/\delta)))$, then it holds that
    \[
    \Pr_{\mathrm{D}_n \sim \calD^n} 
    \left [
    \E_{\vec x \sim \calD_x}
    \left [
    (f(\vec x) - f^{(n)}(\vec x; P_n))^2 
    \right] 
    > \eps 
    \right ] \leq \delta\,.
    \]
    
    \item If $f$ is $r$-sparse as per \Cref{def:sparsity}, if the approximate diminishing returns \Cref{cond:diminish-app} holds and the distribution $\calD_x$ is a product distribution, given 
    $
    n = \wt{O} \left( C^2 2^r \log(d/\delta)/ \eps^3 \right)
    $
    samples and if we set $\log(t) \geq r$, then it holds that
    \[
    \Pr_{\mathrm{D}_n \sim \calD^n} 
    \left [
    \E_{\vec x \sim \calD_x}
    \left [
    (f(\vec x) - f^{(n)}(\vec x; P_n))^2 
    \right] 
    > \eps 
    \right ] \leq \delta\,.
    \]
    
\end{enumerate}
\end{theorem}

Finally, 
in the work of \citet{syrgkanis2020estimation},
a result about Fully Grown Forests via the Breiman's criterion is provided under \Cref{cond:marg-lb}. Shortly, it holds that, using a training set of size $n = \frac{2^r \log(d/\delta)}{\eps \zeta \beta^2}$ and if $s = \wt{\Theta}(\frac{2^r \log(d/\delta)}{\zeta \beta^2})$, then it holds that
$
\Pr_{D_n \sim \calD^n} \left( \E_{\vec x \sim \calD_x} [(f(\vec x) - f^{(n,s)}(\vec x))^2] \geq \eps \right) \leq \delta.
$ Note that every tree $f(\vec x, D_s)$ is built using the \Cref{algo:breiman}, with inputs: $\log(t)$ large enough
so that every leaf has two or three samples, training set $D_s$ and $h=1$.

\section{Background on Statistical Learning Theory}
\label{section:background-slt}
For a detailed exposition of a statistical learning theory perspective to binary classification, we refer to \citet{bousquet2003introduction, boucheron2005theory}.

\paragraph{Talagrand's Inequality.} 
Let $Pf = \E f$ and $P_n f$ be the corresponding empirical functional.
Talagrand's inequality provides a concentration inequality for the random variable
$\sup_{f \in \calF} (Pf - P_n f)$, which depends on the maximum variance attained by any function over the class $\calF$.

\begin{fact}
[Theorem 5.4 in \citet{boucheron2005theory}]
\label{fact:talagrand}
Let $b > 0$ and $\calF$ be a set of functions from $\xSpace$ to $\reals$. Assume that all functions in $\calF$ satisfy $Pf - f \leq b$. Then, with probability at least $1-\delta$, it holds that
\[
\sup_{f \in \calF} (Pf - P_n f) \leq 2\E[\sup_{f \in F}(Pf - P_nf)] + \sqrt{\frac{2 \sup_{f \in \calF} \Var(f) \log(1/\delta)  }{n}} + \frac{4b \log(1/\delta)}{3n}\,.
\]
\end{fact}

\paragraph{On Tsybakov's Condition.} The following property holds for the Tsybakov's noise condition.
\begin{fact}
[Tsybakov's Condition]
\label{fact:tsybakov}
Let $\calG$ be a class of binary classifiers.
Under the Tsybakov's noise condition (see \Cref{cond:incomplete}.(ii)) with $a, B > 0$ and for $i \neq j$, it holds that
\[
\Pr_{(\vec x, \sigma) \sim \calD_R^{\vec q}}[g(\vec x) \neq g^\star_{i,j}(\vec x) | \sigma \ni \{i,j\}] \leq C_{a,B}
(L_{i,j}(g) - L_{i,j}(g^\star))^a\,,
\]
where $C_{a,B} = \frac{B^{1-a}}{ (1-a)^{1-a} a^a}$, $g^\star$ is the Bayes classifier 
and the loss function is defined as
\[
L_{i,j}(g) := \E_{(\vec x, \sigma) \sim \calD_R^{\vec q}} \vec 1\{ g(\vec x) \neq \sgn(\sigma(i) - \sigma(j)) \cap \sigma \ni \{i,j\} \}\,. 
\]
\end{fact}
\begin{proof}
Let us set $L^\star_{i,j} = L_{i,j}(g^\star)$.
Define the quantity
\[
\eta(\vec x) = \E_{(X, \sigma)}[ \sgn(\sigma(i) - \sigma(j)) = +1 | X = \vec x]\,.
\]
The loss of the classifier is equal to
\[
L_{i,j}(g) - L^\star_{i,j}
=
\E_{(\vec x, \sigma) \sim \calD_R^{\vec q}}[ |2 \eta(\vec x) - 1| \cdot \vec 1\{ g(\vec x) \neq g^\star_{i,j}(\vec x)\} | \sigma \ni \{i,j\} ]\,,
\]
and so
\[
L_{i,j}(g) - L_{i,j}^\star 
\geq t \E[ \vec 1\{ g(\vec x) \neq g^\star_{i,j}(\vec x)\} \cdot \vec 1\{ |2 \eta(\vec x) -1 | \geq t \}  | \sigma \ni \{i,j\}]\,.
\]
Using Markov's inequality, we have that for all $t \geq 0$:
\[
\begin{split}
L_{i,j}(g) - L_{i,j}^\star 
&\geq t \Pr[ |2 \eta(\vec x)-1 | \geq t | \sigma \ni \{i,j\} ]\\
& - t \E[ \vec 1\{ g(\vec x)= g^\star_{i,j}(\vec x)  \} \vec 1\{ |2 \eta(\vec x) -1| \geq t \} | \sigma \ni \{i,j\} ]\,.
\end{split}
\]
The Tsybakov's condition implies that
\[
L_{i,j}(g) - L_{i,j}^\star 
\geq t (1 - Bt^{\frac{a}{1-a}})
-t \Pr[ g(\vec x) = g^\star_{i,j}(\vec x) | \sigma \ni \{i,j\} ]\,.
\]
Hence,
\[
L_{i,j}(g) - L_{i,j}^\star 
\geq t (\Pr[ g(\vec x) \neq g^\star_{i,j}(\vec x) | \sigma \ni \{i,j\} ] - Bt^{\frac{a}{1-a}})\,.
\]
Choosing $t$ appropriately, one gets that
\[
\Pr[g(\vec x) \neq g^\star_{i,j}(\vec x) | \sigma \ni \{i,j\}] \leq \frac{B^{1-a}}{ (1-a)^{1-a} a^a} 
(L_{i,j}(g) - L^\star_{i,j})^a\,.
\]
The proof is concluded by setting $C_{a,B} = \frac{B^{1-a}}{ (1-a)^{1-a} a^a}$.
\end{proof}

\section{Noisy Oracle with Incomplete Rankings \& Semi-Supervised Learning}
\label{appendix:incomplete}

The main result in Label Ranking with incomplete permutations (see \Cref{thm:main-inc}) is based on the generative process of \Cref{def:regression-incomplete}. In this generative model, we assume that we do not observe the $\star$ symbol. A natural theoretical question is to consider the easier setting, where the $\star$ symbols are present in the output sample. Specifically, we modify \Cref{def:regression-incomplete} so that it preserves the $\star$ symbol in the ranking. The modified definition follows:
\begin{definition}
[Generative Process for Incomplete Data with $\star$]
\label{def:regression-incomplete-star}
Consider an instance of the Label Ranking problem
with underlying score hypothesis $\vec m : \xSpace \to [1/4,3/4]^k$. 
Consider the survival probabilities vector $\vec q : \xSpace \to [0,1]^k$.
Each sample is generated as follows: 
\begin{enumerate}
    \item $\vec x \in \xSpace$ is drawn from $\calD_x$.
    \item Draw $\vec q(\vec x)$-biased coins $\vec c \in \{-1,+1\}^k$. 
    \item Draw $\vec \xi \in [-1/4,1/4]^k$ from the zero mean distribution $\calE$.
    \item Compute $\vec y = \vec m(\vec x) + \vec \xi$. 
    \item Compute $\sigma$ by setting the alternative $i$ in $\mathrm{argsort}(\vec y)$ equal to $\star$ if $c_i < 0$ for any $i \in [k]$.
    \item Output $(\vec x, \sigma).$
\end{enumerate}
We let $(\vec x, \sigma) \sim_{\star} \calD_R^{\vec q}.$
\end{definition} 
Crucially, we remark that this variant of incomplete rankings does reveal the correct positions of the non-erased alternatives. Hence, we can apply label-wise decomposition techniques in order to address this problem. We shortly discuss some potential future directions.

\paragraph{Semi-supervised Learning Approach.} We can tackle this problem using results from \emph{multi-class learning theory using unlabeled samples}. We can decompose this incomplete LR problem into $k$ multiclass classification problems: each sample $(\vec x, \sigma) \sim_{\star} \calD_R^{\vec q}$ corresponds to $k$ samples $(\vec x, y_i)$, where $y_i = \sigma(i) \in [k] \cup \{\star \}$ and $\sigma(i)$ denotes the alternative in the $i$-th position. We can think of the $\star$ symbol as an unseen label and hence address the incomplete LR problem as a collection of multiclass classification problems with both labeled and unlabeled samples. For generalization bounds on multi-class learning, we refer to \citet{li2018multi} and for similar bounds on multi-class learning using unlabeled examples, we refer to
\citet{li2019multi}. 
Hence, we can reduce the problem of \Cref{def:regression-incomplete-star} into $k$ subproblems, each one corresponding to a problem of multiclass classification with unlabeled examples.

\paragraph{Coarse Labels Approach.} While the previous approach resolves the problem, we remark that we do not have exploited the provided information: since we know that the $\star$ symbol corresponds to the complement of the observed positions, we could adopt a learning from coarse labels approach. 
For this approach, we refer the reader to the theoretical framework of \citet{fotakis21a}.
However, we remark that there are dependencies between the coarsening in each position. The setting of coarse rankings captures various known partial and incomplete settings, e.g., top-$k$ rankings and and we believe it constitutes an interesting direction for future work.


\end{document}